\documentclass[12pt,journal,final,letterpaper,onecolumn]{article}
\usepackage[top=1in, bottom=1in, left=1in, right=1in]{geometry}
\usepackage{hyperref}
\usepackage{url}
\usepackage{array}
\usepackage{graphicx} 
\usepackage{subfigure}
\usepackage{algorithm}
\usepackage{hyperref}
\usepackage{amsmath, amssymb, amsthm}
\usepackage{cite}
\usepackage{natbib}

\usepackage{color}

\usepackage{algpseudocode}
\usepackage{booktabs}
\usepackage{setspace}
\newcommand{\beq}{\vspace{0mm}\begin{equation}}
\newcommand{\eeq}{\vspace{0mm}\end{equation}}
\newcommand{\beqs}{\vspace{0mm}\begin{eqnarray}}
\newcommand{\eeqs}{\vspace{0mm}\end{eqnarray}}
\newcommand{\barr}{\begin{array}}
\newcommand{\earr}{\end{array}}

\newcommand{\xv}{\boldsymbol{x}}

\newcommand{\cdotv}{\boldsymbol{\cdot}}

\newcommand{\Sigmamat}[0]{{\boldsymbol{\Sigma}}}

\newcommand{\betav}[0]{{\boldsymbol{\beta}}}

\newcommand{\muv}[0]{{\boldsymbol{\mu}}}

\newcommand{\E}{\mathbb{E}}

\newcommand{\given}{\,|\,}

\newtheorem{thm}{Theorem} 
\newtheorem{cor}[thm]{Corollary}

\newtheorem{prop}[thm]{Proposition}

\newtheorem{definition}{Definition}

\algblock{ParFor}{EndParFor}
\algnewcommand\algorithmicparfor{\textbf{parfor}}
\algnewcommand\algorithmicpardo{\textbf{do}}
\algnewcommand\algorithmicendparfor{\textbf{end\ parfor}}
\algrenewtext{ParFor}[1]{\algorithmicparfor\ #1\ \algorithmicpardo}
\algrenewtext{EndParFor}{\algorithmicendparfor}

\title{ 
Softplus Regressions and
 Convex Polytopes}

\author{
Mingyuan Zhou\thanks{M. Zhou is an assistant professor of statistics in 
the Department of Information, Risk, \& Operations Management 
and  Department of Statistics \& Data Sciences 
at the University of Texas at Austin, Austin, TX 78712, USA. \emph{Email:} \texttt{mingyuan.zhou@mccombs.utexas.edu}
} 
\\
The University of Texas at Austin, 
Austin, TX 78712 \\
}

\begin{document}

\maketitle
\begin{spacing}{1.25}
\begin{abstract}

To construct flexible nonlinear predictive distributions, the paper introduces a family of softplus function based regression models that convolve, stack, or combine both operations by convolving countably infinite stacked gamma distributions, whose scales depend on the covariates. Generalizing logistic regression that uses a single hyperplane to partition the covariate space into two halves, softplus regressions employ multiple hyperplanes to construct a confined space, related to a single convex polytope defined by the intersection of multiple half-spaces or a union of multiple convex polytopes, to separate one class from the other. The gamma process is introduced to support the convolution of countably infinite (stacked) covariate-dependent gamma distributions. For Bayesian inference, Gibbs sampling derived via novel data augmentation and marginalization techniques is used to deconvolve and/or demix the highly complex nonlinear predictive distribution. Example results demonstrate that softplus regressions provide flexible nonlinear decision boundaries, achieving classification accuracies comparable to that of kernel support vector machine while requiring significant less computation for out-of-sample prediction. 
 
 \vspace{2mm}

\end{abstract}
\end{spacing}
\section{Introduction}\label{sec:introduction}

Logistic and probit regressions that use a single hyperplane to partition the covariate space into two halves are widely used to model binary response variables given the covariates \citep{cox1989analysis,mccullagh1989generalized,albert1993bayesian,holmes2006bayesian}. They are easy to implement and simple to interpret, but 
neither of them is capable of producing nonlinear classification decision boundaries, and they may not provide large margin to achieve accurate out-of-sample predictions. For two classes not well separated by a single hyperplane, 
rather than regressing a binary response variable directly on its covariates, it is common to select a subset of covariate vectors as support vectors, choose a nonlinear kernel function, and regress a binary response variable on the kernel distances between its covariate vector and these support vectors \citep{boser1992training,
cortes1995support,
vapnik1998statistical,
scholkopf1999advances,RVM}. Alternatively, one may construct a deep neural network to nonlinearly transform the covariates in a supervised manner, 
and then regress a binary response variable on its transformed covariates \citep{hinton2006fast,
lecun2015deep,Bengio-et-al-2015-Book}. 

Both kernel learning and deep learning map the original covariates into a more linearly separable space, transforming a nonlinear classification problem into a linear one. 
In this paper, we propose a fundamentally different approach for nonlinear classification. 
Relying on neither the kernel trick nor a deep neural network to transform the covariate space, 
we construct a family of softplus regressions that exploit two distinct types of interactions between hyperplanes to define flexible nonlinear classification decision boundaries directly on the original covariate space. Since kernel learning based methods such as kernel support vector machines (SVMs) \citep{cortes1995support,
vapnik1998statistical} may scale poorly in that the number of support vectors often increases linearly in the size of the training dataset,  
they could be not only slow and memory inefficient to train but also unappealing for making fast out-of-sample predictions  \citep{steinwart2003sparseness,wang2011trading}. One motivation of the paper is to investigate the potential of using a set of hyperplanes, whose number is directly influenced by 
how the interactions of multiple hyperplanes can be used to 
spatially separate two different classes in the covariate space rather than by the training data size, to construct nonlinear classifiers that can match the out-of-sample prediction accuracies of kernel SVMs, but potentially with much lower computational complexity. Another motivation of the paper is to increase the margin of the classifier, related to the discussion in \citet{kantchelian2014large} that for two classes that are linearly separable, even though a single hyperplane is sufficient to  separate the two different classes in the training dataset, using multiple hyperplanes to enclose one class may help clearly increase the total margin of the classifier and hence improve the out-of-sample prediction accuracies. 

Our motivated construction exploits two distinct operations---convolution and stacking---on the gamma distributions with covariate-dependent scale parameters. The convolution operation convolves differently parameterized probability distributions to increase representation power and enhance smoothness, while the stacking operation mixes a distribution in the stack 
with a distribution of the same family that is subsequently pushed into the stack. Depending on whether and how the convolution and stacking operations are used, 
the models in the family differ from each other on how they use the softplus functions 
to construct highly nonlinear probability density functions, and on how they 
construct their hierarchical Bayesian models to arrive at these functions. In comparison to the nonlinear classifiers built on kernels or deep neural networks, the proposed softplus regressions all share a distinct advantage in providing interpretable geometric constraints, which are related to either a single or a union of convex polytopes \citep{polytope}, on the classification decision boundaries defined on the original covariate space. In addition, like neither kernel learning, whose number of support vectors often increases linearly in the size of data \citep{steinwart2003sparseness}, nor deep learning, which often requires carefully tuning both the structure of the deep network and the learning algorithm \citep{Bengio-et-al-2015-Book}, 
 the proposed nonparametric Bayesian softplus regressions 
 naturally provide probability estimates, automatically learn the complexity of the predictive distribution, 
 and quantify  model uncertainties with posterior samples. 

The remainder of the paper is organized as follows. 
In Section \ref{sec:model}, 
we define four different softplus regressions, present their underlying hierarchical models, and describe their distinct geometric constraints on how 
 the covariate space is partitioned. 
In Section~\ref{sec:inference}, we discuss Gibbs sampling via data augmentation and marginalization. 
In Section~\ref{sec:results}, we present experimental results on eight benchmark datasets for binary classification, making comparisons with 
five different classification algorithms. 
 We conclude the paper in Section~\ref{sec:conclusion}. We defer to the Supplementary Materials all the proofs, an accurate approximate sampler and some new properties for the Polya-Gamma distribution, the discussions on related work, and some additional example results. 

\vspace{-3mm}
 \section{ Hierarchical Models and Geometric Constraints}
 \label{sec:model}
\vspace{-2mm}


\subsection{Bernoulli-Poisson link and softplus function}
\label{sec:notation}
\vspace{-1mm}
To model a binary random variable, 
it is common to link it 
to a real-valued latent Gaussian random variable using either the logistic or probit links. 
Rather than following the convention, in this paper, 
we consider
 the Bernoulli-Poisson (BerPo) link \citep{Dunson05bayesianlatent,EPM_AISTATS2015}  to threshold a latent count  at one to obtain a binary outcome $y\in\{0,1\}$~as
\beq
y = \delta(m\ge 1), ~m\sim\mbox{Pois}(\lambda),\label{eq:BerPo}
\eeq
where $m\in\mathbb{Z}$, $\mathbb{Z}:=\{0,1,\ldots\}$, and $\delta(x) = 1$ if the condition $x$ is satisfied and $\delta(x) = 0$ otherwise.
The marginalization of the latent count $m$ from the BerPo link leads to 
$$
y\sim\mbox{Bernoulli}(p),~p=1-e^{-\lambda}.
$$
The conditional distribution of $m$ given $y$ and $\lambda$ 
can be efficiently simulated using a rejection sampler \citep{EPM_AISTATS2015}.
Since its use in 
\citet{EPM_AISTATS2015} to factorize the adjacency matrix of an 
undirected unweighted symmetric network, 
the BerPo link has been further extended for big binary tensor factorization \citep{hu2015zero}, multi-label learning \citep{rai2015large}, and deep Poisson factor analysis \citep{henao2015deep}. This link has also been used by 
\citet{caron2014sparse} and \citet{todeschini2016exchangeable} for network analysis. 

We now refer to $\lambda=-\ln(1-p)$, the negative logarithm of the Bernoulli failure probability, 
 as the BerPo rate for $y$ and simply denote \eqref{eq:BerPo} as 
$
y \sim \mbox{BerPo}(\lambda).\notag
$ 
It is instructive to notice that $1/(1+e^{-x}) = 1-\exp[{-\ln(1+e^x)}], $ and hence
 letting 
\beq
y\sim\mbox{Bernoulli}\big[\sigma(x)\big],~\sigma(x) = 1/(1+e^{-x}) 
\label{eq:BerSigmoid}
\eeq
is equivalent to letting
\beq
y\sim\mbox{BerPo}\big[\varsigma(x)\big],~\varsigma(x) = \ln(1+e^x),
\label{eq:BerPoSoftplus}
\eeq
where 
$\varsigma(x) = \ln(1+e^x)$ was referred to as the softplus function 
in \citet{dugas2001incorporating}. 
It is interesting that the BerPo link appears to be naturally paired with the softplus function, which is often
considered as a smoothed version of the rectifier, or rectified linear unit 
 $$
 \mbox{ReLU}(x) = \max(0,x),
 $$
that is now widely used in deep neural networks, replacing other canonical nonlinear activation functions such as the sigmoid and hyperbolic tangent functions \citep{nair2010rectified,glorot2011deep,lecun2015deep,krizhevsky2012imagenet,CRLU}. 

In this paper, we further introduce the stack-softplus function 
\beq \varsigma(x_1,\ldots,x_t) =
\ln\left(1+e^{x_t}\ln\left\{1+e^{x_{t-1}}\ln\big[1+\ldots
\ln\big(1+e^{x_1}\big)\big]\right\}\right), \label{eq:deep_softplus}
 \eeq
which can be recursively defined 
with 
 $\varsigma(x_1,\ldots,x_t) = \ln[1+e^{x_t} \varsigma(x_1,\ldots,x_{t-1})]$.
 In addition, with $r_k$ as the weights of the countably infinite atoms of a gamma process \citep{ferguson73}, we will introduce the sum-softplus function, expressed as
$
\sum_{k=1}^\infty r_k\, \varsigma(x_k),
 $
 and sum-stack-softplus (SS-softplus) function, expressed as
 $
\sum_{k=1}^\infty r_k\, \varsigma(x_{k1},\ldots,x_{kt}).
 $
 The stack-, sum-, and SS-softplus functions constitute a family of softplus functions, which are used
to construct nonlinear  regression models, as presented below.

\subsection{The softplus regression family}\label{sec:family}
The equivalence between \eqref{eq:BerSigmoid} and \eqref{eq:BerPoSoftplus}, the apparent partnership between the BerPo link and softplus function, and the convenience of 
employing multiple regression coefficient vectors to
parameterize
the BerPo rate, which is constrained to be nonnegative rather than between zero and one,  motivate us to consider using the BerPo link together with softplus function 
 to model binary response variables given the covariates. We first show how a classification model under the BerPo link reduces to logistic regression that uses a single hyperplane to partition the covariate space into two halves. We then generalize it to two distinct multi-hyperplane  classification models: sum- and stack-softplus regressions, and further show how to integrate them into SS-softplus regression. 
These models clearly differ from each other on how the BerPo rates are parameterized with the softplus functions, leading to decision boundaries under distinct geometric constraints. 

 
 To be more specific, for 
 the $i$th covariate vector $\xv_i=(1,x_{i1},\ldots,x_{iV})'\in\mathbb{R}^{V+1}$, where the prime denotes the operation of transposing a vector, we model its binary class label using 
 \beq
 y_i \given \xv_i \sim\mbox{BerPo}[\lambda(\xv_i)],
 \label{eq:BerPoINreg}
 \eeq
 where $\lambda(\xv_i)$, given the regression model parameters that may come from a stochastic process,  is a nonnegative deterministic function of $\xv_i$ that may contain a countably infinite number of parameters. 
Let $G\sim\Gamma\mbox{P}(G_0,1/c)$ denote a gamma process \citep{ferguson73} defined on the product space $\mathbb{R}_+\times \Omega$, where $\mathbb{R}_+=\{x:x>0\}$, $c$ is a scale parameter, and $G_0$ is a finite and continuous base measure defined on a complete separable metric space $\Omega$, such that $G(A_i)\sim\mbox{Gamma}(G_0(A_i),1/c)$ are independent gamma random variables for disjoint Borel sets $A_i$ of $\Omega$. 
%
 Below we show how the BerPo 
 rate function $\lambda(\xv_i)$ is parameterized under four different softplus regressions, two of which use the gamma process to support a countably infinite sum in the parameterization, and also 
  show
 how to arrive at 
each parameterization using a hierarchical Bayesian model built on the BerPo link together with the convolved and/or stacked gamma distributions. 
 
 \begin{definition}[Softplus regression]\label{thm:1}
 Given $\xv_i$, weight $r\in \mathbb{R}_+$, and a regression coefficient vector $\betav\in\mathbb{R}^{V+1}$, 
softplus regression parameterizes $\lambda(\xv_i)$ in \eqref{eq:BerPoINreg} using a softplus function as 
 \beq
 \lambda(\xv_i) =\varsigma(\xv_i'\betav)= r \ln(1+e^{\xv_i'\betav}). \label{eq:BerPoSoftplusReg}
 \eeq
Softplus regression is equivalent to the binary regression model
$$
y_{i} \sim 
\emph{\mbox{Bernoulli}}\left[ 1-\big({1+e^{\xv_i'\betav}}\big)^{-r\,} \right], \notag
$$
which, as proved in Appendix \ref{sec:proof},  can be constructed using the hierarchical model 
\begin{align}
&y_{i} = \delta( m_{i} \ge 1) ,~m_{i}\sim\emph{\mbox{Pois}}(\theta_{i}),~\theta_{i}\sim\emph{\mbox{Gamma}}\big(r,
e^{\xv_i'\betav}\big). \label{eq:Softplus_model}
\end{align}

\end{definition}

\begin{definition}[Sum-softplus regression]  \label{thm:CLR2CNB}
Given a draw from a gamma process $G\sim\Gamma\emph{\mbox{P}}(G_0,1/c)$, 
expressed as $G=\sum_{k=1}^\infty r_k \delta_{\betav_k} $, 
where $\betav_k\in\mathbb{R}^{V+1}$ is an atom 
and $r_k$ is its weight,
sum-softplus regression parameterizes $\lambda(\xv_i)$ in \eqref{eq:BerPoINreg} using a sum-softplus function as 
 \beq
\lambda(\xv_i)= \sum_{k=1}^\infty r_k \,\varsigma(\xv_i' \betav_k) =
\sum_{k=1}^\infty r_k \ln(1+e^{\xv_i' \betav_k}). \label{eq:sum_softplus_reg}
 \eeq
Sum-softplus regression is
equivalent to the binary regression model 
\beq
y_{i} \sim \emph{\mbox{Bernoulli}}\left[ 1-e^{-\sum_{k=1}^{\infty}r_k \varsigma(\xv_i'\betav_{k})} \right] =
\emph{\mbox{Bernoulli}}\left[ 1-\prod_{k=1}^{\infty}\left(\frac{1}{1+e^{\xv_i'\betav_{k}}}\right)^{r_{k}} \right], \notag
\eeq
which, as proved in Appendix \ref{sec:proof}, can be constructed using the hierarchical model
\begin{align}
&y_{i} = \delta( m_{i} \ge 1),~m_{i} \sim\emph{\mbox{Pois}}(\theta_i),~\theta_i = \sum_{k=1}^\infty \theta_{ik}, ~\theta_{ik}\sim{\emph{\mbox{Gamma}}}\big(r_{k},e^{\xv_i'\betav_{k}}\big)
.
 \label{eq:CNB}
\end{align}

 \end{definition}
 
\begin{definition}[Stack-softplus regression]\label{thm:stack-softplus}
 With weight $r\in\mathbb{R}_+$ and $T$ regression coefficient vectors $\betav^{(2:T+1)}:=(\betav^{(2)},\ldots,\betav^{(T+1)})$, where $\betav^{(t)}\in\mathbb{R}^{V+1}$, stack-softplus regression with $T$ layers parameterizes $\lambda(\xv_i)$ in \eqref{eq:BerPoINreg} using a stack-softplus function as
 \beqs
& \displaystyle\lambda(\xv_i) = r\, \varsigma\big(\xv_i'\betav^{(2)},\ldots,\xv_i'\betav^{(T+1)}\big)
\notag\\
&\displaystyle =r\ln\left(1+e^{\xv_i'\betav^{(T+1)}}\ln\bigg\{1+e^{\xv_i'\betav^{(T)}}\ln\Big[1+\ldots
\ln\big(1+e^{\xv_i'\betav^{(2)}}\big)\Big]\bigg\}\right). \label{eq:recurssive_softplus_reg}
 \eeqs
Stack-softplus regression is equivalent to the regression model 
\beqs\small
y_i&\sim&\emph{\mbox{Bernoulli}}\Big(1-e^{-r\,\varsigma\left(\xv_i'\betav^{(2:T+1)}\right)}\Big)\notag\\
&=&\emph{\mbox{Bernoulli}}\left[1 - \left(1\!+\!e^{\xv_i'\betav^{(T+1)}}\ln\bigg\{1\!+\!e^{\xv_i'\betav^{(T)}}\ln\Big[1\!+\!\ldots
\ln\big(1\!+\!e^{\xv_i'\betav^{(2)}}\big)\Big]\bigg\}\right)^{-r} ~\right ],~~ \notag
%
\eeqs
which, as proved in Appendix  \ref{sec:proof}, can be constructed using the hierarchical model that stacks $T$ gamma distributions, whose scales are differently parameterized by the covariates, as
\begin{align}
&~~\theta^{(T)}_{i}\sim\emph{\mbox{Gamma}}\left(r,e^{\xv_i'\betav^{(T+1)}}\right),\notag\\
&~~~~~~~~~~~~~~\cdots\notag\\
&\theta^{(t)}_{i}\sim\emph{\mbox{Gamma}}\left(\theta^{(t+1)}_{i},e^{\xv_i'\betav^{(t+1)}}\right), \notag\\
&~~~~~~~~~~~~~~\cdots\notag\\
y_i = \delta(m_i\ge 1&),~m_{i} \sim \emph{\mbox{Pois}}(\theta^{(1)}_{i}),~\theta^{(1)}_{i}\sim\emph{\mbox{Gamma}}\left(\theta^{(2)}_{i},e^{\xv_i'\betav^{(2)}}\right). 
\label{eq:BerPo_recursive_softplus_reg_model}
\end{align}
 \end{definition}
 
 \begin{definition}[Sum-stack-softplus (SS-softplus) regression]\label{thm:SS-softplus}
Given a drawn from a gamma process $G\sim\Gamma\emph{\mbox{P}}(G_0,1/c)$, 
expressed as $G=\sum_{k=1}^\infty r_k \delta_{\betav^{(2:{T}+1)}_k} $, 
where $\betav^{(2:{T}+1)}_k$ is an atom  and $r_k$ is its weight, with each $\betav_k^{(t)}\in\mathbb{R}^{V+1}$, 
%
SS-softplus regression with  ${T}\in\{1,2,\ldots\}$ layers 
 parameterizes $\lambda(\xv_i)$ in \eqref{eq:BerPoINreg} using a sum-stack-softplus function as
 \beqs
&\displaystyle\lambda(\xv_i)=
\sum_{k=1}^\infty r_k\, \varsigma\big(\xv_i'\betav_k^{(2)},\ldots,\xv_i'\betav_k^{(T+1)}\big)
\notag\\
&\displaystyle=
\sum_{k=1}^\infty r_k
\ln\left(1+e^{\xv_i'\betav^{({T}+1)}_k}\ln\bigg\{1+e^{\xv_i'\betav^{({T})}_k}\ln\Big[1+\ldots
\ln\big(1+e^{\xv_i'\betav^{(2)}_k}\big)\Big]\bigg\}\right). \label{eq:SRS_regression}
\eeqs
SS-softplus regression is equivalent to the regression model 
\beqs\small
y_i&\sim&\emph{\mbox{Bernoulli}}\Big(1-e^{-\sum_{k=1}^\infty r_k\,\varsigma\left(\xv_i'\betav_k^{(2:T+1)}\right)}\Big)\notag\\
&=&\emph{\mbox{Bernoulli}}\left[
1 -\prod_{k=1}^\infty \left(\!1\!+\!e^{\xv_i'\betav_k^{({T}\!+\!1)}}\ln\!\bigg\{1\!+\!e^{\xv_i'\betav_k^{({T})}}\ln\Big[1\!+\!\ldots
\ln\big(1\!+\!e^{\xv_i'\betav_k^{(2)}}\big)\Big]\bigg\}\right)^{-r_k} ~\!\right ] \notag
%
\eeqs
which, as proved in Appendix  \ref{sec:proof}, can be constructed from the hierarchical model that convolves countably infinite stacked gamma distributions that have covariate-dependent scale parameters as
\small\begin{align}
&~~~~~~~~\theta^{({T})}_{ik}\sim\emph{\mbox{Gamma}}\left(r_k,e^{\xv_i'\betav^{({T}+1)}_k}\right)
,\notag\\
&~~~~~~~~~~~~~~~~~~~~~~~~\cdots\notag\\
&~~~~~~~~\theta^{(t)}_{ik}\sim\emph{\mbox{Gamma}}\left(\theta^{(t+1)}_{ik},e^{\xv_i'\betav^{(t+1)}_k}\right), \notag\\
&~~~~~~~~~~~~~~~~~~~~~~~~\cdots\notag\\
y_i = \delta(m_i\ge 1),&~m_i
=\sum_{k=1}^\infty m^{(1)}_{ik},~m^{(1)}_{ik} \sim \emph{\mbox{Pois}}(\theta^{(1)}_{ik}),~\theta^{(1)}_{ik}\sim\emph{\mbox{Gamma}}\left(\theta^{(2)}_{ik},e^{\xv_j'\betav^{(2)}_k}\right)\,\,.
\label{eq:DICLR_model}
\end{align}\normalsize

 \end{definition}

Below we discuss these four different  softplus regression models in detail and  show that both sum- and stack-softplus regressions use the interactions of multiple 
 regression coefficient vectors through the softplus functions to define a confined space,  related to a convex polytope \citep{polytope} defined by the intersection of multiple half-spaces, to separate one class from the other in the covariate space. 
 They differ from each other 
 in that sum-softplus regression infers a convex-polytope-bounded confined space to enclose negative examples (\emph{i.e.}, data samples with $y_i=0$),
 whereas stack-softplus regression infers a convex-polytope-like confined space to enclose positive examples (\emph{i.e.}, data samples with $y_i=1$).
 
 The opposite behaviors of sum- and stack-softplus regressions motivate us to unite them
as 
SS-softplus regression, 
%
 which can place countably infinite convex-polytope-like confined spaces, inside and outside each of which favor positive and negative examples, respectively, at various regions of the covariate space, and use the union of these confined spaces to construct a flexible nonlinear classification decision boundary. 
 Note that softplus regressions all operate on the original covariate space. It is possible to apply them to regress binary response variables on the covariates that have already been nonlinearly transformed with the kernel trick or a deep neural network, which may combine the advantages of these distinct methods 
 to achieve an overall improved classification performance. We leave the integration of softplus regressions with the kernel trick or deep neural networks for future study. 
%


\subsection{Softplus and logistic 
regressions}

 It is straightforward to show that 
 softplus regression with $r=1$ is equivalent to logistic regression 
$
y_i\sim{\mbox{Bernoulli}}[1/(1+e^{-\xv_i'\betav})] 
$, 
 which uses a single hyperplane dividing the covariate space into two halves to separate one class from the other. Similar connection has also been illustrated in \citet{Dunson05bayesianlatent}. Clearly, softplus regression arising from \eqref{eq:Softplus_model} generalizes logistic regression in allowing $r\neq 1$. 
Let $p_0\in(0,1)$ denote the probability threshold to make a binary decision. 
One may consider that softplus regression 
 defines a hyperplane to partition the $V$ dimensional covariate space into two halves: one half is defined with 
$\xv_i'\betav > \ln\big[(1-p_0)^{-\frac{1}{r}}-1\big]$,
 assigned with label $y_i=1$ since $P(y_i=1\,|\,\xv_i,\betav_k)>p_0$ under this condition, and the other half is defined with $\xv_i'\betav \le \ln\big[(1-p_0)^{-\frac{1}{r}}-1\big]$,
 assigned with label $y_i=0$ since $P(y_i=1\,|\,\xv_i,\betav)\le p_0$ under this condition. 
 Instead of using a single hyperplane, the three generalizations in Definitions 2-4 all 
 partition the covariate space using 
 a confined space that is related to  a single convex polytope or the union of multiple convex polytopes, as described below. 

\vspace{-2mm}
\subsection{Sum-softplus regression and convolved NB regressions}

Note that since $m_{i} \sim{\mbox{Pois}}(\theta_i), ~\theta_i = \sum_{k=1}^\infty \theta_{ik}$ in \eqref{eq:CNB} can be equivalently written as $m_i = \sum_{k=1}^\infty m_{ik},~m_{ik}\sim{\mbox{Pois}}(\theta_{ik})$,
 sum-softplus regression can also be constructed with 
\begin{align}
&y_{i} = \delta( m_{i} \ge 1),~~m_{i} = \sum_{k=1}^{\infty} m_{ik},~m_{ik}\sim{\mbox{NB}}\left[r_{k},{1}/{(1+e^{-\xv_i'\betav_{k}})}\right],
 \label{eq:CNB1}
\end{align}
where $m\sim\mbox{NB}(r,p)$ represents a negative binomial (NB) distribution  \citep{Yule,Fisher1943} with shape parameter $r$ and probability parameter $p$, and $m\sim\mbox{NB}[r,1/(1+e^{-\xv'\betav})]$ can be considered as NB regression \citep{LawlessNB87,long:1997,Cameron1998,WinkelmannCount} 
that parameterizes the logit of $p$ 
with $\xv'\betav$. To ensure that the infinite model is well defined, we provide the following proposition and present the proof in Appendix \ref{sec:proof}.  
\begin{prop}\label{lem:finite}
The infinite product $
e^{-\sum_{k=1}^{\infty}r_k \,\varsigma(\xv_i'\betav_{k})}
= \prod_{k=1}^\infty\left( {1+e^{\xv_i'\betav_k}}\right)^{-r_k}$ in sum-softplus regression 
is smaller than one and has a finite expectation that is greater than zero.
\end{prop}

As the probability distribution of the sum of independent random variables is the same as the convolution of these random variables' probability distributions \citep[$e.g.$,][]{fristedt1997}, the probability distribution of the BerPo rate $\theta_i$ is the convolution of countably infinite gamma distributions, each of which parameterizes the logarithm of its scale using the inner product of the same covariate vector and a regression coefficient vector specific for each~$k$. 
As in \eqref{eq:CNB1}, 
since $m_i$ is the summation of countably infinite latent counts $m_{ik}$, each of which is a NB regression response variable, we essentially regress the latent count $m_i$ on $\xv_i$ using a convolution of countably infinite NB regression models. If $\betav_{k}$ are drawn from a continuous distribution, then $\betav_{k}\neq\betav_{{\tilde{k}}}$ a.s. for all $k\neq{\tilde{k}}$, and hence given $\xv_i$ and $\{\betav_k\}_k$, 
the BerPo rate $\theta_i$ would not follow the gamma distribution and 
$m_i$ would not follow the NB distribution.

Note that if we modify the proposed sum-softplus regression model in \eqref{eq:CNB} as
\begin{align}
&y_{ij} = \delta( m_{ij} \ge 1),~m_{ij} \sim{\mbox{Pois}}(\theta_{ij}),~\theta_{ij} = \sum_{k=1}^K \theta_{ijk}, ~\theta_{ijk}\sim{{\mbox{Gamma}}}\big(\phi_{k}^{-1},\phi_k \lambda_{jk}e^{\xv_{ij}'\betav}\big)
,
 \label{eq:CNB_Dunson}
\end{align}
then we have $P(y_{ij}=1\given \xv_{ij}) = \left[ 1-\prod_{k=1}^{K}\left({1+\phi_{k}\lambda_{jk} e^{\xv_{ij}'\betav}}\right)^{-\phi_{k}^{-1}} \right]$, which becomes the same as Eq. 2.7 of \citet{Dunson05bayesianlatent} that is designed to model multivariate binary response variables. Though related, that construction 
 is clearly different from the proposed sum-softplus regression in that it uses only a single regression coefficient vector $\betav$ and does not support $K\rightarrow \infty$. It is of interest to extend the models in \citet{Dunson05bayesianlatent} with the sum-softplus construction discussed above and the stack- and SS-softplus constructions to be discussed below.



\subsubsection{ Convex-polytope-bounded 
confined space that favors negative examples} 
For sum-softplus regression arising from \eqref{eq:CNB}, the binary classification decision boundary is no longer defined by a single hyperplane. Let us make the analogy that each $\betav_k$ is an expert of a committee that collectively make binary decisions. For expert $k$, the magnitude of $r_k$ indicates how strongly its opinion is weighted by the committee, $m_{ik}=0$ represents that it votes ``No," and $m_{ik}\ge 1$ represents that it votes ``Yes." Since the response variable $y_i =\delta\left(\sum_{k=1}^\infty m_{ik}\ge 1\right)$, the committee would vote ``No'' if and only if all its experts vote ``No'' ($i.e.$, all the counts $m_{ik}$ are zeros), in other words, 
the committee would vote 
 ``Yes'' even if only a single expert votes ``Yes.'' 

%

 Let us now examine the confined covariate space for sum-softplus regression 
 that satisfies the inequality
 $P(y_i=1\given \xv_i)=1-e^{-\lambda(\xv_i)}\le p_0$, where a data point is labeled as one with a probability no greater than $p_0$. Although it is not immediately clear what kind of geometric constraints are imposed on the covariate space by this inequality, 
 the following theorem shows that 
 it defines a confined space, which is bounded by a convex polytope defined by 
 the intersection of countably infinite half-spaces.
 
 \begin{thm}\label{thm:sum_polytope}
 For sum-softplus regression, the confined space specified by the inequality $P(y_i=1\given \xv_i)=1-e^{-\lambda(\xv_i)}\le p_0$, which can be expressed as
 \beq
 \lambda(\xv_i) = \sum_{k=1}^\infty r_k \ln(1+e^{\xv_i' \betav_k})\le -\ln(1-p_0), \label{eq:sum_ineuqality}
 \eeq
 is bounded by a convex polytope defined by the set of solutions to countably infinite inequalities 
\beq\label{eq:convex_polytope}
\xv_i' \betav_k \le \ln\big[(1-p_0)^{-\frac{1}{r_k}}-1\big], ~~ k\in\{1,2,\ldots\}.
\eeq
 \end{thm}
 
 \begin{prop}\label{prop:sum_polytope}
 For any data point $\xv_i$ that resides outside the convex polytope defined by
~\eqref{eq:convex_polytope}, which means $\xv_i$ violates at least one of the inequalities in \eqref{eq:convex_polytope} a.s.,
 it will be labeled under sum-softplus regression with $y_i=1$ with a probability greater than $p_0$, and $y_i=0$ with a probability no greater than $1-p_0$.
 \end{prop}

 The convex polytope defined in \eqref{eq:convex_polytope} is enclosed by the intersection of countably infinite $V$-dimensional half-spaces. If we set $p_0=0.5$ as the probability threshold to make binary decisions, then the convex polytope 
 assigns a label of $y_i=0$ to an $\xv_i$ inside the convex polytope ($i.e.$, an $\xv_i$ that satisfies all the inequalities in Eq. \ref{eq:convex_polytope}) with a relatively high 
 probability, and assigns a label of $y_i=1$ to an $\xv_i$ outside the convex polytope ($i.e.$, an $\xv_i$ that violates at least one of the inequalities in Eq. \ref{eq:convex_polytope}) with a probability of at least $50\%$.

Note that as $r_k\rightarrow 0$, 
$r_k\ln(1+e^{\xv_i' \betav_{k}}) \rightarrow 0 $ and $\ln\big[(1-p_0)^{-\frac{1}{r_k}}-1\big]\rightarrow \infty$. Thus 
expert $k$ with a tiny $r_k$ essentially has a negligible impact on both the decision of the committee and the boundary of the convex polytope. 
Choosing the gamma 
process as the nonparametric Bayesian prior sidesteps the need to tune the number of experts in the committee, shrinking the weights of all unnecessary experts and hence allowing a finite number of experts with non-negligible weights to be automatically inferred from the data.
 
%

\vspace{-1mm}
\subsubsection{Illustration for sum-softplus regression}
\vspace{-1mm}
A clear advantage of sum-softplus regression over both softplus and logistic regressions is that it could use multiple hyperplanes 
 to construct a nonlinear decision boundary and, similar to the convex polytope machine of \citet{kantchelian2014large}, to separate two different classes by a large margin. 
To illustrate the imposed geometric constraints, 
we first consider a synthetic two dimensional dataset with two classes, as shown in Fig. \ref{fig:circle_1} (a), where most of the data points of Class $B$ reside within a unit circle and these of Class $A$ reside within a ring outside the unit circle. 

 \begin{figure}[!t]
\begin{center}
 \includegraphics[width=0.75\columnwidth]{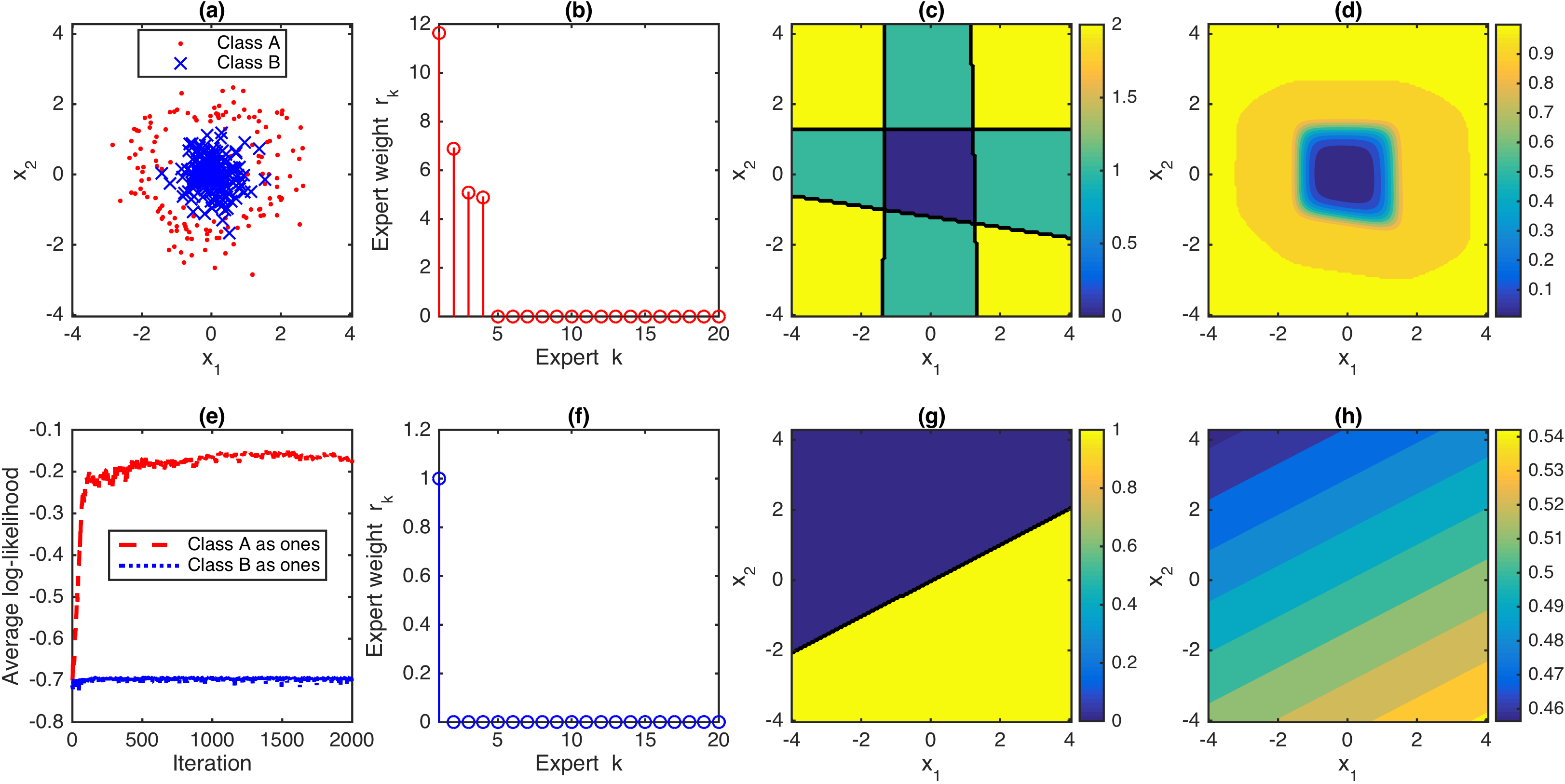}
 \vspace{-.2cm}
\end{center}
\vspace{-4.9mm}
\caption{\small\label{fig:circle_1}
%
%
Visualization of sum-softplus regression with $K_{\max}=20$ experts on a binary classification problem under two opposite labeling settings. For each labeling setting, 2000 Gibbs sampling iterations are used and the MCMC sample that provides the maximum likelihood on fitting the training data labels 
is used to display the results. (a) A two dimensional dataset that consists of 150 data points from Class $A$, whose radiuses are drawn from $\mathcal{N}(2,0.5^2)$ and angles are distributed uniformly at random between 0 and $360$ degrees, and another 150 data points from Class $B$, whose both $x$-axis and $y$-axis values are drawn from $\mathcal{N}(0,0.5^2)$.
 With data points in Classes $A$ and $B$ labeled as ``1'' and ``0,'' respectively, and with $p_0=0.5$, (b) shows the inferred weights $r_k$ of the experts, ordered by their values, (c) shows a contour map, the value of each point of which represents how many inequalities specified in \eqref{eq:convex_polytope} are violated, and whose region with zero values corresponds to 
 the convex polytope enclosed by the intersection of the hyperplanes defined in \eqref{eq:convex_polytope}, and (d) shows the the contour map of the predicted class probabilities. 
 (f)-(h) are analogous plots to (b)-(d), with the data points in Classes $A$ and $B$ relabeled as ``0'' and ``1,'' respectively. 
 (e) The average per data point log-likelihood as a function of MCMC iteration, for both labeling settings.
 }
\end{figure}

 \begin{figure}[!t]
\begin{center}
 \includegraphics[width=0.75\columnwidth]{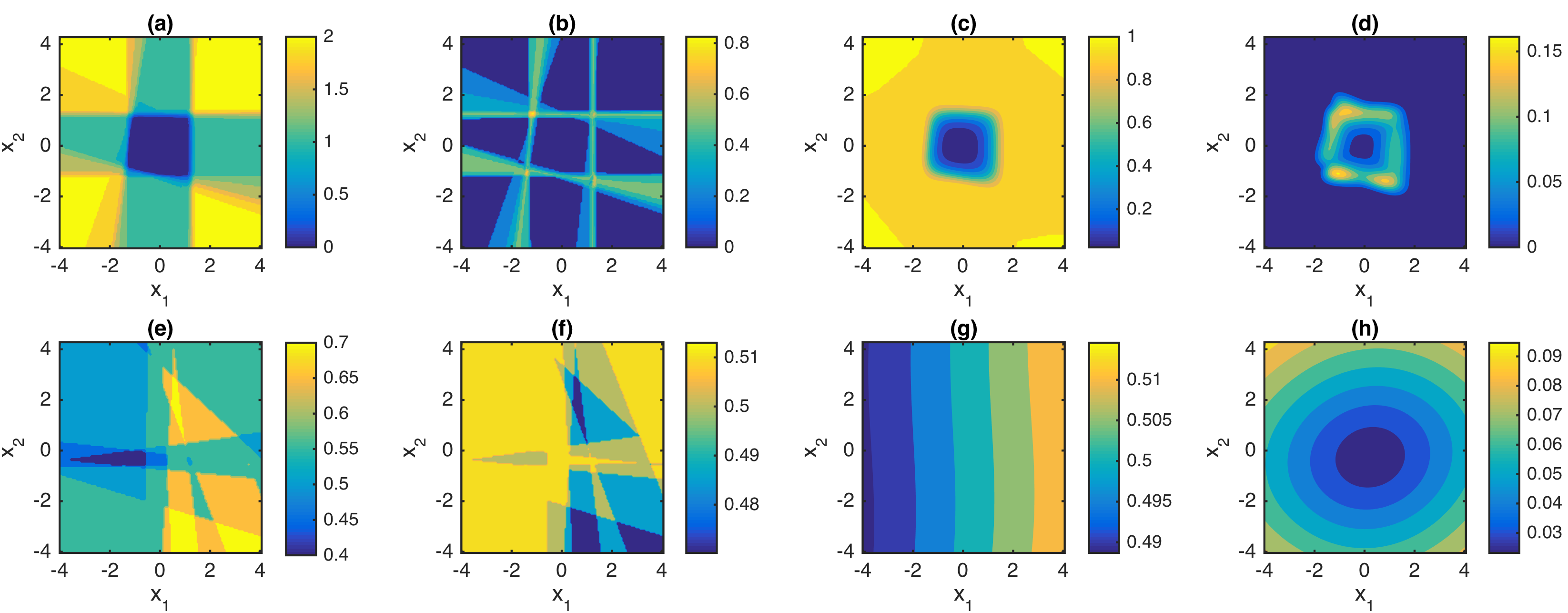}
 \vspace{-.2cm}
\end{center}
\vspace{-5.9mm}
\caption{\small\label{fig:circle_ave_1}
Visualization of the posteriors of sum-softplus regression based on 20 MCMC samples, collected once per every 50 iterations during the last 1000 MCMC iterations, with the same experimental setting used for Fig. \ref{fig:circle_1}. With $p_0=0.5$,
(a) and (b) show the contour maps of the posterior means and standard deviations, respectively, of the number of inequalities specified in \eqref{eq:convex_polytope} that are violated, and (c) and (d) show the contour maps of the posterior means and standard deviations, respectively, of predicted class probabilities. 
 (e)-(h) are analogous plots to (a)-(d), with the data points in Classes $A$ and $B$ relabeled as ``0'' and ``1,'' respectively. 
 }
 \end{figure}

We first label the data points of Class $A$ as ``1'' and these of Class $B$ as ``0.'' Shown in Fig. \ref{fig:circle_1} (b) are the inferred weights $r_k$ of the experts, using the MCMC sample that has the highest log-likelihood in fitting the training data labels. It is evident from Figs. \ref{fig:circle_1} (b) and (c) that sum-softplus regression infers four experts (hyperplanes) with significant weights. 
The convex polytope in Fig. \ref{fig:circle_1} (c) that encloses the space marked as zero is intersected by these four hyperplanes, each of which is defined as in \eqref{eq:convex_polytope} with $p_0=50\%$.
Thus outside the convex polytope are data points that would be labeled as ``1'' with at least 50\% probabilities and inside it are data points that would be labeled as ``0'' with relatively high 
 probabilities. We further show in Fig. \ref{fig:circle_1} (d) the contour map of the inferred probabilities for $P(y_i=1\given \xv_i) = 1-e^{-\lambda(\xv_i)}$, where $\lambda(\xv_i)$ are calculated with~\eqref{eq:sum_softplus_reg}. 
Note that due to the model construction, 
a single expert's influence on the decision boundary 
can be conveniently measured, and the exact decision boundary is bounded by a convex polytope. 
Thus it is not surprising that 
the convex polytope in Fig. \ref{fig:circle_1} (c), which encloses the space marked as zero, 
aligns well with the contour line of $P(y_i=1\given \xv_i)=0.5$ shown in Fig. \ref{fig:circle_1} (d). 

Despite being able to construct a nonlinear decision boundary bounded by a convex polytope, sum-softplus regression has a clear restriction in that if the data labels are flipped, 
its performance 
may substantially deteriorate, becoming no better than 
that of logistic regression. 
For example, for the same data shown in Fig. \ref{fig:circle_1} (a), if we choose the opposite labeling setting where the data points of Class $A$ are labeled as ``0'' and these of Class $B$ are labeled as ``1,'' then sum-softplus regression infers a single expert (hyperplane) with non-negligible weight, as shown in Figs. \ref{fig:circle_1} (f)-(g), and fails to separate the data points of two different classes, as shown in Figs. \ref{fig:circle_1} (g)-(h). The data log-likelihood plots in Fig. \ref{fig:circle_1} (e) also suggest that sum-softplus regression could perform substantially better if the training data are labeled 
in favor of its geometric constraints on the decision boundary. 
An advantage of a Bayesian hierarchical model is that with collected MCMC samples, one may estimate not only the posterior means but also uncertainties. 
The standard deviations shown in Figs. \ref{fig:circle_ave_1} (b) and (d) clearly indicate the uncertainties of sum-softplus regression on its decision boundaries and predictive probabilities in the covariate space, which may be used to help decide how to sequentially query the labels of unlabeled data in an active learning setting \citep{cohn1996active,settles2010active}. 




The sensitivity of sum-softplus regression to how the data are labeled could be mitigated but not completely solved by combining two sum-softplus regression models trained under the two opposite labeling settings. 
In addition, sum-softplus regression may not perform well no matter how the data are labeled if neither of the two classes could be enclosed by a convex polytope. 
To fully resolve these issues,  we first introduce stack-softplus regression, which defines a convex-polytope-like confined space to enclose positive examples. 
We then show how to combine the two distinct, but complementary, softplus regression models to construct SS-softplus regression that provides more flexible nonlinear decision boundaries. 

\subsection{Stack-softplus regression and stacked gamma distributions}

The model in \eqref{eq:BerPo_recursive_softplus_reg_model} combines the BerPo link with a gamma belief network that stacks differently parameterized gamma distributions. Note that here  ``stacking''  is defined as an operation that mixes the shape parameter of a gamma distribution at layer $t$ with a gamma distribution at layer $t+1$, the next one pushed into the stack, and pops out the covariate-dependent gamma scale parameters from layers $T+1$ to 2 in the stack, following the last-in-first-out rule, 
to parameterize the BerPo rate of the class label $y_i$ shown in \eqref{eq:recurssive_softplus_reg}. 

\subsubsection{Convex-polytope-like confined space that favors positive examples}

Let us make the analogy that each $\betav^{(t)}$ is one of the $T$ criteria that an expert examines before making 
a binary decision. 
From \eqref{eq:recurssive_softplus_reg} 
it is clear that as long as a single criterion $t\in\{2,\ldots,{T}+1\}$ of the expert is strongly violated, which means that $\xv_i'\betav^{(t)}$ is much smaller than zero, then the expert would vote ``No'' regardless of the values of $\xv_i'\betav^{({\tilde{t}})}$ for all ${\tilde{t}}\neq t$.
Thus the response variable could be voted ``Yes'' by the expert only if none of the $T$ expert criteria are strongly violated. 
For stack-softplus regression, let us specify a confined space using the inequality $P(y_i=1\given \xv_i)=1-e^{-\lambda(\xv_i)}>p_0$, which can be expressed as 
\beq\small \label{eq:positive_convex_polytope}
\xv_i'\betav^{({T}+1)} + \ln \ln\!\bigg\{1+e^{\xv_i'\betav^{({T})}}\ln\Big[1+\ldots
\ln\big(1+e^{\xv_i'\betav^{(2)}}\big)\Big]\bigg\} > \ln\big[(1-p_0)^{-\frac{1}{r}}-1\big], 
\eeq
and hence any data point $\xv_i$ outside the confined space ($i.e.$, violating the inequality in Eq.~\ref{eq:positive_convex_polytope} a.s.) will be labeled as $y_i=0$ with a probability no less than $1-p_0$.

Considering the covariate space
\beq
\mathcal{T}^{-t}:=\left\{\xv_i:  \xv'_i\betav_{{\tilde{t}}} \ge 0 \text{ for }  {\tilde{t}}\neq t\right\},
\eeq
 where all the criteria except criterion $t$ of the expert  tend to be satisfied,   the decision boundary of stack-softplus regression in $\mathcal{T}^{-t}$ would be clearly influenced by the satisfactory level of criterion $t$, whose hyperplane partitions  $\mathcal{T}^{-t}$ into two parts as
   \beq\small
y_i =
\begin{cases} \vspace{.15cm}
\displaystyle 1
, & {\mbox{if  }} \displaystyle 
1-   
\left(1^{T+1}+\ln\bigg\{1^T+\ln\Big[1^{T-1}+\ldots
+\ln\big(1^t+e^{\xv_i'\betav^{(t)}} g_{t-1}\big)\Big]\bigg\}\right)^{-r} \!\!> p_0,\\
 \displaystyle 0
, &  \mbox{otherwise}, 
\end{cases}
 \eeq 
 for all $\xv_i\in\mathcal{T}^{-t}$. 
Let us define $g_t$ with $g_1 = 1$ and the recursion 
 $g_t=\ln(1+g_{t-1})$ 
 for $t=2,\ldots,T$, and define $h_t$ with $h_{T+1} = (1-p_0)^{-\frac{1}{r}}-1$ and the recursion 
$
 h_t=e^{h_{t+1}} -1 \notag
$
 for $t=T,T-1,\ldots,2$. 
 Using the definition of $g_t$ and $h_t$, 
%
 combining all the $T$ expert criteria, the confined space of stack-softplus regression specified in \eqref{eq:positive_convex_polytope} can be roughly related to a convex polytope, 
which is specified by the solutions to a set of $T$ 
 inequalities as
\begin{align}\label{eq:convex_polytope_recursive}
&\xv_i' \betav^{(t)} > \ln(h_{t})-\ln(g_{t-1}), ~t\in\{2,\ldots,T+1\}.
\end{align}
%
The convex polytope is enclosed by the intersection of $T$ $V$-dimensional hyperplanes, and since none of the $T$ criteria would be strongly violated inside the convex polytope, the label $y_i=1$ ($y_i=0$) would be assigned to an $\xv_i$ inside (outside) the convex polytope with a relatively high (low) probability. 

Unlike the confined space of sum-softplus regression defined in \eqref{eq:sum_ineuqality} that is bounded by a convex polytope defined in \eqref{eq:convex_polytope}, the convex polytope defined in \eqref{eq:convex_polytope_recursive} only roughly corresponds to the confined space of stack-softplus regression, as defined in \eqref{eq:positive_convex_polytope}. Nevertheless, the confined space defined in \eqref{eq:positive_convex_polytope} is referred to as a convex-polytope-like confined space, due to both its connection to the convex polytope in \eqref{eq:convex_polytope_recursive} and the fact that 
\eqref{eq:positive_convex_polytope} is likely to be violated if at least one of the $T$ criteria is strongly dissatisfied ($i.e.$, $e^{\xv_i'\betav^{(t)}}\rightarrow 0$ for some $t$).

\subsubsection{Illustration for stack-softplus regression}

 \begin{figure}[!t]
\begin{center}
 \includegraphics[width=0.75\columnwidth]{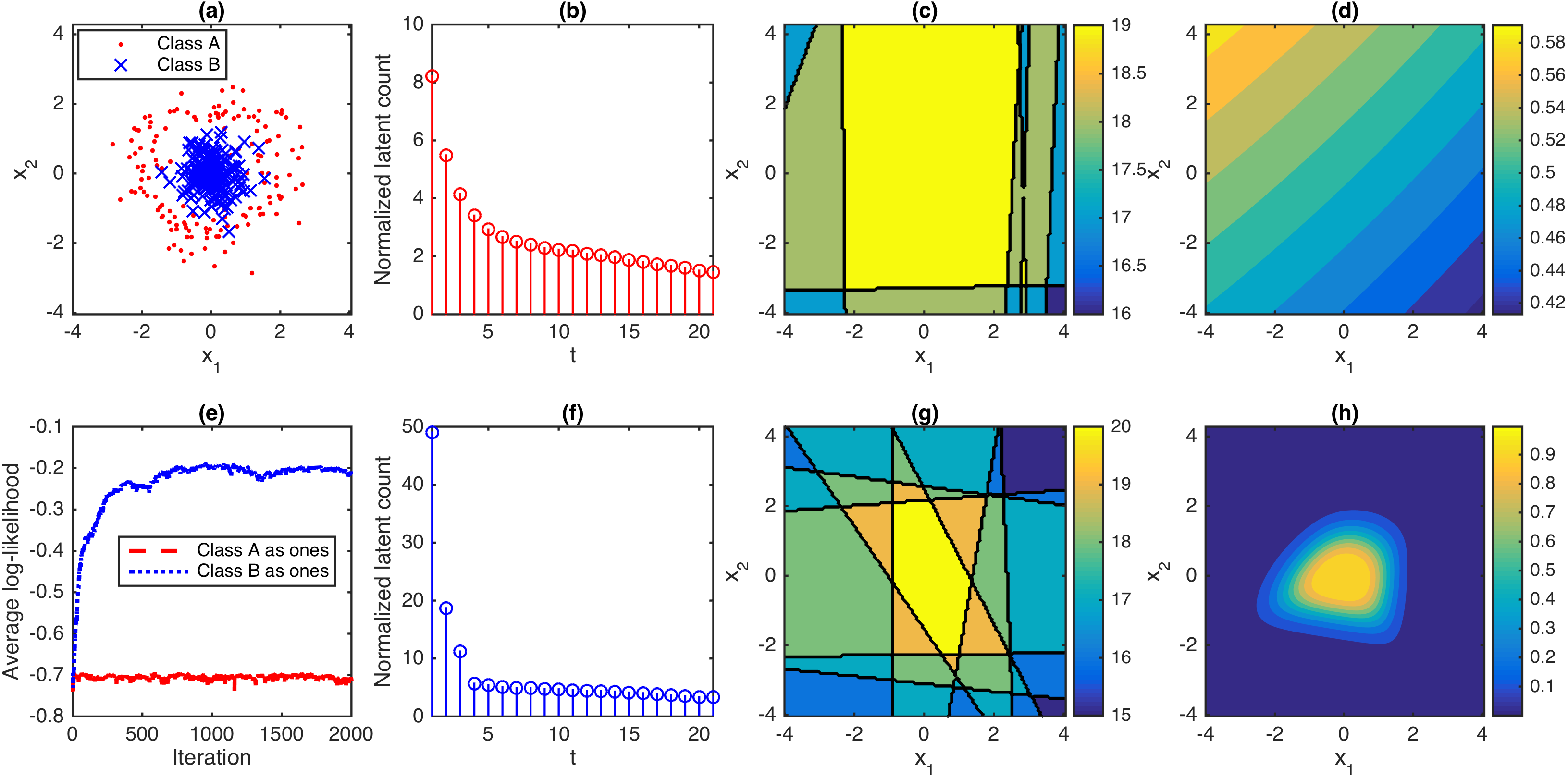}
 \vspace{-.2cm}
\end{center}
\vspace{-4.9mm}
\caption{\small\label{fig:circle_2}
Analogous figure to Fig. \ref{fig:circle_1} for stack-softplus regression with $T=20$ expert criteria, with the following differences: 
 (b) shows the average latent count per positive sample, $\sum_i m_i^{(t)}\big /\sum_i \delta(y_i=1)$, as a function of layer $t$, (c) shows a contour map, the value of each point of which represents how many inequalities specified in \eqref{eq:convex_polytope_recursive} are satisfied, and whose region with the values of $T=20$ corresponds to 
 the convex polytope enclosed by the intersections of the hyperplanes defined in \eqref{eq:convex_polytope_recursive},
 }

\begin{center}
 \includegraphics[width=0.75\columnwidth]{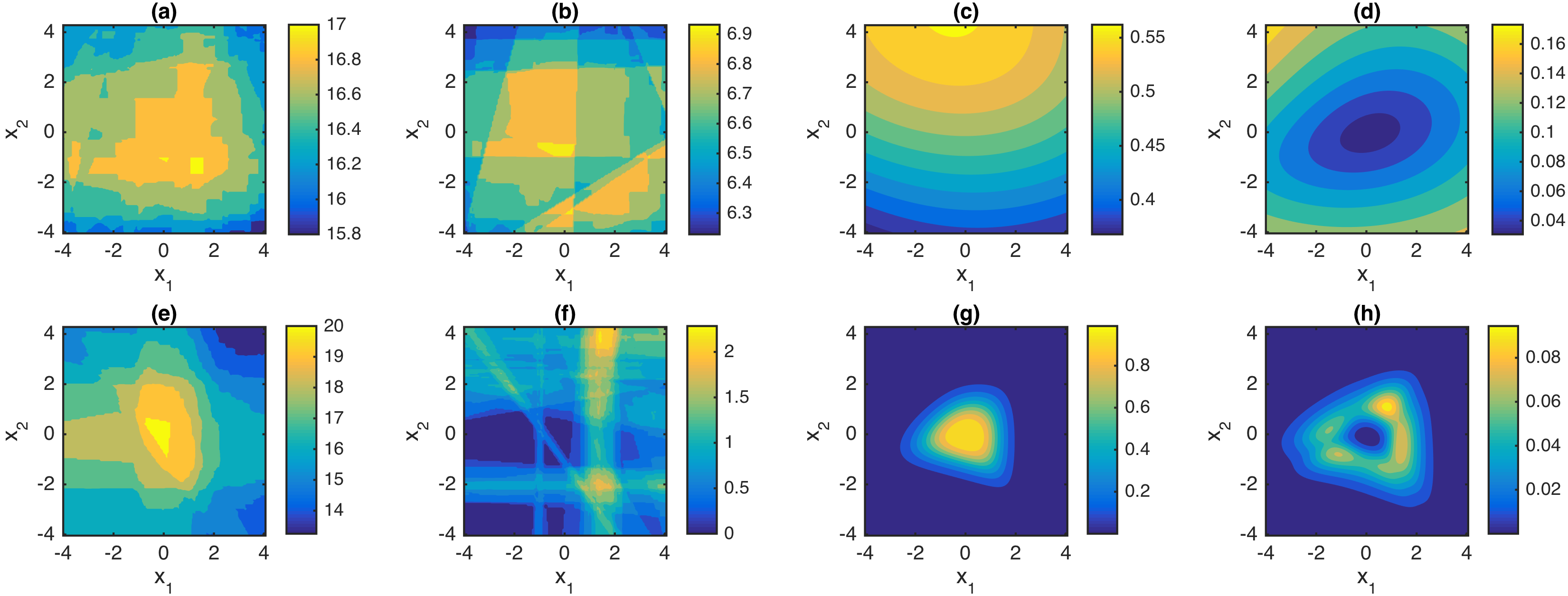}
 \vspace{-.2cm}
\end{center}
\vspace{-4.9mm}
\caption{\small\label{fig:circle_ave_2}
Analogous figure to Fig. \ref{fig:circle_ave_1} for stack-softplus regression, with the following differences: 
(a) and (b) show the contour maps of the posterior means and standard deviations, respectively, of the number of inequalities specified in \eqref{eq:convex_polytope_recursive} that are satisfied. (e)-(f) are analogous plots to (a)-(b) under the opposite labeling setting. 
}
\end{figure}


Let us examine how stack-softplus regression performs on the same data used in Fig. \ref{fig:circle_1}. 
When Class $B$ is labeled as ``1,'' as shown in Fig. \ref{fig:circle_2} (g), stack-softplus regression infers a convex polytope that encloses the space marked as $T=20$ using the intersection of all $T=20$ hyperplanes, each of which is defined as in \eqref{eq:convex_polytope_recursive}; 
and as shown in Fig. \ref{fig:circle_2} (h), it works well by using a convex-polytope-like confined space to enclose positive examples. However, as shown in Figs. \ref{fig:circle_2} (c)-(e), its performance deteriorates 
when the opposite labeling setting is used. 
Note that due to the model construction that introduces complex interactions between the $T$ hyperplanes, \eqref{eq:convex_polytope_recursive} can only roughly describe how 
a single hyperplane could influence the decision boundary determined by all hyperplanes. Thus it is not surprising that neither the convex polytope in Fig. \ref{fig:circle_2} (c), which encloses the space marked with the largest count there, nor the convex polytope in Fig. \ref{fig:circle_2} (g), which encloses the space marked with $T$, align well with the contour lines of $P(y_i=1\given \xv_i)=0.5$ in Figs. \ref{fig:circle_2} (d) and (h), respectively. 

While how the latent count $m_{\cdotv}^{(t)}$ decreases as $t$ increases does not indicate a clear cutoff point for the depth $T$, neither do we observe a clear sign of overfitting when $T$ is set as large as 100 in our experiments. Both Figs. \ref{fig:circle_2} (c) and (g) indicate that most of the hyperplanes are far from any data points and tend to vote ``Yes'' for all training data. 
The standard deviations shown in Figs. \ref{fig:circle_ave_2} (f) and (h) clearly indicate the uncertainties of stack-softplus regression on its decision boundaries and predictive probabilities in the covariate space. 

Like sum-softplus regression, stack-softplus regression also generalizes softplus and logistic regressions in that it uses the boundary of a confined space rather than a single hyperplane to partition the covariate space into two parts.
Unlike the convex-polytope-bounded confined space of sum-softplus regression 
 that 
favors placing negative examples inside it, the convex-polytope-like confined space of stack-softplus regression 
favors placing positive examples inside it.
While both sum- and stack-softplus regressions could be sensitive to how the data are labeled, their distinct behaviors under the same labeling setting motivate us to combine them together as SS-softplus regression, as described below.

\subsection{Sum-stack-softplus (SS-softplus) regression}

Note that if 
${T}=1$, SS-softplus regression reduces to sum-softplus regression; if $K=1$, 
it reduces to stack-softplus regression; and if $K=T=1$, it reduces to softplus regression, which further reduces to logistic regression if the weight of the single expert is fixed at $r=1$. 
To ensure that the SS-softplus regression  model is well defined in its infinite limit, we provide the following proposition and present the proof in Appendix \ref{sec:proof}. 
\begin{prop}\label{lem:finite1}
The infinite product in sum-stack-softplus regression  as $$e^{-\sum_{k=1}^\infty r_k\,\varsigma\left(\xv_i'\betav_k^{(2:T+1)}\right)} = \prod_{k=1}^\infty \left(\!1\!+\!e^{\xv_i'\betav_k^{({T}\!+\!1)}}\ln\!\bigg\{1\!+\!e^{\xv_i'\betav_k^{({T})}}\ln\Big[1\!+\!\ldots
\ln\big(1\!+\!e^{\xv_i'\betav_k^{(2)}}\big)\Big]\bigg\}\right)^{-r_k}$$ 
is smaller than one and has a finite expectation that is greater than zero.
\end{prop}

\subsubsection{Union of convex-polytope-like confined spaces}
We may consider SS-softplus regression as a 
multi-hyperplane
model that employs a committee, consisting of countably infinite experts, to make a decision, where 
each expert is equipped with $T$ criteria to be examined. The committee's distribution is obtained by convolving the distributions of countably infinite experts, each of which mixes $T$ stacked covariate-dependent gamma distributions. 
For each $\xv_i$,
the committee  votes ``Yes''  as long as at least one expert votes ``Yes,'' and an expert could vote ``Yes'' if and only if none of its $T$ criteria are strongly violated. 
Thus
the decision boundary of SS-softplus regression can be considered as a union of convex-polytope-like confined spaces that all favor placing positively labeled data inside them, as described below, with the proofs deferred to Appendix \ref{sec:proof}.


\begin{thm}\label{thm:union_polytope}
 For sum-stack-softplus regression, the confined space specified by the inequality $P(y_i=1\given \xv_i)=1-e^{-\lambda(\xv_i)}> p_0$, which can be expressed as
 \beq 
\lambda(\xv_i) = \sum_{k=1}^\infty r_k\, \varsigma\big(\xv'\betav_k^{(2)},\ldots,\xv'\betav_k^{(T+1)}\big)
> -\ln(1-p_0), \label{eq:SS-softplus_ineuqality}
 \eeq
encompasses the union of convex-polytope-like confined spaces, expressed as
 $$\mathcal{D}_{\star} = \mathcal{D}_1\cup \mathcal{D}_2 \cup\ldots,$$
where the $k$th convex-polytope-like confined space $\mathcal{D}_k$ 
is specified by
the inequality 
\beq
\small
\label{eq:Union_convex_polytope}
\xv_i'\betav_k^{({T}+1)} + \ln \ln\!\bigg\{1+e^{\xv_i'\betav_k^{({T})}}\ln\Big[1+\ldots
\ln\big(1+e^{\xv_i'\betav_k^{(2)}}\big)\Big]\bigg\} > \ln\big[(1-p_0)^{-\frac{1}{r_k}}-1\big]. 
\eeq

 \end{thm}
 
 \begin{cor}\label{cor:union_polytope}
 For sum-stack-softplus regression, the confined space specified by the inequality $P(y_i=1\given \xv_i)=1-e^{-\lambda(\xv_i)}\le p_0$ is bounded by $\bar{\mathcal{D}}_{\star} = \bar{\mathcal{D}}_1 \cap \bar{\mathcal{D}}_2\cap\ldots $
 \end{cor}
 
%
%

\begin{prop}\label{prop:union}
For any data point $\xv_i$ that resides inside 
the union of countably infinite convex-polytope-like confined spaces $\mathcal{D}_{\star} = \mathcal{D}_1\cup \mathcal{D}_2 \cup\ldots$, which means $\xv_i$ satisfies at least one of the inequalities in \eqref{eq:Union_convex_polytope}, 
it will be labeled under sum-stack-softplus regression with $y_i=1$ with a probability greater than $p_0$, and $y_i=0$ with a probability no greater than $1-p_0$.
 \end{prop} 

\subsubsection{Illustration for sum-stack-softplus regression}

Let us examine how SS-softplus regression performs on the same dataset used in Fig. \ref{fig:circle_1}. 
When Class $A$ is labeled as ``1,'' as shown in Figs. \ref{fig:circle_3} (b)-(c), SS-softplus regression infers about eight convex-polytope-like confined spaces, the intersection of six of which defines the boundary of the covariate space that separates 
 the points that violate all inequalities in~\eqref{eq:Union_convex_polytope} from the ones that satisfy at least one inequality in \eqref{eq:Union_convex_polytope}. The union of these convex-polytope-like confined spaces defines a confined covariate space, 
 which is included within the covariate space satisfying $P(y_i\given \xv_i)>0.5$, as shown in Fig. \ref{fig:circle_3} (d). 
 
When Class $B$ is labeled as ``1,'' as shown in Fig. \ref{fig:circle_3} (f)-(g), SS-softplus regression infers about six convex-polytope-like confined spaces, 
one of which defines the boundary of the covariate space that separates the points that violate all inequalities in \eqref{eq:Union_convex_polytope} from the others for the covariate space show in Fig. \ref{fig:circle_3} (g). The union of two convex-polytope-like confined spaces defines a confined covariate space, which is included in the covariate space with $P(y_i\given \xv_i)>0.5$, as shown in Fig. \ref{fig:circle_3} (h).
Figs. \ref{fig:circle_3} (f)-(g) also indicate that except for two convex-polytope-like confined spaces, the boundaries of all the other convex-polytope-like confined spaces are far from any data points and tend to vote ``No'' for all training data.
The standard deviations shown in Figs. \ref{fig:circle_ave_3} (b), (d), (f), and (h) clearly indicate the uncertainties of SS-softplus regression on 
classification decision boundaries and predictive probabilities. 

 \begin{figure}[!t]
\begin{center}
 \includegraphics[width=0.75\columnwidth]{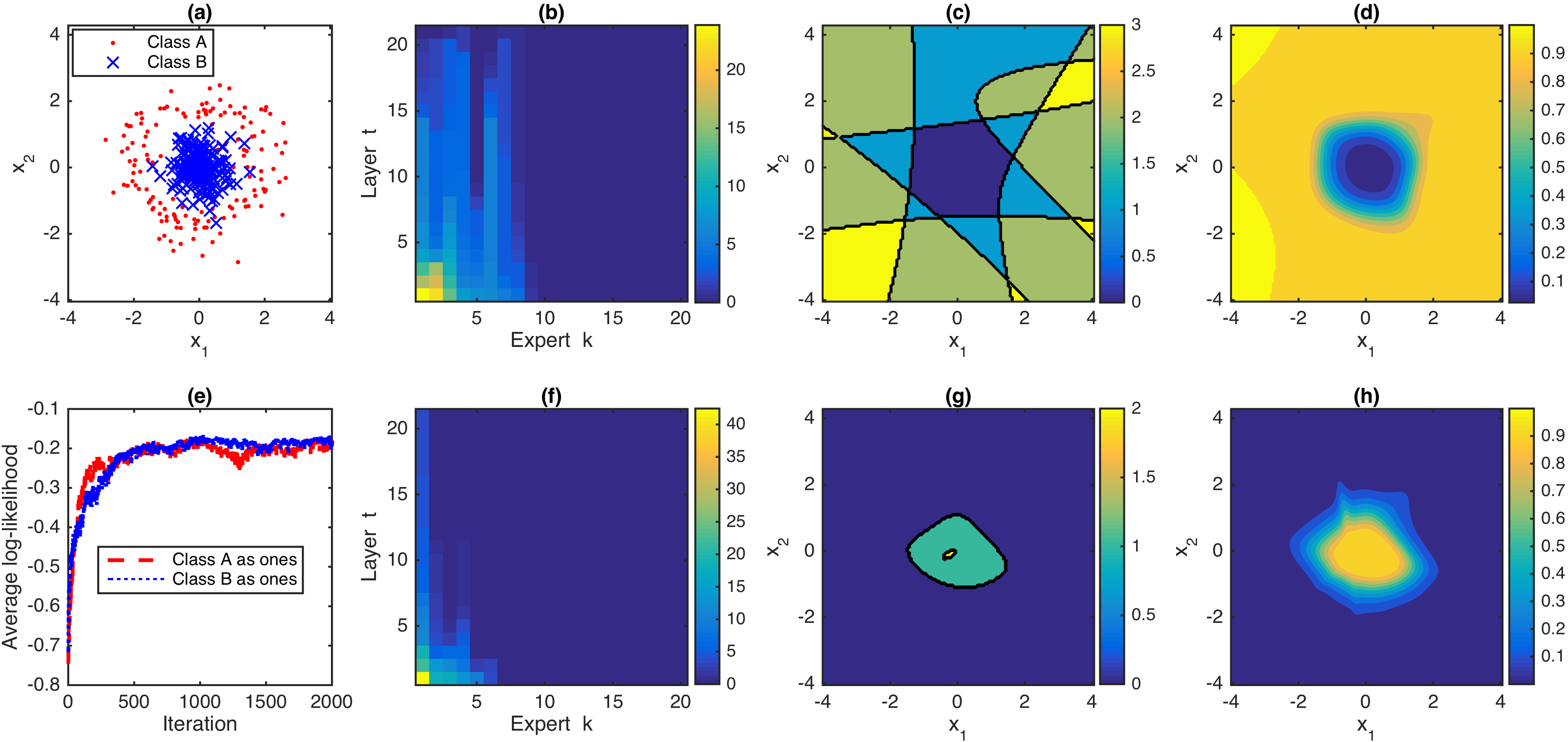}
 \vspace{-.2cm}
\end{center}
\vspace{-4.9mm}
\caption{\small\label{fig:circle_3}
%
Analogous figure 
to Figs. \ref{fig:circle_1} and \ref{fig:circle_2} for SS-softplus regression with $K_{\max}=20$ experts and $T=20$ criteria for each expert, with the following differences: 
 (b) shows the average latent count per positive sample, $\sum_i m_{ik}^{(t)}\big /\sum_i \delta(y_i=1)$, as a function of both the expert index $k$ and layer index $t$, where the experts are ordered based on the values of $\sum_i m_{ik}^{(1)}$, (c) shows a contour map, the value of each point of which represents how many inequalities specified in \eqref{eq:Union_convex_polytope} are satisfied, and whose region with nonzero values corresponds to 
 the union of convex-polytope-like confined spaces, each of which corresponds to an inequality defined in \eqref{eq:Union_convex_polytope},
 and (f) and (g) are analogous plots to (b) and (c) under the opposite labeling setting where data in Class $B$ are labeled as ``1.'' 
 }

\begin{center}
 \includegraphics[width=0.75\columnwidth]{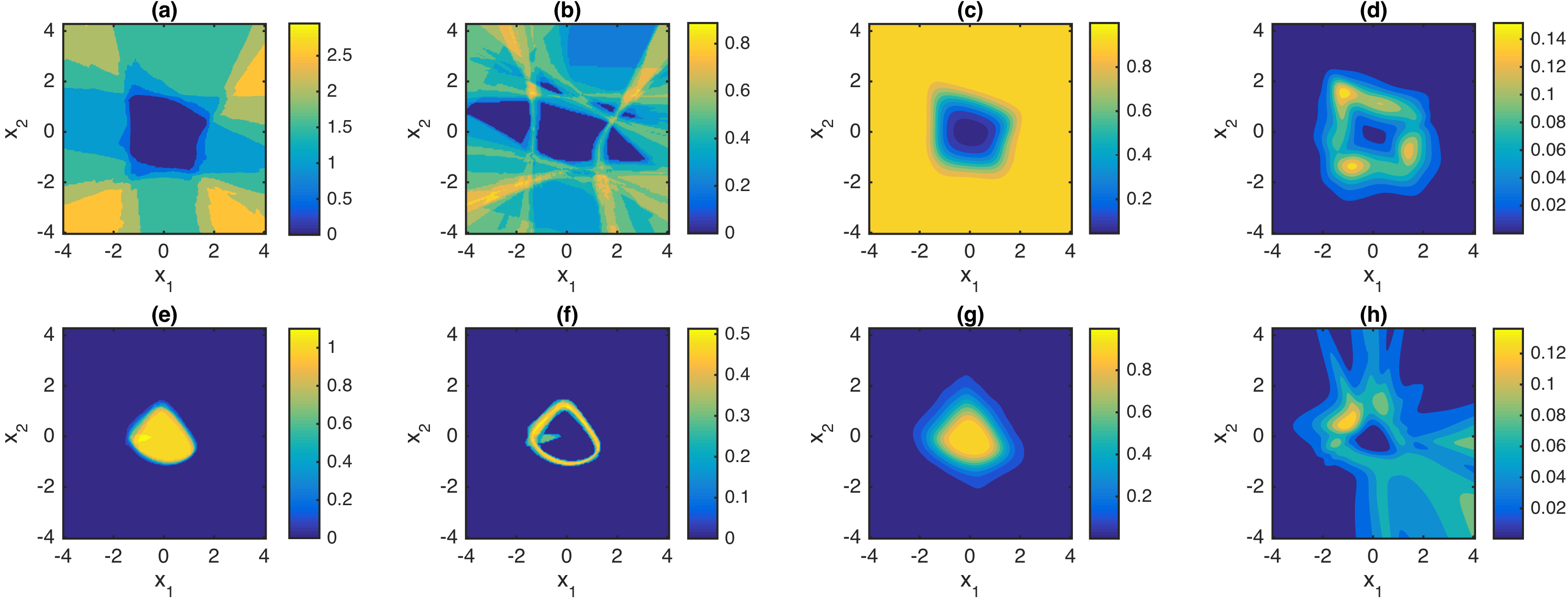}
 \vspace{-.2cm}
\end{center}
\vspace{-5.9mm}
\caption{\small\label{fig:circle_ave_3}
Analogous figure 
to 
Fig. \ref{fig:circle_ave_1} for SS-softplus regression, with the following differences: 
(a) and (b) show the contour maps of the posterior means and standard deviations, respectively, of the number of inequalities specified in \eqref{eq:Union_convex_polytope} that are satisfied. (e)-(f) are analogous plots to (a)-(b).
 }
\end{figure}

\vspace{-4mm}
\section{
Gibbs sampling via data augmentation and marginalization 
}\label{sec:inference}
\vspace{-2mm}



Since logistic, softplus, sum-softplus, and stack-softplus regressions can all be considered as special cases of SS-softplus regression, below we will focus on presenting the nonparametric Bayesian hierarchical  model and Bayesian inference for SS-softplus regression. 

The gamma process $G\sim\Gamma{\mbox{P}}(G_0,1/c)$  has an inherent shrinkage mechanism, as in the prior the number of atoms with weights larger than $\epsilon>0$ follows 
$
\mbox{Pois}\left(\gamma_0\int_{\epsilon}^\infty r^{-1}e^{-cr}dr\right),
$
whose mean is finite a.s., where $\gamma_0=G_0(\Omega)$ is the mass parameter of the gamma process.   
In practice, the atom with a tiny weight generally has a negligible impact on the final decision boundary of the model, 
hence one may 
 truncate either the weight to be above $\epsilon$ or the number of atoms to be below $K$. One may also follow \citet{Wolp:Clyd:Tu:2011} to use a reversible jump MCMC \citep{green1995reversible} strategy to adaptively truncate the number of atoms for a gamma process, which often comes with a high computational cost. For the convenience of implementation, 
 we truncate the number of atoms in the gamma process to be $K$ by choosing a finite discrete base measure as $G_0=\sum_{k=1}^K \frac{\gamma_0} K \delta_{\alpha_{k}}$, where
$K$ will be set sufficiently large to achieve a good approximation to the truly countably infinite model.
 
 We express the truncated SS-softplus regression model using \eqref{eq:DICLR_model} together with
 \begin{align} \small
&r_k\sim\mbox{Gamma}(\gamma_0/K,1/c_0),~\gamma_0\sim\mbox{Gamma}(a_0,1/b_0),~c_0\sim\mbox{Gamma}(e_0,1/f_0),\notag\\
&~~~~~~~~~~~~~~~~~\betav^{(t)}_{k}\sim\prod_{v=0}^{V}\mathcal{N}(0,\alpha_{vtk}^{-1}),~\alpha_{vtk}\sim\mbox{Gamma}(a_t,1/b_t),
\label{eq:ICNBE_finite}
\end{align}
where $t\in\{2,\ldots,T+1\}$. Related to \citet{RVM}, the normal gamma construction in (\ref{eq:ICNBE_finite})  is used to promote sparsity on the regression coefficients $\betav^{(t)}_{k}$.
We derive Gibbs sampling 
 by exploiting local conjugacies under a series of data augmentation and marginalization techniques. 
We comment here that while the proposed Gibbs sampling algorithm is a batch learning algorithm that processes all training data samples in each  iteration, the local conjugacies  revealed under data augmentation and marginalization may be of significant value in developing efficient mini-batch based online learning algorithms, including those based on stochastic gradient MCMC \citep{welling2011bayesian,girolami2011riemann,patterson2013stochastic,ma2015complete} and stochastic variation inference \citep{hoffman2013stochastic}. We leave the maximum likelihood, maximum a posteriori, (stochastic) variational Bayes inference, and stochastic gradient MCMC for softplus regressions for future research. 



For a model with ${T }=1$, 
we exploit the data augmentation techniques developed for the BerPo link in \citet{EPM_AISTATS2015} to sample $m_{i}$, these developed for the Poisson and multinomial distributions \citep{Dunson05bayesianlatent,BNBP_PFA_AISTATS2012} to sample $m_{ik}$, these developed for the NB distribution in \citet{NBP2012} to sample $r_{k}$ and $\gamma_0$, and these developed for logistic regression in \citet{LogitPolyGamma} and further generalized to NB regression in \citet{LGNB_ICML2012} and \citet{polson2013bayesian} to sample $\betav_{k}$. We exploit 
local conjugacies to sample all the other model parameters. 
For a model with ${T}\ge 2$, we further generalize the inference technique developed for the gamma belief network in \citet{PGBN_NIPS2015} to sample the model parameters of deep hidden layers. 
Below we provide a theorem, related to Lemma~1 for the gamma belief network in \citet{PGBN_NIPS2015}, 
to show that each regression coefficient vector can be linked to latent counts under NB regression. 
Let $m\sim\mbox{SumLog}(n,p)$ represent the sum-logarithmic distribution described in \citet{NBP_CountMatrix}, Corollary \ref{cor:sumlog} further shows an alternative representation of   \eqref{eq:DICLR_model}, the hierarchical model of SS-softplus regression, where all the covariate-dependent gamma distributions are marginalized out.

\begin{thm} 
\label{cor:PGBN} Let us denote $p_{ik}^{(t)} = 1-e^{-{q}_{ik}^{(t)}}$, $i.e.$, ${q}_{ik}^{(t)} = -\ln(1-p_{ik}^{(t)})$,
and  $\theta_{ik}^{({T }+1)}=r_k$. With ${q}_{ik}^{(1)}: = 1$ 
and
\beq
{q}_{ik}^{(t+1)} := \ln\left( 1+ {q}_{ik}^{(t)}e^{\xv_i\betav_{k}^{(t+1)}} \right)  \label{eq:lambda}
\eeq
for $t=1,\ldots,T$, which means 
\beqs
\displaystyle
&{q}_{ik}^{(t+1)} = \varsigma\big(\xv_i'\betav_k^{(2)},\ldots,\xv_i'\betav_k^{(t+1)}\big)\notag\\
=
&\displaystyle \ln\left(\!1\!+\!e^{\xv_i'\betav_k^{({t}\!+\!1)}}\ln\!\bigg\{1\!+\!e^{\xv_i'\betav_k^{({t})}}\ln\Big[1\!+\!\ldots
\ln\big(1\!+\!e^{\xv_i'\betav_k^{(2)}}\big)\Big]\bigg\}\right),
 \label{eq:q_ik}
\eeqs
one may find latent counts $m_{ik}^{(t)}$
that are
 connected to the regression coefficient vectors as
\beq
m_{ik}^{(t)}\sim\emph{\mbox{NB}}(\theta_{ik}^{(t+1)},~~ 1 - e^{-{q}_{ik}^{(t+1)}})
 = \emph{\mbox{NB}}\left(\theta_{ik}^{(t+1)}, \frac{1}{1+e^{-\xv_i'\betav_{k}^{(t+1)} - \ln ({q}_{ik}^{(t)})}}\right) .
\label{eq:deepPFA_aug1}
\eeq
\end{thm}

\begin{cor}\label{cor:sumlog}
With $q_{ik}^{(t)} = -\ln(1-p_{ik}^{(t)})$ defined as in \eqref{eq:q_ik} and hence $p_{ik}^{(t)}=1-e^{-q_{ik}^{(t)}}$, the  hierarchical model of sum-stack-softplus regression can also be expressed as
\small\begin{align} \small
&m_{ik}^{(T+1)}\sim\emph{\mbox{Pois}}(r_k q_{ik}^{(T+1)}),~r_k\sim\emph{\mbox{Gamma}}(\gamma_0/K,1/c_0),\notag\\&m_{ik}^{(T)}\sim\emph{\mbox{SumLog}}(m_{ik}^{(T+1)}, p_{ik}^{(T+1)}),~\betav^{({T }+1)}_{k}\sim\prod_{v=0}^{V}\mathcal{N}(0,\alpha_{v({T }+1)k}^{-1}),
\notag\\
&~~~~~~~~~~~~~~~~~~~~~~~~~~~~~~~~~\tiny{\cdots}\notag\\
&m_{ik}^{(t)}\sim\emph{\mbox{SumLog}}(m_{ik}^{(t+1)}, p_{ik}^{(t+1)}),~\betav^{(t+1)}_{k}\sim\prod_{v=0}^{V}\mathcal{N}(0,\alpha_{v(t+1)k}^{-1}),
\notag\\
&~~~~~~~~~~~~~~~~~~~~~~~~~~~~~~~~~\tiny{\cdots}\notag\\
y_i = \delta(m_i&\ge 1),~m_i
=\sum_{k=1}^K m^{(1)}_{ik},~m^{(1)}_{ik} \sim \emph{\mbox{SumLog}}(m_{ik}^{(2)}, p_{ik}^{(2)}),~\betav^{(2)}_{k}\sim\prod_{v=0}^{V}\mathcal{N}(0,\alpha_{v2k}^{-1}),
\label{eq:ICNBE_finite_1}
\end{align}\normalsize

\end{cor}


We outline Gibbs sampling  in Algorithm \ref{alg:1} of Appendix \ref{app:setting}, where to save computation, 
we consider setting $K_{\max}$ as the upper-bound of the number of experts and
 deactivating experts assigned with zero counts during MCMC iterations. 
 We provide several additional model properties in Appendix \ref{app:T} to describe how the latent counts propagate across layers, 
which may be used to decide how to set the number of layers $T$. 
 For simplicity, we consider the number of criteria for each expert as a parameter that determines the model capacity and we fix it as $T$ for all experts in this paper. 
 We defer the details of all Gibbs sampling update equations to Appendix \ref{app:sampling}, in which we also describe in detail how to ensure numerical stability in a finite precision machine. Note that except for the sampling of $\{m^{(1)}_{ik}\}_k$, the sampling of all the other parameters of different experts are embarrassingly parallel.

\vspace{-4mm}
\section{Example Results} \label{sec:results}
\vspace{-2mm}

We compare softplus regressions with logistic regression, 
Gaussian radial basis function (RBF) kernel support vector machine (SVM) \citep{boser1992training,vapnik1998statistical,scholkopf1999advances}, relevance vector machine (RVM) \citep{RVM}, adaptive multi-hyperplane machine (AMM) \citep{wang2011trading}, and convex polytope machine (CPM) \citep{kantchelian2014large}. Except for logistic regression that is a linear classifier, both kernel SVM and RVM are widely used nonlinear classifiers relying on the kernel trick, and both AMM and CPM 
use  the intersection of multiple hyperplanes to construct their decision boundaries. We discuss the connections between softplus regressions and previous work in Appendix \ref{sec:discussion}.

Following \citet{RVM}, we consider the following datasets: 
banana, breast cancer, titanic, waveform, german, and image. 
For each of these six datasets, we consider the first ten predefined random training/testing partitions, and report both the mean and standard deviation of the testing classification  errors. 
Since these datasets, originally provided by \citet{ratsch2001soft}, were no longer available on the authors' websites, we use the version provided by \citet{Diethe15}. 
We also consider two additional datasets---ijcnn1 and a9a---that come with a default training/testing partition, for which we report the results of logistic regression, SVM, and RVM based on a single trial, and report the results of all the other algorithms based on five independent trials with different random initiations. We summarize in Tab. \ref{tab:data} of Appendix \ref{app:setting} the basic information of these  benchmark datasets. 


Since the decision boundaries of all 
softplus regressions 
depend on whether 
the two 
classes are labeled as ``1'' and ``0'' or labeled as ``0'' and ``1,'' 
 we consider
%
repeating the same softplus regression algorithm twice, using both 
$y_{i1}\sim\mbox{BerPo}[\lambda_1(\xv_i)] \text{ and } y_{i2}\sim\mbox{BerPo}[\lambda_2(\xv_i)],
$
where $y_{i1}$
and $y_{i2}:=1-y_{i1}$ are the labels under two opposite labeling settings.
We combine them to 
the following predictive distribution
$ 
y_i \given \xv_i \sim 
\mbox{Bernoulli} \left[(1-e^{-\lambda_1(\xv_i)}+e^{-\lambda_2(\xv_i)})/2\right],\label{eq:fusing}
$ 
%
which no longer depends on how the data are labeled. If we set $p_0=0.5$ as the probability threshold to make binary decisions, then $y_i$ would be labeled as ``1'' if $\lambda_1(\xv_i)>\lambda_2(\xv_i)$ and labeled as ``0'' otherwise. This simple strategy to train the same asymmetric model under two opposite labeling settings and combine their results together 
is related to the one used in \citet{kantchelian2014large}, which, however, lacks of 
probabilistic interpretation.
We leave more sophisticated training and combination strategies 
to future study.

%
%
%
%
%

For all datasets, we consider 1) softplus regression, which generalizes logistic regression with $r\neq 1$, 
2) sum-softplus regression, which reduces to softplus regression if the number of experts is $K=1$, 3) stack-softplus regression, which reduces to softplus regression if the number of layers is $T=1$, and 4) sum-stack-softplus (SS-softplus) regression, which reduces to sum-softplus regression if $T=1$, to stack-softplus regression if $K=1$, and to softplus regression if $K=T=1$. For sum-softplus regresion, we set the upperbound on the number of experts as $K_{\max}=20$, for deep softplus regression, we consider $T\in\{1,2,3,5,10\}$, and 
 for SS-softplus regression, we set $K_{\max}=20$ and consider $T\in\{1,2,3,5,10\}$. 
For all softplus regressions, we consider 5000 Gibbs sampling iterations and record the maximum likelihood sample found during the last 2500 iterations as the point estimates of $r_k$ and $\betav_{k}^{(t)}$, which are used for out-of-sample predictions. 
We set $a_0=b_0=0.01$, $e_0=f_0=1$, and $a_t = b_t = 10^{-6}$ for $t\in\{2,\ldots,T+1\}$. As in 
Algorithm \ref{alg:1} shown in Appendix \ref{app:setting}, we 
 deactivate inactive experts for iterations in $I_{Prune}=\{525,575,\ldots,4975\}$.
 For a fair comparison, to ensure that the same training/testing partitions are used for all algorithms for all datasets, we report the results by using either widely used open-source software packages or the code made public available by the original authors.  We describe in Appendix \ref{app:setting} the settings of all the algorithms that are used for comparison.

\subsection{Illustrations}

 \begin{figure}[!t]
\begin{center}
 \includegraphics[width=0.835
 \columnwidth]{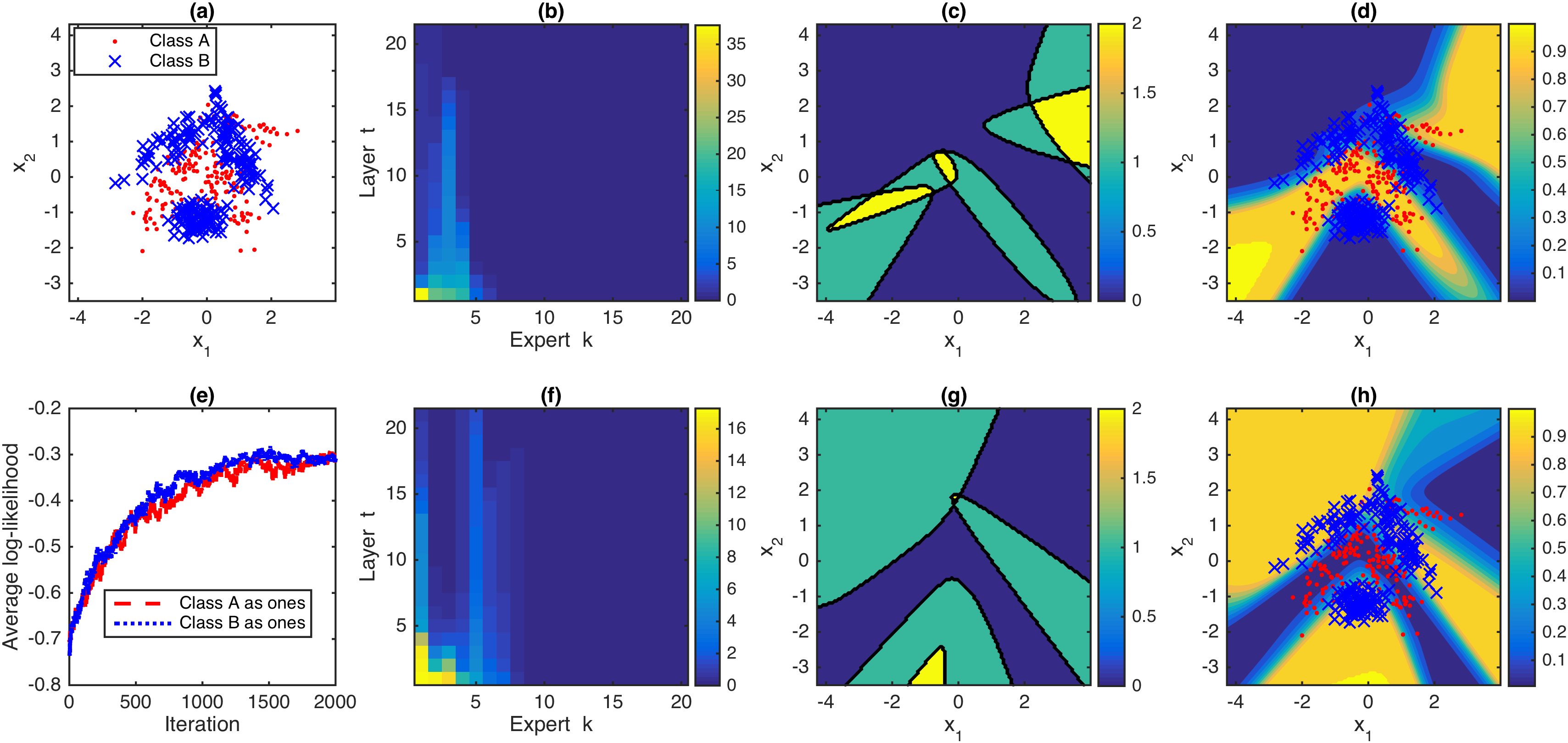}
\vspace{-.3cm}
\end{center}
\vspace{-4.9mm}
\caption{\small\label{fig:banana}
Analogous figure to Fig. \ref{fig:circle_3} for the ``banana'' dataset.}
\begin{center}
 \includegraphics[width=0.83\columnwidth]{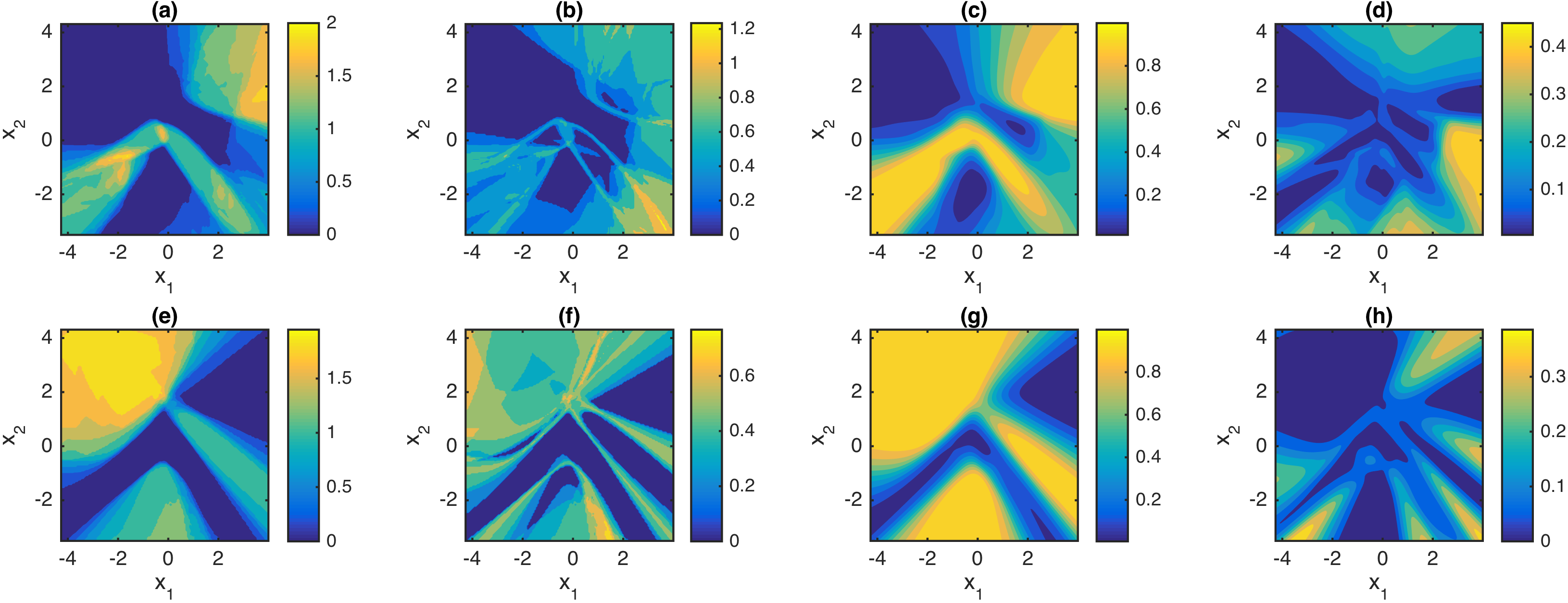}
\vspace{-.3cm}
\end{center}
\vspace{-4.9mm}
\caption{\small\label{fig:banana_ave}
Analogous figure to Fig. \ref{fig:circle_ave_3} for the ``banana'' dataset.
}
\end{figure}


With a  synthetic dataset, Figs. \ref{fig:circle_1}-\ref{fig:circle_ave_3} illustrate the distinctions and connections between the sum-, stack-, and SS-softplus regressions. While both sum- and stack-softplus could work well for the synthetic dataset if the two classes are labeled in their preferred ways, as shown in Figs. \ref{fig:circle_1} and \ref{fig:circle_2}, SS-softplus regression, as shown in Fig. \ref{fig:circle_3}, works well regardless of how the data are labeled. 
To further illustrate how the distinct, but complementary, behaviors of the sum- and stack-softplus regressions are combined together in SS-softplus regression, let us examine how SS-softplus regression performs on the banana dataset shown in Fig.~\ref{fig:banana}~(a). 
When Class $A$ is labeled as ``1,'' as shown in Figs. \ref{fig:banana} (b)-(c), SS-softplus regression infers about six convex-polytope-like confined spaces, the intersection of five of which defines the boundary of the covariate space that separates 
 the points that satisfy at least one inequality in \eqref{eq:Union_convex_polytope}  from the ones that violate all inequalities in \eqref{eq:Union_convex_polytope}. The union of these convex-polytope-like confined spaces defines a confined covariate space, 
 which is included within the covariate space satisfying $P(y_i\given \xv_i)>0.5$, as shown in Fig. \ref{fig:banana} (d). 
 
When Class $B$ is labeled as ``1,'' as shown in Fig. \ref{fig:banana} (f)-(g), SS-softplus regression infers about eight convex-polytope-like confined spaces, 
three of which define the boundary of the covariate space that separates the points that satisfy at least one inequality in \eqref{eq:Union_convex_polytope}  from the others for the covariate space show in Fig. \ref{fig:banana} (g). The union of four convex-polytope-like confined spaces defines a confined covariate space, which is included in the covariate space with $P(y_i\given \xv_i)>0.5$, as shown in Fig. \ref{fig:banana} (h).
Figs. \ref{fig:banana} (f)-(g) also indicate that except for four convex-polytope-like confined spaces, 
all the other inferred convex-polytope-like confined spaces are far away from and tend to vote ``No'' for all training data.
 The standard deviations shown in Figs. \ref{fig:banana_ave} (b), (d), (f), and (h) 
  indicate the uncertainties of SS-softplus regression on classification decision boundaries and predictive probabilities in the covariate space.



In Figs. \ref{fig:xor}-\ref{fig:dbmoon_ave} of Appendix \ref{app:setting}, we  further illustrate SS-softplus regression   
on an exclusive-or (XOR) dataset 
and a double-moon dataset used in \citet{haykin2009neural}. 
For the banana, XOR, and double-moon datasets, where the two classes cannot be well separated by a single convex-polytope-like confined space, 
neither sum- nor stack-softplus regressions work well regardless of how the data are labeled,  whereas SS-softplus regression infers the union of multiple convex-polytope-like confined spaces that successfully separates the two classes.

\subsection{Classification performance on benchmark data}

\begin{table}[t!] 
\footnotesize
\caption{\small Comparison of classification errors of logistic regression (LR), RBF kernel support vector machine (SVM), relevance vector machine (RVM), adaptive multi-hyperplane machine (AMM), convex polytope machine (CPM), softplus regression, sum-softplus (sum-$\varsigma$) regression with $K_{\max}=20$, stack-softplus (stack-$\varsigma$) regression with $T=5$, and SS-softplus (SS-$\varsigma$) regression with $K_{\max}=20$ and $T=5$. 
Displayed in each column of the last row is the average of the classification errors of an algorithm normalized by
those of kernel SVM.
}\label{tab:Error} 
\centering 
\begin{tabular}{c|ccccccccc} 
\toprule 
 Dataset & LR & SVM & RVM & AMM & CPM & \!softplus &\! sum-$\varsigma$\! & \!\!stack-$\varsigma$ ($T$=5)\! & \!\!SS-$\varsigma$ ($T$=5) \\ 
\midrule 
\footnotesize 
banana & $47.76$ & $\mathbf{10.85}$ & $11.08$ & $18.76$ & $21.39$ & $47.87$ & $30.78$ & $33.21$ & $11.89$ \\ 
 & $\pm 4.38$ & $\pm 0.57$ & $\pm 0.69$ & $\pm 4.09$ & $\pm 1.72$ & $\pm 4.36$ & $\pm 8.68$ & $\pm 5.76$ & $\pm 0.61$ \\ 
\midrule 
 breast  & $28.05$ & $28.44$ & $31.56$ & $31.82$ & $32.08$ & $28.70$ & $30.13$ & $\mathbf{27.92}$ & $28.83$ \\ 
cancer & $\pm 3.68$ & $\pm 4.52$ & $\pm 4.66$ & $\pm 4.47$ & $\pm 4.29$ & $\pm 4.76$ & $\pm 4.23$ & $\pm 3.31$ & $\pm 3.40$ \\ 
\midrule 
titanic & $22.67$ & $22.33$ & $23.20$ & $28.85$ & $22.37$ & $22.53$ & $22.48$ & $22.71$ & $\mathbf{22.29}$ \\ 
 & $\pm 0.98$ & $\pm 0.63$ & $\pm 1.08$ & $\pm 8.56$ & $\pm 0.45$ & $\pm 0.43$ & $\pm 0.25$ & $\pm 0.70$ & $\pm 0.80$ \\
\midrule 
waveform & $13.33$ & $\mathbf{10.73}$ & $11.16$ & $11.81$ & $12.76$ & $13.62$ & $11.51$ & $12.25$ & $11.69$ \\ 
 & $\pm 0.59$ & $\pm 0.86$ & $\pm 0.72$ & $\pm 1.13$ & $\pm 1.17$ & $\pm 0.71$ & $\pm 0.65$ & $\pm 0.69$ & $\pm 0.69$ \\
\midrule 
german & $23.63$ & $23.30$ & $23.67$ & $25.13$ & $25.03$ & $24.07$ & $23.60$ & $\mathbf{22.97}$ & $24.23$ \\ 
 & $\pm 1.70$ & $\pm 2.51$ & $\pm 2.28$ & $\pm 3.73$ & $\pm 2.49$ & $\pm 2.11$ & $\pm 2.39$ & $\pm 2.22$ & $\pm 2.46$ \\
\midrule 
image & $17.53$ & $2.84$ & $3.82$ & $3.82$ & $3.25$ & $17.55$ & $3.50$ & $7.97$ & $\mathbf{2.73}$ \\ 
 & $\pm 1.05$ & $\pm 0.52$ & $\pm 0.59$ & $\pm 0.87$ & $\pm 0.41$ & $\pm 0.75$ & $\pm 0.73$ & $\pm 0.52$ & $\pm 0.53$ \\
\bottomrule
\begin{tabular}{@{}c@{}}\scriptsize Mean of SVM\\ \scriptsize normalized errors\end{tabular} & $2.472$ & $1$ & $1.095$ & $1.277$ & $1.251$ & $2.485$ & $1.370$ & $1.665$ & $1.033$ \\ 
\end{tabular}

\footnotesize 
\centering 
\setlength{\tabcolsep}{5.5pt} 
\caption{\small Analogous table to Tab. \ref{tab:Error} for comparing the number of experts (times the number of hyperplanes per expert), 
where an expert contains $T$ hyperplanes for both stack- and SS-softplus regressions and contains a single hyperplane/support vector for all the others. The computational complexity for out-of-sample prediction is about linear in the number of hyperplanes/support vectors. 
Displayed in each column of the last row is the average of the number of experts (times the number of hyperplanes per expert) of an algorithm normalized by
those of RBF kernel SVM.
 } 
\label{tab:K} 
\begin{tabular}{c|ccccccccc} 
\toprule 
 Dataset & LR & SVM & RVM & AMM & CPM & softplus & sum-$\varsigma$ & stack-$\varsigma$\! ($T$=5) & SS-$\varsigma$\! ($T$=5) \\ 
\midrule 
 
banana & $1$ & $129.20$ & $22.30$ & $9.50$ & $14.60$ & $2$ & $3.70$ & $2~(\times 5)$ & $7.60~(\times 5)$ \\ 
 & & $\pm 32.76$ & $\pm 26.02$ & $\pm 2.80$ & $\pm 7.49$ & & $\pm 0.95$ & & $\pm 1.17~(\times 5)$ \\ 
\midrule 
breast & $1$ & $115.10$ & $24.80$ & $13.40$ & $12.00$ & $2$ & $3.10$ & $2~(\times 5)$ & $6.40~(\times 5)$ \\ 
 cancer & & $\pm 11.16$ & $\pm 28.32$ & $\pm 0.84$ & $\pm 8.43$ & & $\pm 0.74$ & & $\pm 1.43~(\times 5)$ \\
\midrule 
titanic & $1$ & $83.40$ & $5.10$ & $14.90$ & $5.20$ & $2$ & $2.30$ & $2~(\times 5)$ & $4.00~(\times 5)$ \\ 
 & & $\pm 13.28$ & $\pm 3.03$ & $\pm 3.14$ & $\pm 2.53$ & & $\pm 0.48$ & & $\pm 0.94~(\times 5)$ \\ 
\midrule 
waveform & $1$ & $147.00$ & $21.10$ & $9.50$ & $6.40$ & $2$ & $4.40$ & $2~(\times 5)$ & $8.90~(\times 5)$ \\ 
 & & $\pm 38.49$ & $\pm 10.98$ & $\pm 1.18$ & $\pm 2.27$ & & $\pm 0.84$ & & $\pm 2.33~(\times 5)$ \\
\midrule 
german & $1$ & $423.60$ & $11.00$ & $18.80$ & $8.80$ & $2$ & $6.70$ & $2~(\times 5)$ & $14.70~(\times 5)$ \\ 
 & & $\pm 55.02$ & $\pm 3.20$ & $\pm 1.81$ & $\pm 7.79$ & & $\pm 0.95$ & & $\pm 1.77~(\times 5)$ \\ 
\midrule 
image & $1$ & $211.60$ & $35.80$ & $10.50$ & $23.00$ & $2$ & $11.20$ & $2~(\times 5)$ & $17.60~(\times 5)$ \\ 
 & & $\pm 47.51$ & $\pm 9.19$ & $\pm 1.08$ & $\pm 6.75$ & & $\pm 1.32$ & & $\pm 1.90~(\times 5)$ \\ 
\bottomrule 
\begin{tabular}{@{}c@{}}\scriptsize Mean of SVM\\ \scriptsize normalized $K$\end{tabular} & $0.007$ & $1$ & $0.131$ & $0.088$ & $0.075$ & $0.014 $ & $0.030$ & $0.014~(\times 5)$ & $0.057~(\times 5)$ \\ 
\end{tabular}

\end{table}

We summarize in Tab. \ref{tab:Error} the results for the first six benchmark datasets described in Tab.~\ref{tab:data}, for each of which we report the results based on the first ten predefined random training/testing partitions. Overall for these six datasets, the RBF kernel SVM has the highest average out-of-sample prediction accuracy, followed closely by SS-softplus regression, whose mean of the errors normalized by these of of the SVM is as small as 1.033, and then by RVM, whose mean of normalized errors is 1.095. Overall, logistic regression does not perform well, which is not surprising as it is a linear classifier that uses a single hyperplane to partition the covariate space into two halves to separate one class from the other. Softplus regression, which uses an additional parameter over logistic regression, 
 fails to reduce the classification errors of logistic regression; both sum-softplus regression, a multi-hyperplane generalization using the convolution operation, and stack-softplus regression, a multi-hyperplane generalization using the stacking operation, clearly reduce the classification errors;  and SS-regression that combines both the convolution and stacking operations further improves the overall performance. Both CPM and AMM  perform similarly to sum-softplus regression, which is not surprising given their  connections   
discussed in Appendix \ref{sec:CPM}. 

For out-of-sample prediction, the computation of a classification algorithm generally increases linearly in the number of used hyperplanes or support vectors. We summarize the number of experts (times the number of hyperplanes  per expert if that number is not one) in Tab. \ref{tab:K}, which indicates that in comparison to SVM that consistently requires the most number of experts (each expert corresponds to a support vector for SVM), the RVM, AMM, CPM, and the three proposed multi-hyperplane softplus regressions all require significantly less time for predicting the class label of a new data sample. It is also interesting to notice that the number of hyperplanes automatically inferred from the data by sum-softplus regression is 
generally smaller than
 the ones of AMM and CPM, both of which are selected through cross validations. Note that the number of active experts, defined as the value of $\sum_{k}\delta(\sum_i m^{(1)}_{ik}  >0)$, inferred by both sum- and SS-softplus regressions shown in Tab. \ref{tab:K} will be further reduced if we only take into consideration the experts whose weights are larger than a certain threshold, such as those with $r_k>0.001$ for $k\in\{1,\ldots,K_{\max}\}$. 

Except for banana, a two-dimensional dataset, sum-softplus regression performs similarly to both AMM and CPM; and except for banana and image, stack-softplus regression performs similarly to both AMM and CPM. These results are not surprising as CPM, closely related to AMM, uses a convex polytope, defined as the intersection of multiple hyperplanes, to enclose one class, whereas the classification decision boundaries of sum-softplus regression, defined by the interactions of multiple hyperplanes via the sum-softplus function, can be bounded within a convex polytope that encloses negative examples, and that of stack-softplus regression can be related to a convex-polytope-like confined space that encloses positive examples. Note that while both sum- and stack-softplus regressions can partially remedy their sensitivity to how the data are labeled by combining the results obtained under two opposite labeling settings, 
the decision boundaries of them and those of both AMM and CPM  are still restricted to a confined space related to a single convex polytope, which may be used to explain why on both banana and image, as  well as on the XOR and double-moon datasets shown in Appendix \ref{app:setting}, 
they all clearly underperform SS-softplus regression, which separates two classes using the union of  convex-polytope-like confined spaces. 

For breast cancer, titanic, and german, all classifiers have comparable classification errors, suggesting minor or no advantages of using a  nonlinear classifier on them. For these three datasets, it is interesting to notice that, as shown in Figs. \ref{fig:rk}-\ref{fig:SS-softplus_rk}, sum- and SS-softplus regressions infer no more than two and three experts, respectively, with non-negligible weights under both labeling settings. 
These interesting connections imply that for two linearly separable classes,  while providing no obvious benefits but also no clear harms, both sum- and SS-softplus regressions tend to infer a few active experts, and both stack- and SS-softplus regressions exhibit no clear sign of overfitting as  the number of expert criteria $T$ increases. 

Whereas for banana, waveform, and image, all nonlinear  classifiers clearly outperform logistic regression, and as shown in Figs. \ref{fig:rk}-\ref{fig:SS-softplus_rk}, sum- and SS-softplus regressions infer at least two and four experts, respectively, with non-negligible weights under at least one of the two  labeling settings.  These interesting connections imply that  for two classes not linearly separable, both sum- and SS-softplus regressions may significantly outperform logistic regression by inferring a sufficiently large number of active experts, and both stack- and SS-softplus regressions may significantly outperform logistic regression by setting the number of expert criteria  as $T\ge 2$,  exhibiting no clear sign of overfitting as $T$ further increases. 


 \begin{figure}[!t]
\begin{center}
 \includegraphics[width=0.77\columnwidth]{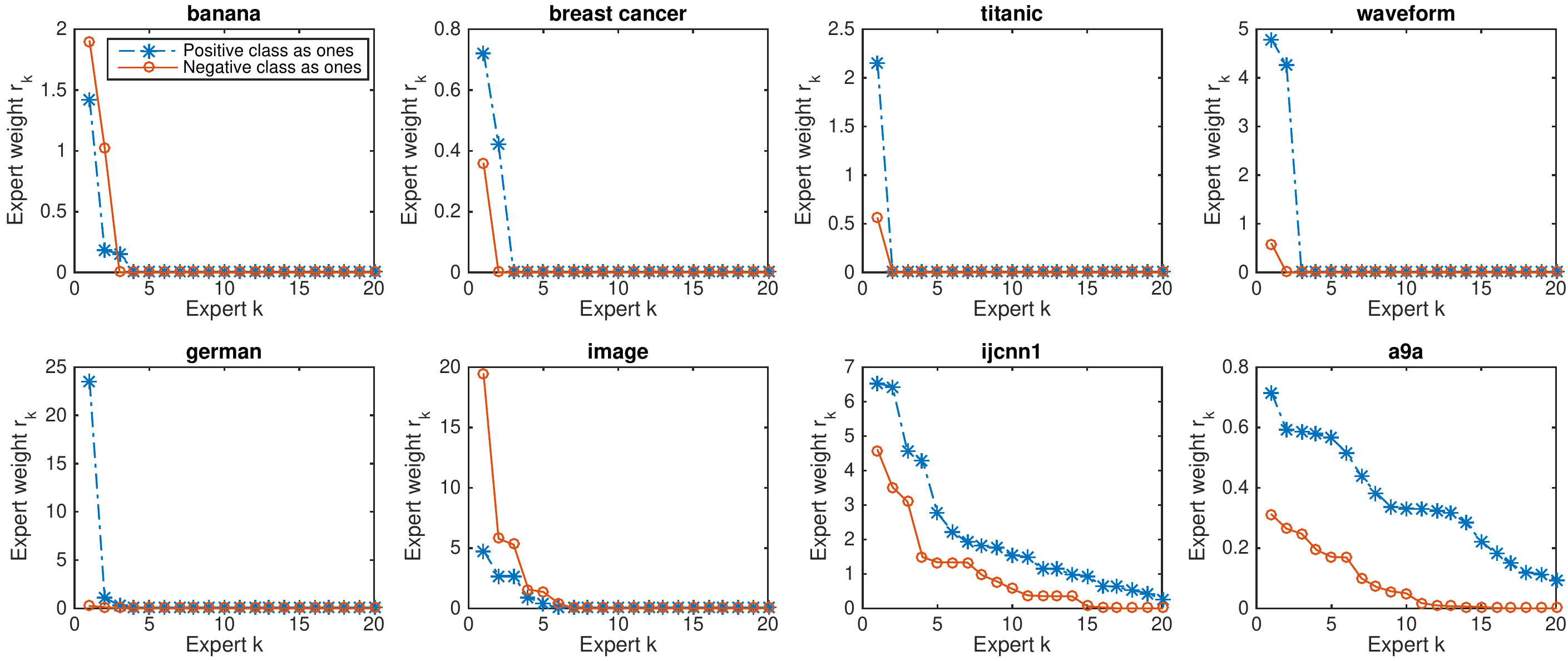}
\end{center}
\vspace{-7.9mm}
\caption{\small\label{fig:rk}
The inferred weights of the $K_{\max}=20$ experts of sum-softplus regression, ordered from left to right according to their weights, on all datasets shown in Tab. \ref{tab:data}, based on the maximum likelihood sample of a single random trial.
}

\begin{center}
 \includegraphics[width=0.77\columnwidth]{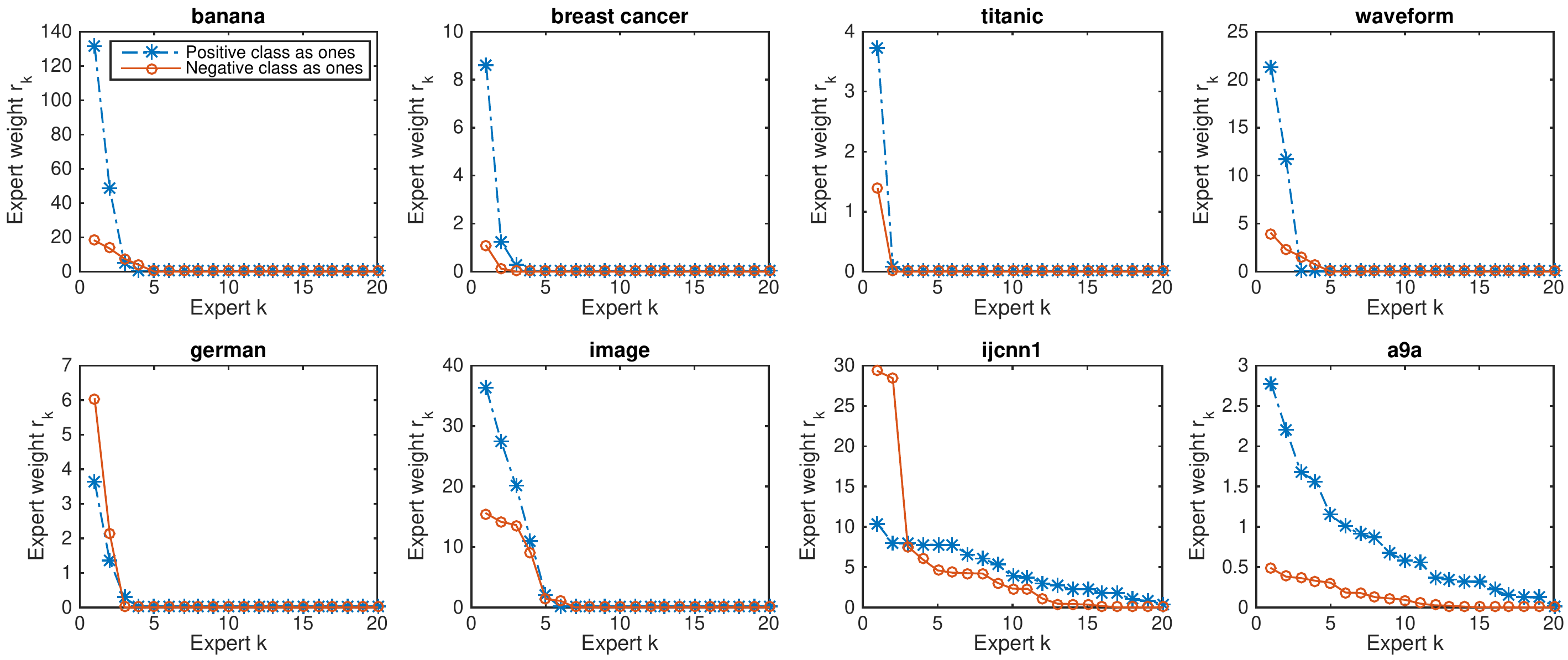}
\end{center}
\vspace{-7.9mm}
\caption{\small\label{fig:SS-softplus_rk}
Analogous to 
 Fig. \ref{fig:rk} for SS-softplus regression with $K_{\max}=20$ and $T=5$. }
 \vspace{-3mm}
\end{figure}

 For both stack- and SS-softplus regressions, the computational complexity in both training and out-of-sample prediction increases linearly in $T$, the depth of the stack. To understand how increasing $T$ affects the performance, we show in Tabs. \ref{table:4}-\ref{tab:5} of Appendix \ref{app:setting} the classification errors of stack- and SS-softplus regressions, respectively, for $T\in\{1,2,3,5,10\}$. It is clear that increasing $T$ from 1 to 2 generally leads to the most significant improvement if there is a clear advantage of increasing $T$, and once $T$ is sufficiently large, further increasing $T$ leads to small fluctuations of the performance 
 but
 does not appear to lead to clear overfitting. It is also interesting to examine the number of active experts inferred by SS-softplus regression, where each expert is equipped with $T$ hyperplanes, as $T$ increases. As shown in Tab. \ref{table:6} of Appendix \ref{app:setting}, this number has a clear increasing trend as $T$ increases. This is not surprising as each expert is able to fit more complex geometric structure as $T$ increases, and hence SS-softplus regression can employ more of them to more detailedly describe the decision boundaries. This phenomenon is also clearly visualized in comparing the inferred experts and decision boundaries for SS-softplus regression, as shown in Fig. \ref{fig:circle_3}, with those for sum-softplus regression, as shown in Fig. \ref{fig:circle_1}.
 
 In addition to comparing softplus regressions with related algorithms on the six benchmark datasets used in \citet{RVM}, we also consider ijcnn1 and a9a, two larger-scale benchmark datasets that have also been used in \citet{chang2010training,wang2011trading} and \citet{kantchelian2014large}. In Appendix \ref{app:setting}, we report results on both datasets, whose training/testing partition is predefined, based on a single random trial for logistic regression, SVM, and RVM, and five independent random trials for AMM, CPM, and all softplus regressions. As shown in Tabs. \ref{tab:Error1}-\ref{table:10} of Appendix \ref{app:setting}, we observe similar relationships between the classification errors and the number of expert criteria for both stack- and SS-softplus regressions, and both sum- and SS-softplus regressions provide a good comprise between the classification accuracies and amount of computation required for out-of-sample predictions.
 




As shown in Figs. \ref{fig:rk}-\ref{fig:SS-softplus_rk}, with the upper-bound of the number of experts set as $K_{\max}=20$, for each of the first six datasets, both sum- and SS-softplus regressions shrink the weights of most of the 20 experts to be close to zero, clearly inferring the number of experts with non-negligible weights under both labeling settings. For both ijcnn1 and a9a, at one of the two labeling setting for both sum- and SS-softplus regressions, $K_{\max}=20$ does not seem to be large enough to accommodate all experts with non-negligible weights. Thus we have also tried setting $K_{\max}=50$, which is found to more clearly show the ability of the model to shrink the weights of unnecessary experts for both ijcnn1 and a9a, 
but at the expense of clearly increased computational complexity in both training and testing. The automatic shrinkage mechanism of the gamma process based sum- and SS-softplus regressions is attractive for both computation and implementation, as it allows setting $K_{\max}$ as large as permitted by the computational budget, without the need to worry about overfitting. Having the ability to support countably infinite 
 experts in the prior and inferring a finite number of experts with non-negligible weights in the posterior is an attractive property of the proposed nonparametric Bayesian softplus regression models. 

We comment that while we choose a fixed truncation to approximate a countably infinite nonparametric Bayesian model, it is possible to 
adaptive truncate the number of experts for the proposed gamma process based models, 
using strategies such as marginalizing out the underlying stochastic processes \citep{HDP,lijoi2007controlling}, performing reversible-jump MCMC \citep{green1995reversible,Wolp:Clyd:Tu:2011}, and using slice sampling \citep{walker2007sampling,neal2003slice}, which would be interesting topics for future research.

\vspace{-4mm}
\section{Conclusions}\label{sec:conclusion}
\vspace{-2mm}
To regress a binary response variable on its covariates, we propose sum-, stack-, and sum-stack-softplus regressions that use, respectively, 
a convex-polytope-bounded confined space to enclose the negative class, a convex-polytope-like confined space to enclose the positive class, and a union of convex-polytope-like confined spaces to enclose the positive class. Sum-stack-softplus regression,  including logistic regression and all the other  softplus regressions as special examples, 
 constructs a highly flexible nonparametric Bayesian predictive distribution 
by 
mixing the convolved and stacked covariate-dependent gamma distributions with 
the Bernoulli-Poisson distribution. 
The predictive distribution is deconvolved and demixed by inferring the parameters of the underlying nonparametric Bayesian hierarchical model using  a series of data augmentation and marginalization techniques. 
In the proposed Gibbs sampler that has closed-form update equations, 
the parameters of different stacked gamma distributions can be updated in parallel within each iteration. 
Example results demonstrate that the proposed softplus regressions can achieve classification accuracies comparable to those of kernel support vector machine, but consume significant less computation for out-of-sample predictions, provide probability estimates, quantify uncertainties, and place interpretable geometric constraints on its classification decision boundaries directly in the original covariate space. 
It is of great interest to investigate how to generalize the proposed softplus regressions to model 
count, categorical, ordinal, and continuous response variables, and to model observed or latent multivariate discrete vectors. For example, to introduce covariate-dependence into a stick-breaking process mixture model \citep{ishwaran2001gibbs}, one may consider replacing the normal cumulative distribution function used in the probit stick-breaking process of \citet{PSBP_Dunson} or the logistic function used in the logistic stick-breaking process of \citet{Lu_LSBP} with the proposed softplus functions. 

\small
\begin{spacing}{1.125} 
\setlength{\bibsep}{1pt plus 0ex}
\bibliographystyle{abbrvnat}
\bibliography{References052016}

\normalsize

\end{spacing}
\newpage
\appendix
\normalsize

\begin{center}
{\LARGE{Softplus Regressions and Convex Polytopes:
\vspace{3mm}\\ Supplementary Materials}}

\end{center}

\section{Polya-Gamma distribution}

To infer the regression coefficient vector for each NB regression, we use the Polya-Gamma random variable $X\sim \mbox{PG}(a,c)$, defined in \citet{LogitPolyGamma} as the weighted sum of infinite independent, and identically distributed ($i.i.d.$) gamma random variables as 
\beq
X = \frac{1}{2\pi^2}\sum_{k=1}^\infty \frac{g_k}{(k-1/2)^2+c^2/(4\pi^2)}, ~~g_k\sim\mbox{Gamma}(a,1). \label{eq:PGdefine}
\eeq
As in \citet{polson2013bayesian}, the moment generating function of the Polya-Gamma random variable $X\sim\mbox{PG}(a,c)$ can be expressed as
\beq
\E[e^{sX}] = {\cosh^a\left(\frac{c}{2}\right)}{\cosh^{-a}\left(\sqrt{\frac{c^2/2-s}{2}}\right)}. \notag
\eeq
Let us denote $f(s) = \sqrt{\frac{c^2/2-s}{2}}$ and hence $f'(s) = -1/[4f(s)]$. Since
$$
\frac{d\E[e^{sX}] }{ds} = \frac{a}{4}\E[e^{sX}]  \frac{\tanh[f(s)]}{f(s)}, 
$$
$$
\frac{d^2\E[e^{sX}] }{ds^2} 
= \frac{1}{\E[e^{sX}]}  \left(\frac{d\E[e^{sX}] }{ds}\right)^2 +
 \frac{a}{4}\E[e^{sX}] \cosh^{-2}[f(s)]\frac{\sinh[2f(s)]-2f[s]}{[2f(s)]^3} ,\notag
$$
the mean  can be expressed as
\beq
\E[X\given a, c] = \frac{ d\E[e^{sX}] }{ds}(0) = \frac{a}{2|c|}\tanh\left(\frac{|c|}{2}\right), 
\label{eq:PGmean}
\eeq
and the variance   can be expressed as
\begin{align} \label{eq:varPG}
\mbox{var}[X\given a, c] &= \frac{ d^2\E[e^{sX}] }{ds^2}(0)  - (\E[X])^2=\frac{a\cosh^{-2}\big({|c|}/{2}\big) {[\sinh(|c|)-|c|]}}{4|c|^3}=\frac{a}{2|c|^3} \frac{{\sinh(|c|)-|c|}}{\cosh(|c|)+1} \notag\\
&= \frac{a}{2|c|^3} \frac{1-e^{-2|c|}-2|c|e^{-|c|}}{1+e^{-2|c|}+2e^{-|c|}} = \frac{a\cosh^{-2}\left(\frac{|c|}{2}\right) }{4} 
 \left(\frac{1}{6} + \sum_{n=1}^\infty \frac{|c|^{2n}}{(2n+3)!}\right),
\end{align}
which matches the variance shown in \citet{glynn2015bayesian} but with a much simpler derivation.

 \begin{figure}[!t]
\begin{center}
\includegraphics[width=0.66\columnwidth]{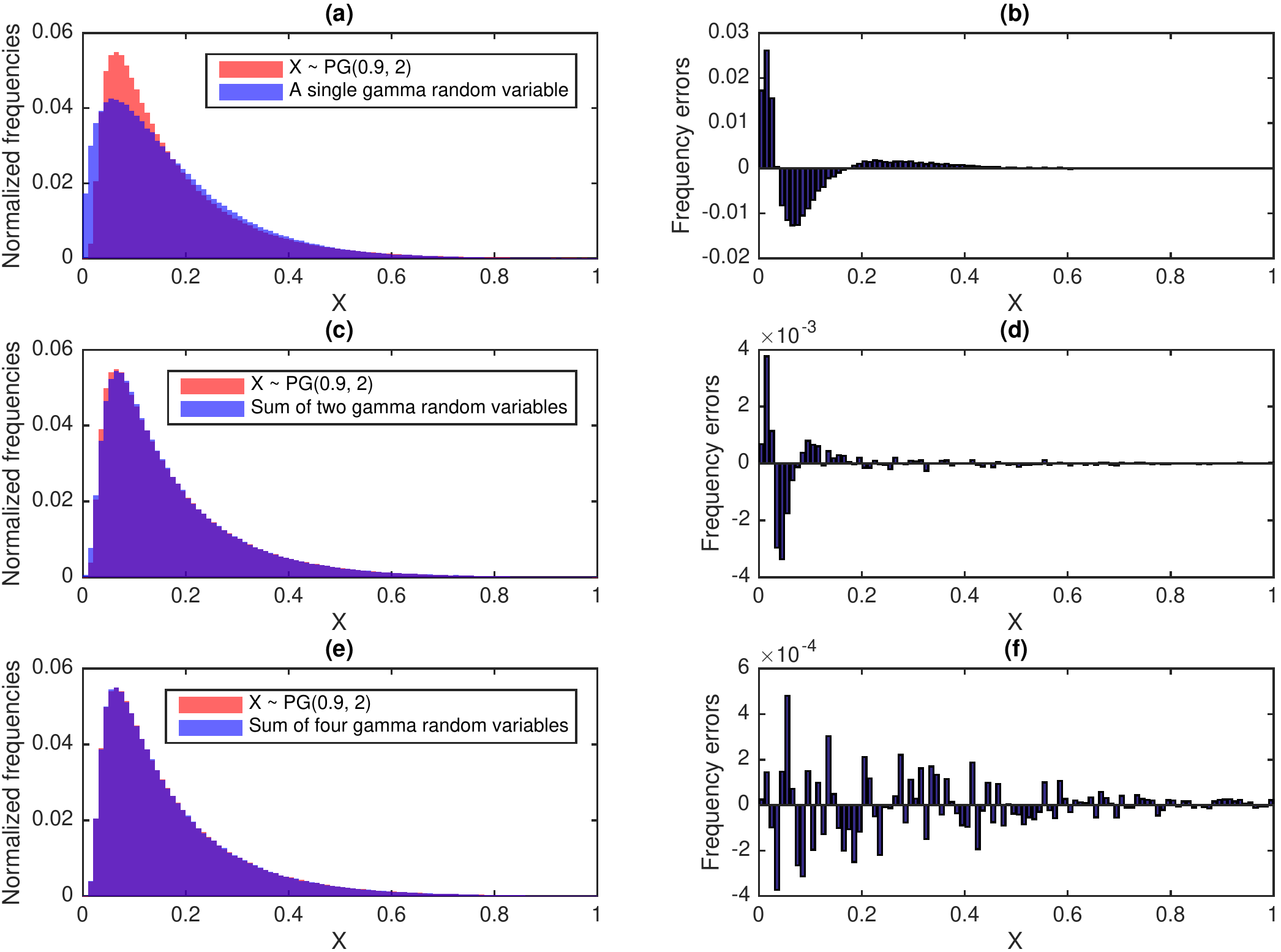}
\vspace{-.3mm}
\end{center}
\vspace{-5.9mm}
\caption{\small\label{fig:PG_Gamma}
(a) Comparison of the normalized histogram of  $10^6$ independent random samples following $X\sim\mbox{PG}(0.9,2)$ (simulated using the truncated sampler at a truncation level of 1000)   and that of $10^6$ ones 
simulated with the approximate sampler truncated at one, $i.e.$, simulated from
 $X\sim\mbox{Gamma}\left( {\mu_{\triangle}^2}/{\sigma^2_{\triangle}}, {\sigma^2_{\triangle}}/{\mu_{\triangle}}\right)$, where $\mu_{\triangle}=0.1714$ and $\sigma^2_{\triangle}=0.0192$ are chosen to match the mean and variance of $X\sim\mbox{PG}(0.9,2)$;
 (b) the differences between the normalized frequencies of these two histograms.
 (c)-(d): analogous plots to (a)-(b), 
 with the truncated sampler truncated at two.
 (e)-(f): analogous plots to (a)-(b), with 
 the truncated sampler truncated at four. 
 }
\end{figure}
As in \eqref{eq:PGdefine}, a $\mbox{PG}$ distributed random variable can be generated from an infinite sum of weighted $i.i.d.$ gamma random variables. 
In \citet{polson2013bayesian}, when $a$ is an integer, $X\sim \mbox{PG}(a,c)$ is sampled exactly using a rejection sampler, but when $a$ is a positive real number, it is sampled approximately by truncating the infinite sum in \eqref{eq:PGdefine}. However, the mean of $X$ approximately generated in this manner is guaranteed to be left biased. An improved approximate sampler is proposed in \citet{LGNB_ICML2012} to compensate the bias of the mean, but not the bias of the variance. 

We present in Proposition \ref{prop:PGtruncate} 
an approximate sampler that is unbiased in both the mean and variance, using the summation of a finite number of gamma random variables. As shown in Fig. \ref{fig:PG_Gamma}, 
the approximate sampler is quite accurate even only using two 
 gamma random variables. 
We also provide some additional propositions, whose proofs are deferred to Appendix~\ref{sec:proof}, to describe 
some important properties that will be used in inference.

\begin{prop}\label{prop:PGtruncate}
Denoting $K\in\{1,2,\ldots\}$ as a truncation level, if
\beq\label{eq:Xhat}
\hat{X} = \frac{1}{2\pi^2}\sum_{k=1}^{K-1} \frac{g_k}{(k-1/2)^2+c^2/(4\pi^2)},~~g_k\sim\emph{\mbox{Gamma}}(a,1)
\eeq
and
$ 
X_{\triangle}\sim\emph{\mbox{Gamma}}\left( {\mu_{\triangle}^2}{\sigma^{-2}_{\triangle}}, {\mu^{-1}_{\triangle}}{\sigma^2_{\triangle}}\right),
$ 
where
\beqs
\displaystyle\mu_{\triangle} = 
\frac{a}{2|c|}\tanh\left(\frac{|c|}{2}\right) -\E[\hat{X}],~~~
\displaystyle\sigma^2_{\triangle} = 
 \frac{a}{2|c|^3} \frac{{\sinh(|c|)-|c|}}{\cosh(|c|)+1} - \emph{\mbox{var}}[\hat{X}], \notag
\eeqs
then 
$\hat{X}+X_{\triangle}$ has the same mean and variance as those of $X\sim\emph{\mbox{PG}}(a,c)$, and 
the difference between their cumulant generating functions 
can be expressed as
\beq
\ln \E[e^{sX}] - \ln \E[e^{s(\hat{X}+X_{\triangle})}] = \sum_{n=3}^\infty \frac{s^n }{n} \left[ \left( a\sum_{k=K}^\infty d_k^{-n}\right)- \frac{\mu_{\triangle}^2}{\sigma^2_{\triangle}}
 \left(\frac{\sigma^2_{\triangle}}{\mu_{\triangle}}\right)^{n} \right],\notag
\eeq
where
$d_k=2\pi^2(k-1/2)^2+c^2/2.$
\end{prop}

\begin{prop}\label{prop:PG}
If $X\sim\emph{\mbox{PG}}(a,c)$, then $\lim_{|c|\rightarrow \infty}X= 0$ and $\lim_{|c|\rightarrow \infty}|c|X= a/2$.
\end{prop}
\begin{prop}\label{prop:PG1}
If $X\sim\emph{\mbox{PG}}(a,0)$, then $\E[X\given a, 0]= a/4$ and $\emph{\mbox{var}}[X\given a, 0]= a/24$.
\end{prop}

\begin{prop}\label{prop:PG1_1}
If $X\sim\emph{\mbox{PG}}(a,c)$, then $\displaystyle\emph{\mbox{var}}[X\given a, c] \ge \frac{a}{24}{\cosh^{-2}\left(\frac{|c|}{2}\right) }
$, where the equality holds if and only if $c=0$.
\end{prop}

\begin{prop}\label{prop:PG2}
$X\sim\emph{\mbox{PG}}(a,c)$ has a variance-to-mean ratio as 
\begin{align} 
\frac{\emph{\mbox{var}}[X\given a, c] }{\E[X\given a, c]} 
 = \frac{1}{|c|^{2}} - \frac{1}{|c|\sinh(|c|)} \notag
\end{align}
and is always under-dispersed, since 
$ \emph{\mbox{var}}[X\given a, c]\le\E[X\given a, c]/6$ almost surely, where the equality holds if and only if $c=0$.
\end{prop}

\section{Proofs}\label{sec:proof}
\begin{proof}[Proof for Definition \ref{thm:1}]
For the hierarchical model in \eqref{eq:Softplus_model}, we have $P(y_i=0\given \theta_i) = e^{-\theta_i}$. Further using the moment generating function of the gamma distribution,  we have
\beq
  P(y_i =0 \given \xv_i, \betav)  = \E_{\theta_i}[e^{-\theta_i}] = (1 +  e^{\xv_i'\betav})^{-1}. \notag
\eeq
As $\lambda(\xv_i) =-\ln[P(y_i =0 \given \xv_i, \betav)] $ by definition, we have $\lambda(\xv_i) = \ln(1+e^{\xv_i'\betav})$.
\end{proof}


\begin{proof}[Proof for Definition \ref{thm:CLR2CNB}]
For the hierarchical model in \eqref{eq:CNB}, we have $P(y_i=0\given \theta_i) = e^{-\theta_i} = \prod_{k=1}^\infty e^{-\theta_{ik}}$. Using the moment generating function of the gamma distribution,  we have
\beq
  P(y_i =0 \given \xv_i, \{\betav_k\}_k)  = \E_{\theta_{i}}\left[\prod_{k=1}^\infty e^{-\theta_{ik}}\right]  = \prod_{k=1}^\infty \E_{\theta_{ik}}[e^{-\theta_{ik}}] = \prod_{k=1}^\infty (1 +  e^{\xv_i'\betav_k})^{-r_k}.\notag
\eeq
As $\lambda(\xv_i) = -\ln[P(y_i =0 \given \xv_i, \{\betav_k\}_k) ]$ by definition, we obtain 
\eqref{eq:sum_softplus_reg}.
\end{proof}

\begin{proof}[Proof for Definition \ref{thm:stack-softplus}]
For the hierarchical model in \eqref{eq:BerPo_recursive_softplus_reg_model}, we have $P(y_i=0\given \theta^{(1)}_i) = e^{-\theta_i^{(1)}} $. Using the moment generating function of the gamma distribution,  we have
\beq
  P(y_i =0 \given \xv_i, \betav^{(2)},\theta_i^{(2)})  = \E_{\theta_{i}^{(1)}}\left[e^{-\theta_{i}^{(1)}}\right]  =  (1 +  e^{\xv_i'\betav^{(2)}})^{-\theta_i^{(2)}} = e^{-\theta_i^{(2)} \ln(1+e^{\xv_i'\betav^{(2)}}) }.\notag
\eeq
Marginalizing out $\theta_i^{(2)}$ leads to
\beq
  P(y_i =0 \given \xv_i, \betav^{(2:3)},\theta_i^{(3)})  = \E_{\theta_{i}^{(2)}}\left[e^{-\theta_i^{(2)} \ln(1+e^{\xv_i'\betav^{(2)}})}\right]  =  e^{-\theta_i^{(3)} \ln[1+e^{\xv_i'\betav^{(3)}}\ln(1+e^{\xv_i'\betav^{(2)}})] }.\notag
\eeq
Further marginalizing out $\theta_i^{(3)},\ldots,\theta_i^{(T)}$ and with 
 $\lambda(\xv_i) =-\ln[P(y_i =0 \given \xv_i, r,\betav^{(2:T+1)})] $ by definition, we obtain 
\eqref{eq:recurssive_softplus_reg}.
\end{proof}
%
%
%

\begin{proof}[Proof for Definition \ref{thm:SS-softplus}]
For the hierarchical model in \eqref{eq:DICLR_model}, we have $P(y_i=0\given \{\theta^{(1)}_{ik}\}_k) = e^{-\sum_{k=1}^\infty \theta_{ik}^{(1)}} $. Using the moment generating function of the gamma distribution,  we have
\beq
  P(y_i =0 \given \xv_i, \{\betav_k^{(2)},\theta_{ik}^{(2)}\}_k)  = \prod_{k=1}^\infty e^{-\theta_{ik}^{(2)} \ln\big(1+e^{\xv_i'\betav_k^{(2)}}\big) }.\notag
\eeq
Further marginalizing out $\{\theta_{ik}^{(2)}\}_k,\ldots,\{\theta_{ik}^{({T})}\}_k$  
 and  by definition
with 
 $\lambda(\xv_i) =-\ln[P(y_i =0 \given \xv_i, \{r_k,\betav_k^{(2:{T}+1)}\}_k) ]$, we obtain 
\eqref{eq:SRS_regression}.
\end{proof}

\begin{proof}[Proof of Proposition \ref{lem:finite}]
By construction, the infinite product would be equal or small than one. We need to further make sure that the infinite product would not degenerate to zero. 
Using the  L\'evy-Khintchine theorem \citep{kallenberg2006foundations}, we have
\begin{align}
-\ln\left\{\E_G\left[e^{{-\sum_{k=1}^\infty r_k \ln[1+\exp(\xv_i'\betav_k)]}}\right]\right\}&=\int_{\mathbb{R}_+\times\Omega} \left[ 1-\left(\frac{1}{1+e^{\xv_i'\betav}}\right)^{r} \right] \nu(drd\betav).\notag
\end{align}
where $\nu(drd\betav)=r^{-1}e^{-cr}drG_0(d\betav)$ is  the L\'evy measure of the gamma process. %
Since if $c\ge 0$, then $1-e^{-c x}\le c x$ for all $x\ge 0$, the right-hand-side term of the above equation would be bounded below
\begin{align}
&\int_{\mathbb{R}_+\times\Omega} r\ln[1+e^{\xv_i'\betav}] \nu(drd\betav)= \frac{\gamma_0}{c} \int_{\Omega} \ln[1+e^{\xv_i'\betav}] g_0(d\betav).\label{eq:normalint}
\end{align}
Since $e^{e^x} = 1+e^x + \sum_{n=2}^\infty \frac{e^{nx}}{n!} \ge 1+e^x,$
we have 
\beq
\ln(1+e^x)\le e^{x},
\eeq
where the equality is true if and only if $x=-\infty$.
Assuming $\betav\sim\mathcal{N}(0,\Sigmamat)$, we have
\beq
\int_{\Omega} \ln[1+e^{\xv_i'\betav}] g_0(d\betav)
\le  \int_{\Omega} e^{\xv'_i\betav} \mathcal{N}(\betav;0,\Sigmamat)d\betav = e^{\frac{1}{2}\xv_i'\Sigmamat\xv_i}\notag
\eeq
%
%
Thus  the integral in the right-hand-side of (\ref{eq:normalint}) is finite and hence the infinite product $\prod_{k=1}^\infty\left[{1+e^{\xv_i'\betav_k}}\right]^{-r_k}$ has a finite expectation that is greater than zero. 
\end{proof}

  \begin{proof}[Proof of Theorem \ref{thm:sum_polytope}]
 Since $\sum_{k'\neq k}r_{k'} \ln(1+e^{\xv_i' \betav_{k'}}) \ge 0$ a.s., if \eqref{eq:sum_ineuqality} is true, then 
$
 r_k \ln(1+e^{\xv_i' \betav_{k}}) \le -\ln(1-p_0)
$
 a.s. for all $k\in\{1,2,\ldots\}$. Thus if \eqref{eq:sum_ineuqality} is true, then \eqref{eq:convex_polytope} is true a.s., which means the set of solutions to  \eqref{eq:sum_ineuqality} is included in the set of solutions to \eqref{eq:convex_polytope}.
%
  \end{proof}

\begin{proof}[Proof of Proposition \ref{prop:sum_polytope}]
    Assuming $\xv_i$ violates at least the $k$th inequality, which means 
   $
    \xv_i' \betav_k > \ln\big[(1-p_0)^{-\frac{1}{r_k}}-1\big],
   $
   then we have $$\lambda(\xv_i) =r_k\ln(1+e^{\xv_i' \betav_{k}}) +  \sum_{k'\neq k}r_{k'} \ln(1+e^{\xv_i' \betav_{k'}}) \ge r_k\ln(1+e^{\xv_i' \betav_{k}}) > -\ln(1-p_0) $$ and hence $P(y_i=1\given \xv_i)=1-e^{-\lambda(\xv_i)}> p_0$ and $P(y_i=0\given \xv_i)\le 1- p_0$.
    \end{proof}

\begin{proof}[Proof of Proposition \ref{lem:finite1}]
By construction, the infinite product would be equal or small than one. We need to further make sure that the infinite product would not degenerate to zero. 
Using the  L\'evy-Khintchine theorem \citep{kallenberg2006foundations} and $1-e^{-c x}\le c x$ for all $x\ge 0$ if $c\ge 0$, we have
\footnotesize{
\begin{align}
&-\ln\left\{\E_G\exp \left[ { - \sum_{k=1}^\infty r_k
\ln\left(1+e^{\xv_i'\betav^{({T}+1)}_k}\ln\Bigg\{1+e^{\xv_i'\betav^{({T})}_k}\ln\bigg[1+\ldots
\ln\Big(1+e^{\xv_i'\betav^{(2)}_k}\Big)\bigg]\Bigg\}\right) }\right]\right\}\notag\\
&=\int 
\left[ 1-\left(1+e^{\xv_i'\betav^{({T}+1)}}\ln\Bigg\{1+e^{\xv_i'\betav^{({T})}}\ln\bigg[1+\ldots
e^{\xv_i'\betav^{(3)}}\ln\Big(1+e^{\xv_i'\betav^{(2)}}\Big)\bigg]\Bigg\}\right)^{-r} \right] \nu(drd\betav^{({T}+1:2)})\notag\\
&\le \int 
r\ln \left(1+e^{\xv_i'\betav^{({T}+1)}}\ln\Bigg\{1+e^{\xv_i'\betav^{({T})}}\ln\bigg[1+\ldots
e^{\xv_i'\betav^{(3)}}\ln\Big(1+e^{\xv_i'\betav^{(2)}}\Big)\bigg]\Bigg\}\right) \nu(drd\betav^{({T}+1:2)}).
\label{eq:Levy_Khin}
\end{align}}\normalsize
Since
$
\ln(1+e^x)\le e^{x},
$
we have
\begin{align}
~&\ln \left(1+e^{\xv_i'\betav^{({T}+1)}}\ln\Bigg\{1+e^{\xv_i'\betav^{({T})}}\ln\bigg[1+\ldots
e^{\xv_i'\betav^{(3)}}\ln\Big(1+e^{\xv_i'\betav^{(2)}}\Big)\bigg]\Bigg\}\right)\notag\\
\le~& \ln \left(1+e^{\xv_i'\betav^{({T}+1)}}\ln\Bigg\{1+e^{\xv_i'\betav^{({T})}}\ln\bigg[1+\ldots
e^{\xv_i'\betav^{(4)}}\ln\Big(1+e^{\xv_i'(\betav^{(3)}+\betav^{(2)})}\Big)\bigg]\Bigg\}\right)\notag\\
\le~& \ln \left(1+e^{\xv_i'\betav^{(\star)}}\right)
\le e^{\xv_i'\betav^{(\star)}}\notag
\end{align}
where $\betav^{(\star)}:=\betav^{({T}+1)}+\betav^{({T})}+\ldots+\betav^{(2)}$. Assuming $\betav^{(t)} \sim\mathcal{N}(0,\Sigmamat_t)$, the right hand side of \eqref{eq:Levy_Khin} would be bound below
$ 
\int 
re^{\xv_i'\betav^{(\star)}} \nu(drd\betav^{(\star)}) = \gamma_0 c^{-1}e^{\frac{1}{2} \xv_i' (\sum_{t=2}^{T+1} \Sigmamat_t)\xv_i }. \notag
$ 
Therefore, the integral in the right-hand-side of (\ref{eq:Levy_Khin}) is finite and hence the infinite product 
in Proposition \ref{lem:finite1}
has a finite expectation that is greater than zero under the gamma process.
\end{proof}

  \begin{proof}[Proof of Theorem \ref{thm:union_polytope}]
 Since $$\sum_{k'\neq k} r_{k'}
\ln\left(1+e^{\xv_i'\betav^{({T}+1)}_{k'}}\ln\Bigg\{1+e^{\xv_i'\betav^{({T})}_{k'}}\ln\bigg[1+\ldots
\ln\Big(1+e^{\xv_i'\betav^{(2)}_{k'}}\Big)\bigg]\Bigg\}\right) \ge 0$$ a.s., if 
\eqref{eq:Union_convex_polytope} is true for at least one $k\in\{1,2,\ldots\}$, then 
\eqref{eq:SS-softplus_ineuqality}
is true a.s., which means the set of solutions to $\eqref{eq:SS-softplus_ineuqality}$ encompass $\mathcal{D}_{\star}$. 
%
  \end{proof}

  \begin{proof}[Proof of Proposition \ref{prop:union}]
    Assume $\xv_i$ satisfies at least the $k$th inequality, which means 
   \eqref{eq:Union_convex_polytope} is true,
   then
   $$\small
\lambda(\xv_i)\ge
 r_k
\ln\left(1+e^{\xv_i'\betav^{({T}+1)}_k}\ln\Bigg\{1+e^{\xv_i'\betav^{({T})}_k}\ln\bigg[1+\ldots
\ln\Big(1+e^{\xv_i'\betav^{(2)}_k}\Big)\bigg]\Bigg\}\right)> -\ln(1-p_0) $$ and hence $P(y_i=1\given \xv_i)=1-e^{-\lambda(\xv_i)}> p$ and $P(y_i=0\given \xv_i)\le 1- p_0$.
    \end{proof}

\begin{proof}[Proof of Theorem \ref{cor:PGBN}]
By construction (\ref{eq:deepPFA_aug1}) is true for $t=1$. 
Suppose (\ref{eq:deepPFA_aug1}) is also true for $t\ge 2$,
then we can augment each $m_{ik}^{(t)}$ under its compound Poisson representation as 
\beq
m_{ik}^{(t)} \given m_{ik}^{(t+1)}\sim \mbox{SumLog}(m_{ik}^{(t+1)},~p_{ik}^{(t+1)}),~~
m_{ik}^{(t+1)}\sim\mbox{Pois}\left(\theta_{ik}^{(t+1)} q_{ik}^{(t+1)}\right), \label{eq:CompP}
\eeq
where the joint distribution of $m_{ik}^{(t)}$ and $m_{ik}^{(t+1)}$, according to Theorem 1 of \citet{NBP2012}, is the same as that in 
\beq
m_{ik}^{(t+1)} \given m_{ik}^{(t)} \sim\mbox{CRT}(m_{ik}^{(t)}, ~\theta_{ik}^{(t+1)}),~~m_{ik}^{(t)}\sim{\mbox{NB}}(\theta_{ik}^{(t+1)},~ p_{ik}^{(t+1)}) ,\notag
\eeq
where CRT refers to the Chinese restaurant table distribution described in \citet{NBP2012}.
Marginalizing $\theta_{ik}^{(t+1)}$ from the Poisson distribution in \eqref{eq:CompP} leads to
$m_{ik}^{(t+1)}\sim{\mbox{NB}}(\theta_{ik}^{(t+2)},$ $p_{ik}^{(t+2)}) $.
Thus if (\ref{eq:deepPFA_aug1}) is true for layer $t$, then 
it is also true for layer $t+1$. 
\end{proof}

\begin{proof}[Proof of Proposition \ref{prop:PGtruncate}]
Since $\hat{X}$ and $X_{\triangle}$ are independent to each other, 
with \eqref{eq:PGmean} and \eqref{eq:varPG}, we have
and $
\E[X]  = \E[\hat{X}] + \mu_{\triangle} = \E[\hat{X} + X_{\triangle}]$
and 
$
\mbox{var}[X] =\mbox{var}[\hat{X}] + \sigma^2_{\triangle} = \mbox{var}[\hat{X} + X_{\triangle} ]$.
Using Taylor series expansion, we have
\begin{align}
\ln \E[e^{sX}] &= -\sum_{k=1}^\infty a\ln(1-sd_k^{-1}) = a \sum_{k=1}^\infty \sum_{n=1}^\infty \frac{s^n d_k^{-n}}{n} 
 = s \E[X] + s^2 \frac{\mbox{var}[X] }{2}+a \sum_{n=3}^\infty \sum_{k=1}^\infty  \frac{s^n d_k^{-n}}{n},\notag 
\end{align}
\begin{align}
\ln \E[e^{s\hat{X}}] &= -\sum_{k=1}^{K-1} a\ln(1-sd_k^{-1})
 = s \E[\hat{X}] + s^2 \frac{\mbox{var}[\hat{X}]}{2} +a \sum_{n=3}^\infty \sum_{k=1}^{K-1}  \frac{s^n d_k^{-n}}{n},\notag\\
\ln \E[e^{sX_{\triangle}}] &= - \frac{\mu_{\triangle}^2}{\sigma^2_{\triangle}} \ln\left(1-s \frac{\sigma^2_{\triangle}}{\mu_{\triangle}} \right)  =  \frac{\mu_{\triangle}^2}{\sigma^2_{\triangle}}\sum_{n=1}^\infty \frac{s^n (\frac{\sigma^2_{\triangle}}{\mu_{\triangle}})^{n}}{n} = s\mu_{\triangle} + s^2 \frac{\sigma^2_{\triangle}}{2}  +\frac{\mu_{\triangle}^2}{\sigma^2_{\triangle}}\sum_{n=3}^\infty \frac{s^n (\frac{\sigma^2_{\triangle}}{\mu_{\triangle}})^{n}}{n}.\notag
\end{align}
The proof is completed with $\ln\E[e^{sX}] - \ln \E[e^{s(\hat{X}+X_{\triangle})}]  =\ln \E[e^{sX}] - \ln \E[e^{s(\hat{X})}]  - \ln \E[e^{X_{\triangle})}] .$
\end{proof}

\begin{proof}[Proof of Proposition \ref{prop:PG}]
For the Polya-Gamma random variable $X\sim\mbox{PG}(a,c)$, since 
$
\E[X] = \frac{a}{2|c|}\tanh\big(\frac{|c|}{2}\big)
$
and $\lim_{|c|\rightarrow  \infty}\tanh\big(\frac{|c|}{2}\big)=1$, 
we have $\lim_{|c|\rightarrow \infty}\E[X]=0$. 
With the expression of $\mbox{var}[X]$ shown in \eqref{eq:varPG},  
we have $\lim_{|c|\rightarrow \infty}\mbox{var}[X]=0$.
Therefore, we have $X\rightarrow 0$ as  $|c|\rightarrow \infty$.

Since $\lim_{|c|\rightarrow  \infty} \E[|c|X] =a/2$ and 
$$\mbox{var}[|c|X] = (|c|)^2 \mbox{var}[X] =
\frac{a}{2|c|} \frac{{\sinh(|c|)-|c|}}{\cosh(|c|)+1} = \frac{a}{2}  \frac{{\frac{\tanh(|c|)}{|c|}-\cosh^{-1}(|c|)}}{1+ \cosh^{-1}(|c|)}.
$$
 we have $\lim_{|c|\rightarrow \infty}\mbox{var}[|c|X] = 0$ and hence $\lim_{|c|\rightarrow \infty} |c|X = a/2$.
\end{proof}

\begin{proof}[Proof of Proposition \ref{prop:PG1}]
Using \eqref{eq:PGmean}, we have
\beq
\E[X\given a, 0] = \lim_{c\rightarrow 0} \E[X\given a, c] = \lim_{c\rightarrow 0 }\frac{a}{2|c|}\tanh\left(\frac{|c|}{2}\right)  = 
\lim_{c\rightarrow 0 } \frac{a}{2} \frac{e^{|c|}-1}{|c|}\frac{1}{e^{|c|}+1} = \frac{a}{4}.\notag
\eeq
Using \eqref{eq:varPG}, we have \newline
$\displaystyle
~~~~~~~\mbox{var}[X\given a, 0] = \lim_{c\rightarrow 0} \mbox{var}[X\given a, c]   
= \lim_{c\rightarrow 0} \frac{a\cosh^{-2}\left(\frac{|c|}{2}\right) }{4}  \left(\frac{1}{6} + \sum_{n=1}^\infty \frac{|c|^{2n}}{(2n+3)!}\right) = \frac{a}{24}.\notag
$
\end{proof}

\begin{proof}[Proof of Proposition \ref{prop:PG1_1}]
Using \eqref{eq:varPG}, we have $
\mbox{var}[X\given a, c] = \frac{a\cosh^{-2}\left(\frac{|c|}{2}\right) }{4} 
 \left(\frac{1}{6} + \sum_{n=1}^\infty \frac{|c|^{2n}}{(2n+3)!}\right)  \ge \frac{a}{24} \cosh^{-2}\left(\frac{|c|}{2}\right) 
$, with the equality holds if and only if $c=0$.
\end{proof}

\begin{proof}[Proof of Proposition \ref{prop:PG2}]
With \eqref{eq:PGmean} and \eqref{eq:varPG}, the variance-to-mean ratio is 
\beq\small
\frac{\mbox{var}[X\given a, c] }{\E[X\given a, c]}  =  \frac{1}{|c|^{2}} - \frac{1}{|c|\sinh(|c|)} =  \frac{\sinh(|c|)-|c|}{|c|^2\sinh(|c|)}
=\left[{\frac{1}{6}\displaystyle\sum_{n=0}^{\infty} \frac{|c|^{2n+3} \, 3!}{(2n+3)!}}\right]\left/\left[{\displaystyle \sum_{n=0}^{\infty} \frac{|c|^{2n+3}}{(2n+1)!}}\right]\right..\notag
\eeq
When $c=0$, with Proposition \ref{prop:PG1}, we have $ {\mbox{var}[X\given a, c] } \big / {\E[X\given a, c]} =1/6$.
When $c\neq 0$, since $3!/(2n+3)! <1/(2n+1)!$ a.s. for all $n\in\{1,\ldots\}$, we have $$\sum_{n=0}^{\infty} \frac{|c|^{2n+3} \, 3!}{(2n+3)!} <  \sum_{n=0}^{\infty} \frac{|c|^{2n+3}}{(2n+1)!}$$ a.s.
and hence $ {\mbox{var}[X\given a, c] } \big / {\E[X\given a, c]} <1/6$ a.s.
\end{proof}

\section{Gibbs sampling for sum-stack-softplus regression}\label{app:sampling}

For 
SS-softplus regression, 
Gibbs sampling via data augmentation and marginalization proceeds as follows.\\
\textbf{\emph{Sample $m_{i}$.}} Denote $\theta_{i\cdotv} = \sum_{k=1}^K \theta^{(1)}_{ik}$\,. 
Since $m_{i}=0$ a.s. given $y_{i}=0$ and $m_{i}\ge 1$ given $y_{i}=1$, and in the prior we have $m_{i}\sim\mbox{Pois}(\theta_{i\cdotv})$, following the inference for the Bernoulli-Poisson link in \citet{EPM_AISTATS2015}, we can sample $m_{i}$ as 
\beq
(m_{i}\,|\,-)\sim y_{i}\mbox{Pois}_+\left(\theta_{i\cdotv}\right),\label{eq:m_i}
\eeq 
where $m \sim \mbox{Pois}_+(\theta)$ denotes a draw from the truncated Poisson distribution, with PMF 
$f_M(m\,|\,y_i=1,\theta) = (1-e^{-\theta})^{-1}{\theta^m e^{-\theta}}/{m!}$, where $m\in\{1,2,\ldots\}$. To draw truncated Poisson random variables, we use an efficient rejection sampler  described in \cite{EPM_AISTATS2015}, whose smallest acceptance rate, which happens when the Poisson rate is one, is 63.2\%. \\
\textbf{\emph{Sample $m^{(1)}_{ik}$.}} Since letting $m_{i}=\sum_{k=1}^K m^{(1)}_{ik},~m^{(1)}_{ik}\sim\mbox{Pois}(\theta^{(1)}_{ik})$ is equivalent in distribution to letting $(m^{(1)}_{i1},\ldots,m^{(1)}_{iK})\,|\,m_{i}\sim\mbox{Mult}\left(m_{i},{\theta^{(1)}_{i1}}/{\theta_{i\cdotv}},\ldots,{\theta^{(1)}_{iK}}/{\theta_{i\cdotv}}\right),~m_{i}\sim \mbox{Pois}\left( \theta_{i\cdotv}\right)$, similar to \citet{Dunson05bayesianlatent} and \citet{BNBP_PFA_AISTATS2012}, we sample $m^{(1)}_{ik}$ as
\beq
(m^{(1)}_{i1},\ldots,m^{(1)}_{iK}\,|\,-)\sim\mbox{Mult}\left(m_{i},{\theta^{(1)}_{i1}}/{\theta_{i\cdotv}},\ldots,{\theta^{(1)}_{iK}}/{\theta_{i\cdotv}}\right).\label{eq:m_i_1}
\eeq
\textbf{\emph{Sample $m^{(t)}_{ik}$ for $t\ge 2$}.}  
As in Theorem \ref{cor:PGBN}'s proof, we sample $m^{(t)}_{ik}$ for $t=2,\ldots,{T}+1$ as
\beq
(m_{ik}^{(t)}\given m_{ik}^{(t-1)},\theta_{ik}^{(t)})\sim{\mbox{CRT}}\big(m_{ik}^{(t-1)},~\theta_{ik}^{(t)}\big).\label{eq:m_i_t}
\eeq
\textbf{\emph{Sample $\betav^{(t)}_{k}$.}}
Using data augmentation for NB regression, as in \citet{LGNB_ICML2012} and \citet{polson2013bayesian}, 
we denote $\omega^{(t)}_{ik}$ as a random variable drawn from the Polya-Gamma ($\mbox{PG}$) distribution \citep{LogitPolyGamma} as
$
\omega^{(t)}_{ik}\sim\mbox{PG}\left(m^{(t-1)}_{ik}+\theta^{(t)}_{ik},~0 \right), 
\notag
$
under which we have 
$\E_{\omega^{(t)}_{ik}} \left[\exp(-\omega^{(t)}_{ik}(\psi^{(t)}_{ik})^2/2)\right]  = {\cosh^{-(m^{(t-1)}_{ik}+\theta^{(t)}_{ik})}(\psi_{ik}^{(t)}/2)}$. 
Thus the likelihood of $\psi_{ik}^{(t)}:= \xv_i\betav_{k}^{(t)} + \ln {q}_{ik}^{(t-1)}= \ln\big(e^{q_{ik}^{(t)}}-1\big) $ in (\ref{eq:deepPFA_aug1}) can be expressed as
\begin{align}
\mathcal{L}(\psi^{(t)}_{ik})&\propto \frac{{(e^{\psi^{(t)}_{ik}})}^{m^{(t-1)}_{ik}}}{{(1+e^{\psi^{(t)}_{ik}})}^{m^{(t-1)}_{ik}+\theta^{(t)}_{ik}}}  = \frac{2^{-(m^{(t-1)}_{ik}+\theta^{(t)}_{ik})}\exp({\frac{m^{(t-1)}_{ik}-\theta^{(t)}_{ik}}{2}\psi^{(t)}_{ik}})}{\cosh^{m^{(t-1)}_{ik}+\theta^{(t)}_{ik}}(\psi^{(t)}_{ik}/2)}\nonumber\\ &\propto \exp\left({\frac{m^{(t-1)}_{ik}-\theta^{(t)}_{ik}}{2}\psi_i}\right)\E_{\omega^{(t)}_{ik}}\left[\exp[-\omega^{(t)}_{ik}(\psi^{(t)}_{ik})^2/2]\right]. \notag
\end{align} 
%
Combining the likelihood
$
\mathcal{L}(\psi^{(t)}_{ik}, \omega^{(t)}_{ik})\propto  \exp\left({\frac{m^{(t-1)}_{ik}-\theta^{(t)}_{ik}}{2}\psi_i}\right)\exp[-\omega^{(t)}_{ik}(\psi^{(t)}_{ik})^2/2]
$ and the prior, 
we sample auxiliary Polya-Gamma random variables $\omega^{(t)}_{ik}$ as
\beq
(\omega^{(t)}_{ik}\,|\,-)\sim\mbox{PG}\left(m^{(t-1)}_{ik}+\theta^{(t)}_{ik},~\xv_i'\betav_{k}^{(t)} + \ln{q}_{ik}^{(t-1)} \right),\label{eq:omega_i}
\eeq
conditioning on which we sample $\betav^{(t)}_{k}$ as
\beqs
&(\betav^{(t)}_{k}\,|\,-)\sim\mathcal{N}(\muv^{(t)}_{k}, \Sigmamat^{(t)}_{k}),~~~\displaystyle\Sigmamat^{(t)}_{k} = \left( \mbox{diag}(\alpha_{0tk},\ldots,\alpha_{Vtk}) + \sum \nolimits_{i} \omega^{(t)}_{ik}\xv_i\xv_i'  \right)^{-1}, \notag\\
&\displaystyle \muv^{(t)}_{k} = \Sigmamat^{(t)}_{k}\left[ \sum \nolimits_{i} \left(  -\omega_{ik}^{(t)}\ln {q}_{ik}^{(t-1)}  +   \frac{m^{(t-1)}_{ik}-\theta^{(t)}_{ik}}{2}\right)\xv_i\right]. \label{eq:beta}
\eeqs
Once we update $\betav^{(t)}_k$, we calculate $q^{(t)}_{ik}$ using \eqref{eq:lambda}. 
 To draw Polya-Gamma random variables, we 
 use the approximate sampler described in
 Proposition \ref{prop:PGtruncate}, which is unbiased in both its mean and its variance. The approximate sampler is found to be highly accurate even for a truncation level as small as one, for various combinations of the two Polya-Gamma parameters. 
 Unless stated otherwise, we set the truncation level of drawing a Polya-Gamma random variable as six, which means the summation of six independent gamma random variables is used  to approximate a Polya-Gamma random variable. \\
 \textbf{\emph{Sample $\theta^{(t)}_{ik}$.}} Using the gamma-Poisson conjugacy, we sample $\tau^{(t)}_{ik} := \theta^{(t)}_{ik}q_{ik}^{(t)}$ 
 as
 \beq
(\tau^{(t)}_{ik}\,|\,-)\sim\mbox{Gamma}\left(\theta^{(t+1)}_{ik}+m^{(t)}_{ik},\, 1-e^{-q_{ik}^{(t+1)}} 
\right).\label{eq:theta}
\eeq
%
 \textbf{\emph{Sample $\alpha_{vtk}$.}} We sample $\alpha_{vtk}$ as
 \beq
 (\alpha_{vtk} \,|\,-)\sim\mbox{Gamma}\left(e_0+  \frac{1}{2},\frac{1}{f_0+\frac{1}{2}(\beta^{(t)}_{vk})^2}\right).\label{eq:alpha}
 \eeq
 \textbf{\emph{Sample $c_0$.}} 
 We sample $c_0$ as
\beq
(c_0\,|\,-)\sim\mbox{Gamma}\left(e_0+\gamma_0,\frac{1}{f_0+ \sum_k r_k} 
\right).\label{eq:c0}
\eeq
 \textbf{\emph{Sample  $\gamma_0$ and $r_{k}$.}} 
 Let us denote 
$$
\tilde{p}_{k} :=  {\sum \nolimits_{i}q^{({T}+1)}_{ik}}\Big/{\big(c_0+ \sum \nolimits_{i}q^{({T}+1)}_{ik}\big)}.
$$
Given $l_{\cdotv k} = \sum_{i }m^{({T}+1)}_{ik}$, 
we first sample 
\beq
(\tilde{l}_{k}\,|\,-) \sim\mbox{CRT}(l_{\cdotv k} , \gamma_0/K).
\eeq
With these latent counts, 
we then sample $\gamma_0$ and $r_{k}$ as
\begin{align}
&(\gamma_0 \,|\, -) \sim\mbox{Gamma}\left(a_0+{\tilde{l}}_{\cdotv},\,\frac{1}{b_0-\frac{1}{K}\sum_{k}\ln(1-\tilde{p}_{k})}
\right),
\notag\\
&(r_{k}\,|\,-)\sim\mbox{Gamma}\left(\frac{\gamma_0}{K}+ l_{\cdotv k}, \, \frac{1}{c_0
+\sum_i q^{({T}+1)}_{ik}
}\label{eq:r_k}
\right).
\end{align}

\subsection{Numerical stability}

For stack-softplus and SS-softplus regressions with $T>1$, if for some data point $\xv_i$, the inner product $\xv_i'\betav_k^{(t)}$ takes such a large negative number that $e^{-\xv_i'\betav_k^{(t)}}=0$ under 
a finite numerical precision, 
 then ${q}_{ik}^{(t)}=0$ and $\ln {q}_{ik}^{(t)} = -\infty$. For example, in both 64 bit Matlab (version R2015a)  and 64 bit R (version 3.0.2), if $\xv_i'\betav_k^{(t)}\le-745.2$, then $e^{\xv_i'\betav_k^{(t)}}=0$ and hence ${q}_{ik}^{(t)}=0$, $p_{ik}^{(t)}=0$, and $\ln {q}_{ik}^{(t)} = -\infty$. 

If  ${q}_{ik}^{(t)}=0$, then with (\ref{eq:omega_i}), we let  $\omega^{(t+1)}_{ik} = 0$, and with Proposition \ref{prop:PG}, we let
$$-\omega^{(t+1)}_{ik}\ln {q}_{ik}^{(t)} + \frac{m_{ik}^{(t)}- \theta_{ik}^{(t+1)}}{2} = \frac{m_{ik}^{(t)}+\theta_{ik}^{(t+1)}}{2} + \frac{m_{ik}^{(t)}- \theta_{ik}^{(t+1)}}{2} =m_{ik}^{(t)}, $$ and with \eqref{eq:lambda}, we let ${q}_{ik}^{(\tilde{t})}=0$ for all $\tilde{t}\ge t$. 
Note that if ${q}_{ik}^{(t)}=0$, drawing $\omega_{ik}^{(\tilde{t})}$ for $\tilde{t}\in\{ t+1,\ldots,T+1\}$ becomes unnecessary.
To avoid the numerical issue of calculating $\theta_{ik}^{(t)}$ with $\tau_{ik}^{(t)}/{q}_{ik}^{(t)}$ when ${q}_{ik}^{(t)}=0$, we let
\beq
\theta_{ik}^{(t)} ={ \tau_{ik}^{(t)}}\big/{\max\big\{\epsilon,{q}_{ik}^{(t)}\big\}},
\eeq
where we set $\epsilon=10^{-10}$ to for illustrations 
 and $\epsilon=10^{-6}$ to produce the results in the tables. 
To ensure that the covariance matrix for $\betav_k^{(t)}$ is positive definite, we bound 
$\alpha_{vtk}$
above $10^{-3}$.

\subsection{The propagation of latent counts across layers}\label{app:T}

As the number of tables occupied by the customers is in the same order as the logarithm of the number of customers in a Chinese restaurant process, $m^{(t+1)}_{ik}$ in \eqref{eq:ICNBE_finite_1} is in the same order as $\ln\big( m^{(t)}_{ik}\big)$ and hence often quickly decreases as $t$ increases, especially when $t$ is small. In addition, since $m_{ik}^{(t+1)}\le m_{ik}^{(t)}$ almost surely (a.s.), $m_{ik}^{(t)}=0$ a.s. if $m_{ik}^{(1)}=0$, $m_{ik}^{(t)}\ge 1$ a.s. if $m_{ik}^{(1)}\ge 1$, and $m_{i} \ge 1$ a.s. if $y_i=1$, 
we have the following two corollaries.  

\begin{cor}\label{cor:mono1}
The latent count 
$m_{\cdotv k}^{(t)} = \sum_{i} m_{ik}^{(t)}$ 
monotonically decreases as $t $ increases and 
$
m_{\cdotv k}^{(t)}\ge \sum_{i}\delta(m_{ik}^{(1)}\ge 1).
$
\end{cor}

\begin{cor}\label{cor:mono2}
The latent count
$
m_{\cdotv \cdotv}^{(t)} = \sum_{k} m_{\cdotv k}^{(t)}$
monotonically decreases as $t $ increases 
and 
$
m_{\cdotv \cdotv}^{(t)}
\ge \sum_{i} \delta(y_i=1).
$
\end{cor}
With Corollary \ref{cor:mono2},  one may consider using the values of $m_{\cdotv\cdotv}^{(t)}/\sum_{i} \delta(y_i=1)$  to decide whether $T$, the depth of the gamma belief network used in SS-softplus regression, need to be increased  to increase the model capacity,   or whether $T$ could be decreased to reduce the computational complexity. Moreover, with Corollary \ref{cor:mono1},  one may consider using the values of $m_{\cdotv k}^{(t)}/\sum_{i} \delta(y_i=1)$ to decide how many criteria would be sufficient to equip each individual expert. For simplicity, we consider the number of criteria for each expert as a parameter that determines the model capacity and we fix it  as $T$ for all experts in this paper.

\section{Related Methods and Discussions}\label{sec:discussion}

While we introduce a novel nonlinear regression framework for binary response variables, we recognize some interesting connections with previous work, including 
 the gamma belief network, several binary classification algorithms that use multiple hyperplanes, and the ideas of using the mixture or product of multiple probability distributions to construct a more complex predictive distribution, 
 as discussed below.

\subsection{Gamma belief network}\label{sec:GBN}

The Poisson gamma belief network is proposed in \citet{PGBN_NIPS2015} to construct a deep Poisson factor model, in which the shape parameters of the gamma distributed factor score matrix at layer $t$ are factorized under the gamma likelihood into the product of a factor loading matrix and a gamma distributed factor score matrix at layer $t+1$. While the scale parameters of the gamma distributed factor scores depend on the indices of data samples, they are constructed to be independent of the indices of the latent factors, making it convenient to derive closed-form Gibbs sampling update equations via data augmentation and marginalization. The gamma belief networks in both stack- and SS-softplus regressions, on the other hand, do not factorize the gamma shape parameters but parameterize the logarithm of each gamma scale parameter using the inner product of the corresponding covariate and regression coefficient vectors. Hence a gamma scale parameter in softplus regressions depends on both the index of the data sample and that of the corresponding latent expert. On a related note,
while the gamma distribution function is the building unit for both gamma belief networks, the one in \citet{PGBN_NIPS2015} is used to factorize the Poisson rates of the observed or latent high-dimensional count vectors, extracting multilayer deep representations in an unsupervised manner, whereas the one used in \eqref{eq:BerPo_recursive_softplus_reg_model} is designed for supervised learning to establish a direct functional relationship to predict a label given its covariates, without introducing factorization within the gamma belief network.

\subsection{Multi-hyperplane regression models}\label{sec:CPM}
Generalizing the construction of multiclass support vector machines in \citet{crammer2002algorithmic},
the idea of combining multiple hyperplanes to define nonlinear binary classification decision boundary has been discussed in \citet{aiolli2005multiclass}, \citet{wang2011trading}, \citet{manwani2010learning, manwani2011polyceptron}, and \citet{kantchelian2014large}. 
In particular,
\citet{kantchelian2014large} clearly connects the idea of combining multiple hyperplanes for nonlinear classification with the learning of a convex polytope, defined by the intersection of multiple hyperplanes, to separate one class from the other, and shows that a convex polytope classifier can provide larger margins than a linear classifier equipped with a single hyperplane. 

From this point of view, the proposed sum-softplus regression is closely related to the convex polytope machine (CPM) of \citet{kantchelian2014large} as its decision boundary can be explicitly bounded by a convex polytope that encloses negative examples, as described in Theorem \ref{thm:sum_polytope} and illustrated in Fig. \ref{fig:circle_1}. Distinct from the CPM that uses a convex polytope as its decision boundary, and provides no probability estimates for class labels and no principled ways to set its number of equally-weighted hyperplanes, sum-softplus regression makes its decision boundary smoother than the corresponding bounding convex polytope, as shown in Figs. \ref{fig:circle_1} (c)-(d), using more complex interactions between hyperplanes than simple intersection, provides probability estimates for its labels, and supports countably infinite differently-weighted hyperplanes with the gamma-negative binomial process. 

In addition, to solve the objective function that is non-convex, the CPM relies on heuristics to hard assign a positively labeled data point to one and only one of the hyperplanes, making the learning of the parameters for each hyperplane become a convex optimization problem, whereas all softplus regressions use Bayesian inference with closed-form Gibbs sampling update equations, in which each data point is assigned to one or multiple hyperplanes to learn their parameters. 
Moreover, distinct from the CPM and sum-softplus regression that use either a single convex polytope or a single convex-polytope bounded space to enclose \emph{negative} examples,
the proposed stack-softplus regression defines a single convex-polytope-like confined space to enclose \emph{positive} examples, and the proposed SS-softplus regression further generalizes all of them in that its decision boundary is related to the union of multiple convex-polytope-like confined spaces. 

\subsection{Mixture, product, convolution, and stack of experts}
With each regression coefficient vector analogized as an expert,
the proposed softplus regressions can also be related to the idea of
combining multiple experts' beliefs 
to improve a model's predictive performance. 
 Conventionally, if an expert's belief is expressed as a probability density/mass function, then one may consider using the linear opinion pool \citep{LOP} or logarithmic opinion pool \citep{
 genest1986combining,heskes1998selecting} to aggregate multiple experts' probability distributions into a single one. To reach a single aggregated distribution of some unknown quantity $y$, the linear opinion pool, also known as mixture of experts (MoE), aggregates experts additively by 
taking a weighted average of their distributions on $y$, while the logarithmic opinion pool, including the product of experts (PoE) of \citet{POE} as a special case, aggregates them multiplicatively by taking a weighted geometric mean of these distributions \citep{jacobs1995methods,clemen1999combining}. 

Opinion pools with separately trained experts can also be related to ensemble methods \citep{hastie01statisticallearning, zhou2012ensemble}, including both bagging  \citep{breiman1996bagging}  and boosting \citep{freund1997decision}. 
Another common strategy is to jointly train the experts using the same set of features and data. For example, the PoE of  \citet{POE} trains its equal-weighted experts jointly on the same data.  The proposed softplus regressions follow the second strategy to jointly train on the same data not only its experts but also their weights. 



Assume there are $K$ experts and the $k$th expert's belief on $y$ is expressed as a probability density/mass function $f_k(y\,|\,\theta_k)$, where $\theta_k$ represents the distribution parameters. The linear opinion pool aggregates the $K$ expert distributions into a single one using 
\beq
f_Y(y\,|\,\{\theta_k\}_k) = \sum_{k=1}^K \pi_k f_k(y\,|\,\theta_k), \label{eq: LinearOP}
\eeq
where $\pi_k$ are nonnegative weights and sum to one. The logarithmic opinion pool aggregates the $K$ expert distributions using  
\beq
f_Y(y\,|\,\{\theta_k\}_k) = \frac{\prod_{k=1}^K [f_k(y\,|\,\theta_k)]^{\pi_k}}{\sum_{y} \prod_{k=1}^K  [f_k(y\,|\,\theta_k)]^{\pi_k}}, \label{eq: LOP}
\eeq
where $\pi_k\ge 0$  and the constraint $\sum_{k=1}^K \pi_k = 1$   is also  commonly imposed. If $\pi_k=1$ for all $k$, then a logarithmic opinion pool becomes a product of experts (PoE). 

In decision and risk analysis, the functions $f_k(y)$ usually represent independent experts' subjective probabilities, which are often assumed to be known \emph{a priori}, and the focus is to optimize the expert weights~\citep{clemen1999combining}. Whereas in statistics and machine learning, both the functions $f_k(y)$ and the expert weights are typically learned from the data.
One common strategy is to first train different experts separately, such as using different feature sets, different data, and different learning algorithms, and subsequently aggregate  their distributions
into a single one. For example,  in  \cite{tax2000combining}, a set of classifiers are first separately trained on different independent data sources and then aggregated additively or multiplicatively to construct a single classifier. 


\subsubsection{Mixture of experts}
A linear opinion pool is also commonly known as a mixture of experts (MoE). 
Not only are there efficient algorithms, using expectation-maximization (EM) or MCMC, to jointly learn the experts (mixture components) and their mixture weights in a MoE, there are nonparametric Bayesian algorithms, such as Dirichlet process mixture models \citep{DP_Mixture_Antoniak,rasmussen2000infinite}, that support a potentially infinite number of experts. 
 We note it is possible to  combine the proposed softplus regression models with linear opinion pool  to further improve their performance. We leave that extension for future study.

%
\subsubsection{Product of experts}
In contrast to MoEs, 
the logarithmic opinion pool could produce a  probability distribution with sharper boundaries, but at the same time is usually much more challenging to train due to a normalization constant that is typically intractable to compute.
 Hinton's product of experts (PoE)  is one of the most well-know logarithmic opinion pools; 
 since jointly training the experts and their weights  in the logarithmic opinion pool 
makes the inference much more difficult, all the experts in a PoE  are weighted equally \citep{POE}. A PoE is distinct from previously proposed logarithmic opinion pools in that its experts are trained jointly rather than separately.  Even with the restriction of equal weights, the exact gradients of model parameters in a PoE are often intractable to compute, and hence contrastive divergence that approximately computes  these gradients is commonly employed for approximate maximum likelihood inference. Moreover, no nonparametric Bayesian prior is available to allow the number of experts in a PoE to be automatically inferred. PoEs have been successfully applied to binary image modeling \citep{POE} and one of its special forms, the restricted Boltzmann machine, was widely used as a basic building block in constructing a  deep neural network \citep{hinton2006fast,bengio2007greedy}.  PoEs for non-binary data and several other logarithmic opinion pools inspired by PoEs have also been proposed, with applications to  image analysis, information retrieval, and 
 computational linguistics \citep{welling2004exponential,xing2005mining,smith2005logarithmic}. 
We note one may apply the proposed softplus regressions to regress the binary response variables on the covariates transformed by PoEs, such as restricted Boltzmann machine and its deep constructions, to further improve the classification performance. We leave that extension for future study. 

\subsubsection{Stack of experts}

Different from both MoE and PoE, 
we propose stack of experts (SoE) that repeatedly mixes the same distribution with respect to the same distribution parameter as
$$x\sim f(\,r_1,w_1),\ldots,~r_{k-1}\sim f(r_{k},w_{k}),\ldots,~r_{K-1}\sim f(r_K,w_{K}).$$
In a SoE, the marginal distribution of $x$ given $r$ and $\{w_k\}_{1,K}$ can be 
expressed as
$$
f(x\given r_K,\{w_k\}_{1,K}) = \int\! \ldots\!\int f(x\,|\,r_1,w_1) f(r_1\,|\,r_2,w_2)\ldots f(r_{K-1}\,|\,r_K,w_K)dr_{K-1}\ldots dr_{1},
$$
where the parameter $w_k$ that is pushed into the stack after $w_{k-1}$ will pop out before $w_{k-1}$ to parameterize the marginal distribution.
In both stack- and SS-softplus regressions, we obtain a stacked gamma distribution $x\sim f(r_K,\{w_k\}_{1,K})$ by letting $f(x_k\,|\,r_k,w_k)=\mbox{Gamma}(x_k;r_k,e^{w_k})$, and as shown in \eqref{eq:recurssive_softplus_reg}
and \eqref{eq:SRS_regression}, the regression coefficient vectors that are pushed into the stack later appear earlier in both the stack- and SS-softplus functions.




\subsubsection{Convolution of experts}
Distinct from both MoE and PoE, we may consider that both sum- and SS-softplus regressions use a convolution of experts (CoE) strategy to aggregate multiple experts' beliefs by convolving their probability distributions into a single one. A CoEs is based on a fundamental law in probability and statistics: the probability distribution of independent random variables' summation is equal to the convolution of their probability distributions \citep[$e.g.$,][]{fristedt1997}. Thus even though it is possible that the convolution is extremely difficult to solve and hence the explicit form of the aggregated probability density/mass function might not be available, simply adding together the random samples independently drawn from a CoE's experts would lead to a random sample drawn from the aggregated distribution that is smoother than any of the distributions used in convolution.

In a general setting,
denoting $G=\sum_{k=1}^\infty r_k\delta_{\omega_k}$ as a draw from 
 a completely random measure \citep{
 PoissonP} that consists of countably infinite atoms, 
one may construct an infinite CoE model to generate random variables from a convolved distribution as
 $x=\sum_{k=1}^\infty x_k, ~x_k\sim f_k,$
 where $f_k=f(r_k,\omega_k)$ are independent experts parameterized by both the weights $r_k$ and atoms~$\omega_k$ of $G$. Denoting $(f_i*f_j)(x):=\int f_i(\tau)f_j(x-\tau)d\tau$ as the convolution operation, under the infinite CoE model, we have
\beq
f_X(x)= ( f_1 * f_2*\ldots *f_\infty) (x),\notag
\eeq
where 
 the same distribution function is repeatedly convolved to increase the representation power 
 to better fit complex data. 
 
 Under this general framework, we may consider both sum- and SS-softplus regressions as infinite CoEs, with the gamma process used as the underlying completely random measure \cite{Kingman,PoissonP} to support countably infinite differently weighted probability distributions for convolution. 
For sum-softplus regression, as shown in \eqref{eq:CNB} of Theorem \ref{thm:CLR2CNB}, each expert can be considered as a gamma distribution whose scale parameter is parameterized by the inner product of the covariate vector and an expert-specific regression coefficient vector, and the convolution of countably infinite experts' gamma distributions is used as the distribution of the BerPo rate of the response variable; and alternatively, as shown in \eqref{eq:CNB1} of Theorem \ref{thm:CLR2CNB}, each expert can be considered as a NB regression, and the convolution of countably infinite experts' NB distributions is used as the distribution of the latent count response variable. 
For SS-softplus regression, each expert can be considered as a stacked gamma distribution, and the convolution of countably infinite experts' stacked gamma distributions is used as the distribution of the BerPo rate of a response variable.
 

Related to the PoE of \citet{POE}, a CoE trains its experts jointly on the same data. Distinct from that, 
a CoE does not have an intractable normalization constant in the aggregated distribution, its experts can be weighted differently, and its number of experts could be automatically inferred from the data in a nonparametric Bayesian manner. The training for a CoE is also unique, as inferring the parameters of each expert essentially corresponds to deconvolving the aggregated distribution. Moreover, the convolution operation ensures that the aggregated distribution is smoother than every expert distribution. 
For inference, while it is often challenging 
to analytically deconvolve the convolved distribution function, 
we consider first constructing a hierarchical Bayesian model that can generate random variables from the convolved distribution, and then developing a MCMC algorithm to decompose the total sum $x$ into 
the $x_k$ of individual experts, which are then used to infer the model parameters $r_k$ and $\omega_k$ for each expert.

\section{Experimental settings and additional results}\label{app:setting}

We use the $L_2$ regularized logistic regression provided by the LIBLINEAR package \citep{REF08a} to train a linear classifier, 
where a bias term is included and the regularization parameter $C$ is five-fold 
 cross-validated on the training set from $(2^{-10}, 2^{-9},\ldots, 2^{15})$.

For kernel SVM, a Gaussian RBF kernel is used and three-fold cross validation is used to tune both the regularization parameter $C$ and kernel width on the training set. We use the LIBSVM package \citep{LIBSVM}, where we three-fold 
 cross-validate both the regularization parameter $C$ and kernel-width parameter $\gamma$ on the training set from $(2^{-5}, 2^{-4},\ldots, 2^{5})$, and choose the default settings for all the other parameters. Following \citet{chang2010training}, for the ijcnn1 dataset, we choose $C=32$ and $\gamma=2$, and for the a9a dataset, we choose $C=8$ and $\gamma=0.03125$.
 

For RVM, instead of directly quoting the results from \citet{RVM}, which only reported the mean but not standard deviation of the classification errors for each of the first six datasets in Tab. \ref{tab:data}, we use the matlab code\footnote{\url{http://www.miketipping.com/downloads/SB2_Release_200.zip}} provided by the author, using a Gaussian RBF kernel whose kernel width is three-fold cross-validated on the training set from $(2^{-5},2^{-4.5},\ldots,2^{5})$ for both ijcnn1 and a9a and from $(2^{-10},2^{-9.5},\ldots,2^{10})$ for all the others. 
 
 We consider adaptive multi-hyperplane machine (AMM) of \citet{wang2011trading}, as implemented in the BudgetSVM\footnote{\url{http://www.dabi.temple.edu/budgetedsvm/}} (Version 1.1) software package \citep{BudgetSVM}. We consider the batch version of the algorithm. Important parameters of the AMM include both the regularization parameter $\lambda$ and training epochs $E$. As also observed in \citet{kantchelian2014large}, we do not observe the testing errors of AMM to strictly decrease as $E$ increases. Thus,
in addition to cross validating the regularization parameter $\lambda$ on the training set from $\{10^{-7}, 10^{-6},\ldots, 10^{-2}\}$, as done in \citet{wang2011trading}, for each $\lambda$, we try $E\in\{5,10,20,50,100\}$ sequentially until the 
cross-validation error begins to decrease, $i.e.$, under the same $\lambda$, we choose $E=20$ if the cross-validation error of $E=50$ is greater than that of $E=20$. 
We use the default settings for all the other parameters. 
 
 We consider the convex polytope machine (CPM) of \citep{kantchelian2014large}, using the python code\footnote{\url{https://github.com/alkant/cpm}} provided by the authors. Important parameters of the CPM include the entropy parameter $h$, regularization factor $C$, and number of hyperplanes $K$ for each side of the CPM (2$K$ hyperplanes in total). Similar to the setting of \citep{kantchelian2014large}, we first fix $h=0$ and select the best regularization factor $C$ from $\{10^{-4},10^{-3},\ldots,10^{0}\}$ using three-fold cross validation on the training set. For each $C$, we try $K\in\{1,3,5,10,20,40,60,80,100\}$ sequentially until the 
cross-validation error begins to decrease. With both $\lambda$ and $K$ selected, we then select $h$ from $\{0,\ln(K/10), \ln(2K/10),
 \ldots,\ln(9K/10)\}$. For each trial, we consider $10$ million iterations in cross-validation and $32$ million iterations in training with the cross-validated parameters. Note different from \citet{kantchelian2014large}, which suggests that the error rate decreases as $K$ increases, 
 we cross-validate $K$ as 
 we have found that the testing errors of the CPM may increase once it increases over certain limits. 


\begin{algorithm}[t!]
\small
 \caption{\small Upward-downward Gibbs sampling for sum-stack-softplus (SS-softplus) regression.\newline
 \textbf{Inputs:} $y_i$: the observed labels, $\xv_i$: covariate vectors, $K_{\max}$: the upper-bound of the number of experts, 
 $T$: the number of criteria of each expert, $I_{Prune}$: the set of 
 iterations at which the operation of deactivating 
 experts 
 is performed, and the model hyper-parameters.
\newline
\textbf{Outputs:} $KT $ regression coefficient vectors $\betav^{(t)}_{k}$ and $K$ weights $r_k$, where $K\le K_{\max}$ is the total number of active experts that are associated with nonzero latent counts. 
 }\label{alg:1}
 \begin{algorithmic}[1] 
 \State \text{
 Initialize the model parameters with $\betav^{(t)}_{k}=0$ and $r_k=1/K_{\max}$.}
 \For{\text{$iter=1:maxIter$ 
 }} Gibbs sampling
 
 \ParFor{\text{$k=1,\ldots,K_{\max}$}} Downward sampling
 

 \For{\text{$t=T,T-1\ldots,1$}} 
 \State 
 \text{Sample $\theta_{ik}^{(t)}$ if Expert $k$ is active } ; 
\EndFor
 \EndParFor
 \State 
 \text{Sample $m_i$ }; \text{Sample $\{m^{(1)}_{ik}\}_k$ };

 \ParFor{\text{$k=1,\ldots,K_{\max}$}} Upward sampling
 \If{Expert $k$ is active }
 \For{\text{$t=2,3,\ldots,{T}+1$}} 
 \State 
 \text{Sample $m_{ik}^{(t)}$} ; \text{Sample $\omega_{ik}^{(t)}$} ; \text{Sample $\betav_k^{(t)}$ and Calculate $p_{ik}^{(t)}$ and ${q}_{ik}^{(t)}$} ;
 \EndFor
 \EndIf 
 \State
Deactivate Expert $k$ if $iter \in I_{Prune}$ and $m_{\cdotv k}^{(1)} = 0$ ;
%
 \EndParFor
 \State
 \text{Sample $\gamma_0$ and $c_0$} ;
 \State
 \text{Sample $r_{1},\ldots,r_{K_{\max}}$} ;
 
 \EndFor
 
 \end{algorithmic}
 \normalsize
\end{algorithm}%

 \begin{figure}[!h]
\begin{center}
 \includegraphics[width=0.75\columnwidth]{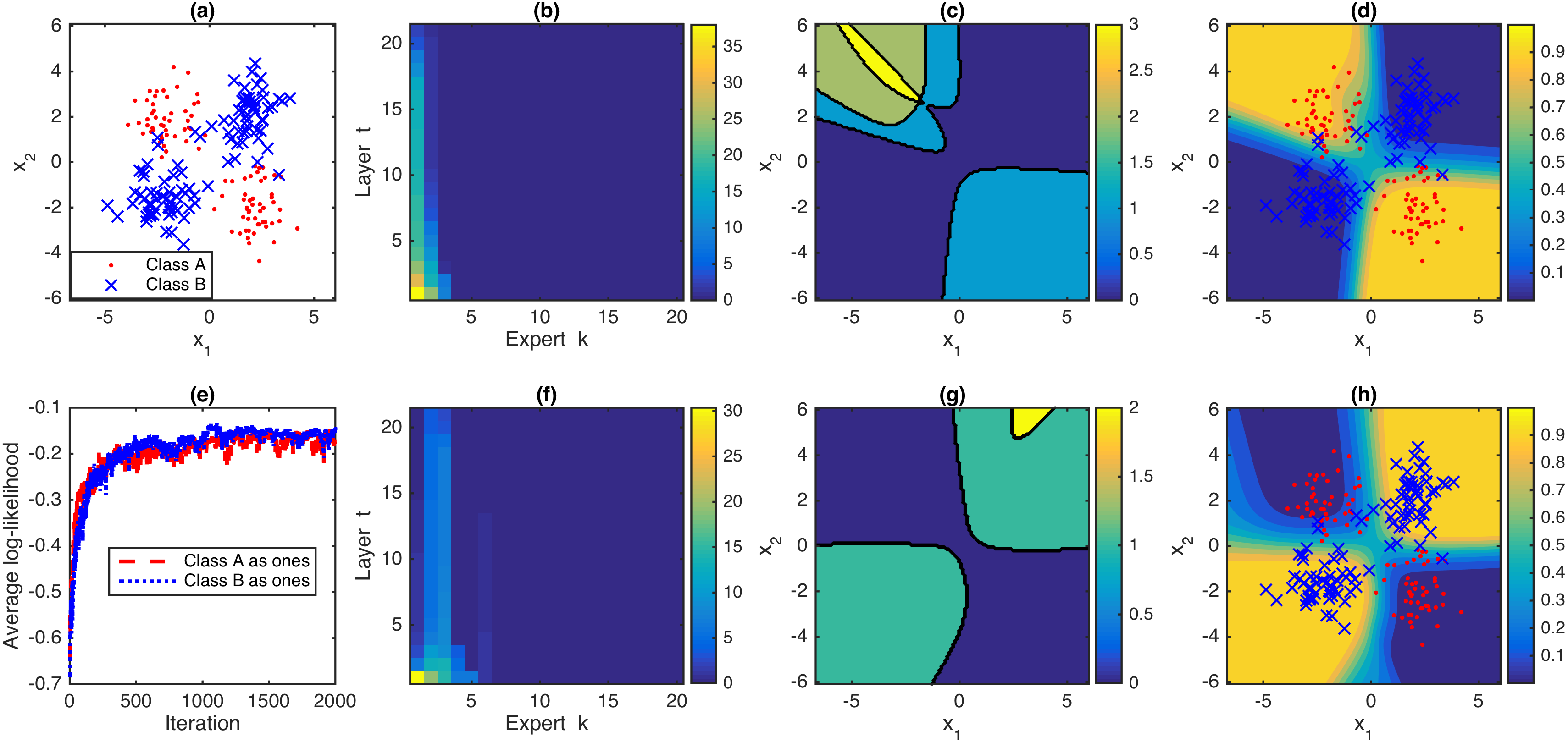}
\vspace{-.2cm}
\end{center}
\vspace{-4.9mm}
\caption{\small\label{fig:xor}
Analogous figure to to Fig. \ref{fig:banana} for SS-softplus regression for a different dataset with two classes, where Class $A$ consists of 50 data points $(x_i,y_i)$ centering around $(-2,2)$, where $x_i\sim\mathcal{N}(-2,1)$ and $y_i\sim\mathcal{N}(2,1)$, and another 50 such kind of data points centering around $(2,-2)$, and Class $B$ consists of 50 such kind of data points centering around $(2,2)$ and another 50 such kind of data points centering around $(-2,-2)$.
}

\begin{center}
 \includegraphics[width=0.75\columnwidth]{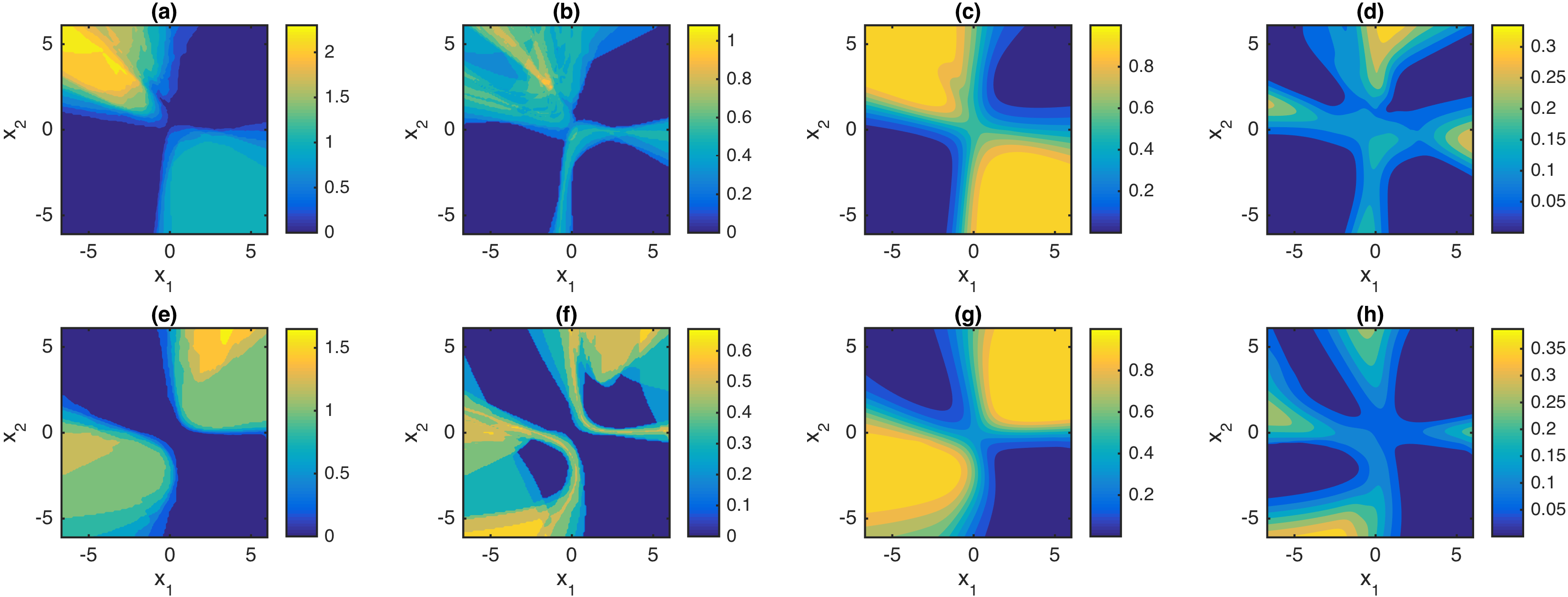}
 \vspace{-.2cm}
\end{center}
\vspace{-5.9mm}
\caption{\small\label{fig:xor_ave}
Analogous figure to Fig. \ref{fig:banana}, 
with the same experimental setting used for Fig. \ref{fig:xor}.}
\end{figure}

 \begin{figure}[!t]
\begin{center}
 \includegraphics[width=0.75\columnwidth]{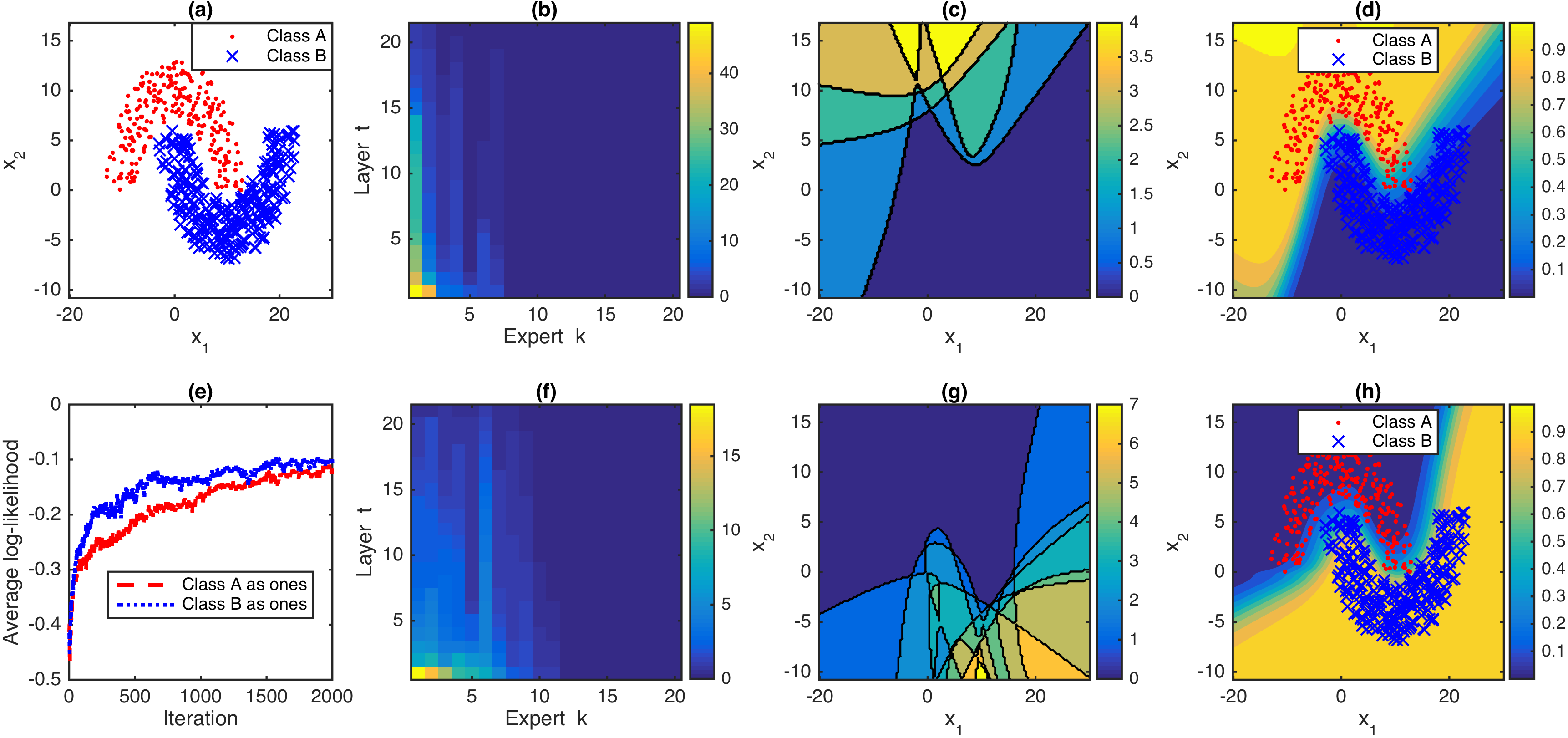}
\vspace{-.2cm}
\end{center}
\vspace{-5.9mm}
\caption{\small\label{fig:dbmoon}
Analogous figure to Fig. \ref{fig:banana} for SS-softplus regression for a double moon dataset, where both Classes $A$ and $B$ consist of 250 data points $(x_i,y_i)$.
}

\begin{center}
 \includegraphics[width=0.75\columnwidth]{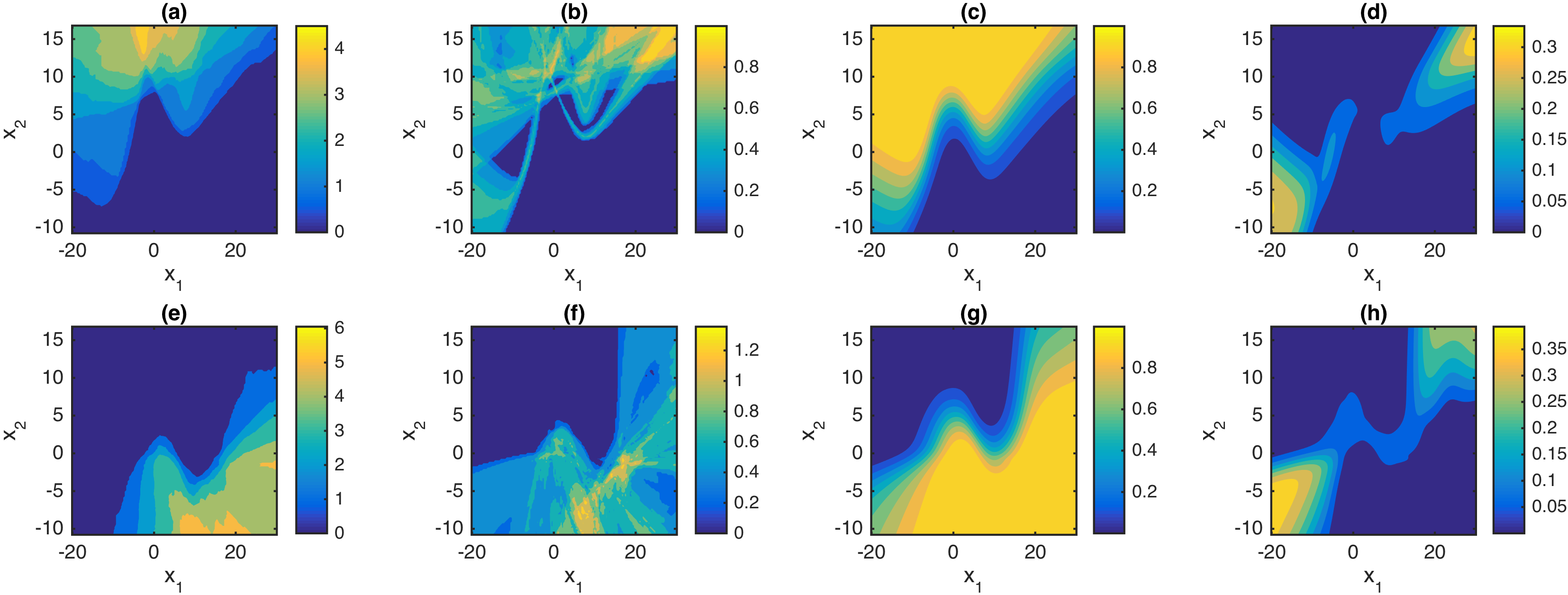}
\vspace{-.2cm}
\end{center}
\vspace{-4.9mm}
\caption{\small\label{fig:dbmoon_ave}
Analogous figure to Fig. \ref{fig:banana_ave}, 
with the same experimental setting used for Fig. \ref{fig:dbmoon}. }
\end{figure}

\begin{table}[t!]
\footnotesize
\caption{\small Binary classification datasets used in experiments, where 
$V$ is the feature dimension.\vspace{1mm}}\label{tab:data}
\vspace{-4mm}
\begin{center}
\begin{tabular}{c | c c c c c c c c c}
\toprule
Dataset & banana & breast cancer & titanic & waveform & german & image & ijcnn1 & a9a \\
\midrule
Train size &400 &200&150&400&700&1300 &49,990& 32,561 \\
Test size &4900 &77&2051&4600&300&1010 & 91,701& 16,281 \\
$V$ & 2 &9 & 3 &21 &20 &18&22& 123\\
\bottomrule
\end{tabular}
\end{center}
\vspace{-4mm}

\footnotesize 
\caption{\small Performance of stack-softplus regression with the depth set as $T\in\{1,2,3,5,10\}$, where stack-softplus regression with $T=1$ reduces to softplus regression.\vspace{1mm}} 
\label{table:4} 
\centering 
\begin{tabular}{c|ccccc} 
\toprule 
 Dataset & softplus & stack-$\varsigma$ ($T$=2) & stack-$\varsigma$ ($T$=3) & stack-$\varsigma$ ($T$=5) & stack-$\varsigma$ ($T$=10) \\ 
 \midrule 
banana & $47.87 \pm 4.36$ & $34.66 \pm 5.58$ & $32.19 \pm 4.76$ & $33.21 \pm 5.76$ & $\mathbf{30.67} \pm 4.23$ \\ 
breast cancer & $28.70 \pm 4.76$ & $29.35 \pm 2.31$ & $29.48 \pm 4.94$ & $\mathbf{27.92} \pm 3.31$ & $28.31 \pm 4.36$ \\ 
titanic & $22.53 \pm 0.43$ & $22.80 \pm 0.59$ & $22.48 \pm 0.55$ & $22.71 \pm 0.70$ & $22.84 \pm 0.54$ \\ 
waveform & $13.62 \pm 0.71$ & $12.52 \pm 1.14$ & $\mathbf{12.23} \pm 0.79$ & $12.25 \pm 0.69$ & $12.33 \pm 0.65$ \\ 
german & $24.07 \pm 2.11$ & $23.73 \pm 1.99$ & $23.67 \pm 1.89$ & $\mathbf{22.97} \pm 2.22$ & $23.80 \pm 1.64$ \\ 
image & $17.55 \pm 0.75$ & $9.11 \pm 0.99$ & $8.39 \pm 1.05$ & $7.97 \pm 0.52$ & $\mathbf{7.50} \pm 1.17$ \\ 
 \bottomrule 
\begin{tabular}{@{}c@{}}\scriptsize Mean of SVM\\\scriptsize normalized errors\end{tabular} & $2.485$ & $1.773 $ & $1.686$ & $1.665$ & $1.609$ \\ 
\end{tabular} 

\footnotesize 
\caption{\small Performance of SS-softplus regression with $K_{\max}=20$ and the depth set as $T\in\{1,2,3,5,10\}$, where SS-softplus regression with $T=1$ reduces to sum-softplus regression.\vspace{1mm}} 
\label{tab:5} 
\centering 
\begin{tabular}{c|ccccc} 
\toprule 
Dataset & sum-$\varsigma$ & SS-$\varsigma$ ($T$=2) & SS-$\varsigma$ ($T$=3) & SS-$\varsigma$ ($T$=5) & SS-$\varsigma$ ($T$=10) \\ 
 \midrule 
banana & $30.78 \pm 8.68$ & $15.00 \pm 5.31$ & $12.54 \pm 1.18$ & $\mathbf{11.89} \pm 0.61$ & $11.93 \pm 0.59$ \\ 
breast cancer & $30.13 \pm 4.23$ & $29.74 \pm 3.89$ & $30.39 \pm 4.94$ & $28.83 \pm 3.40$ & $\mathbf{28.44} \pm 4.60$ \\ 
titanic & $22.48 \pm 0.25$ & $22.56 \pm 0.65$ & $22.42 \pm 0.45$ & $22.29 \pm 0.80$ & $\mathbf{22.20} \pm 0.48$ \\ 
waveform & $11.51 \pm 0.65$ & $11.41 \pm 0.96$ & $\mathbf{11.34} \pm 0.70$ & $11.69 \pm 0.69$ & $12.92 \pm 1.00$ \\ 
german & $23.60 \pm 2.39$ & $\mathbf{23.30} \pm 2.54$ & $\mathbf{23.30} \pm 2.20$ & $24.23 \pm 2.46$ & $23.90 \pm 1.50$ \\ 
image & $3.50 \pm 0.73$ & $2.76 \pm 0.47$ & $\mathbf{2.59} \pm 0.47$ & $2.73 \pm 0.53$ & $2.93 \pm 0.46$ \\ 
\bottomrule 
\begin{tabular}{@{}c@{}}\scriptsize Mean of SVM\\\scriptsize normalized errors\end{tabular} & $1.370$ & $1.079$ & $1.033$ & $1.033$ & $1.059$ \\ 
\end{tabular} 
 

\caption{\small Analogous table to Tab. \ref{tab:5} for the number of inferred experts (hyperplanes). 
\vspace{1mm}} 
\label{table:6} 
\begin{tabular}{c|ccccc} 
\toprule 
Dataset & sum-$\varsigma$ & SS-$\varsigma$ ($T$=2) & SS-$\varsigma$ ($T$=3) & SS-$\varsigma$ ($T$=5) & SS-$\varsigma$ ($T$=10) \\ 
\midrule 
banana & $3.70 \pm 0.95$ & $5.70 \pm 0.67$ & $6.80 \pm 0.79$ & $7.60 \pm 1.17$ & $9.80 \pm 2.39$ \\                                                                                                          
breast cancer & $3.10 \pm 0.74$ & $4.10 \pm 0.88$ & $5.70 \pm 1.70$ & $6.40 \pm 1.43$ & $9.50 \pm 1.51$ \\                                                                                                
titanic & $2.30 \pm 0.48$ & $3.30 \pm 0.82$ & $3.80 \pm 0.92$ & $4.00 \pm 0.94$ & $6.20 \pm 1.23$ \\                                                                                                    
waveform & $4.40 \pm 0.84$ & $6.20 \pm 1.62$ & $7.00 \pm 2.21$ & $8.90 \pm 2.33$ & $11.50 \pm 2.72$ \\                                                                                                   
german & $6.70 \pm 0.95$ & $9.80 \pm 1.48$ & $11.10 \pm 2.64$ & $14.70 \pm 1.77$ & $20.00 \pm 2.40$ \\                                                                                                 
image & $11.20 \pm 1.32$ & $13.20 \pm 2.30$ & $14.60 \pm 2.07$ & $17.60 \pm 1.90$ & $21.40 \pm 2.22$ \\   
\bottomrule 
\begin{tabular}{@{}c@{}}\scriptsize Mean of SVM \\\scriptsize normalized $K$\end{tabular} & $0.030$ & $0.041~(\times 2)$ & $0.048~(\times 3)$ & $0.057~(\times 5)$ & $0.077~(\times 10)$ \\ 
\end{tabular} 
\end{table}

\begin{table}[t!] 
\footnotesize 
\caption{\small Comparison of classification errors of logistic regression (LR), support vector machine (SVM), 
adaptive multi-hyperplane machine (AMM), convex polytope machine (CPM), softplus regression, sum-softplus (sum-$\varsigma$) regression with $K_{\max}=20$, stack-softplus (stack-$\varsigma$) regression with $T=5$, and SS-softplus regression with $K_{\max}=20$ and $T=5$. \vspace{1mm}}\label{tab:Error1} 
\centering 
\begin{tabular}{c|ccccccccc} 
\toprule 
 Dataset & LR & SVM & RVM & AMM & CPM & softplus & sum-$\varsigma$ & stack-$\varsigma$ ($T$=5) & SS-$\varsigma$ ($T$=5) \\ 
 \midrule 
ijcnn1 & $8.00$ & $1.30$ & $\mathbf{1.29}$ & $2.06$ & $2.57$ & $8.41$ & $3.39$ & $6.43$ & $2.24$ \\ 
 & & & & $\pm 0.27$ & $\pm 0.17$ & $\pm 0.03$ & $\pm 0.17$ & $\pm 0.15$ & $\pm 0.12$ \\
 \midrule 
a9a & $15.00$ & $\mathbf{14.88}$ & $14.95$ & $15.03$ & $15.08$ & $15.02$ & $\mathbf{14.88}$ & $15.00$ & $15.02$ \\ 
 & & & & $\pm 0.17$ & $\pm 0.07$ & $\pm 0.06$ & $\pm 0.05$ & $\pm 0.06$ & $\pm 0.11$ \\
 \bottomrule 
\end{tabular} 
 

\footnotesize 
\caption{\small 
 Analogous table to Tab. \ref{tab:Error1} for the number of inferred experts (hyperplanes).
\vspace{1mm}} 
\label{table:K1} 
\centering 
\begin{tabular}{c|ccccccccc} 
\toprule 
 & LR & SVM & RVM & AMM & CPM & softplus & sum-$\varsigma$ & stack-$\varsigma$ ($T$=5) & SS-$\varsigma$ ($T$=5) \\ 
 \midrule 
ijcnn1 & 1 & $2477$ & $296$ & $8.20$ & $58.00$ & 2 & $37.60$ & $2~(\times 5)$ & $38.80~(\times 5)$ \\ 
 & & & & $\pm 0.84$ & $\pm 13.04$ & & $\pm 1.52$ & & $\pm 0.84~(\times 5)$ \\
 \midrule 
a9a & 1 & $11506$ & $109$ & $28.00$ & $7.60$ & 2 & $37.60$ & $2~(\times 5)$ & $40.00~(\times 5)$ \\ 
 & & & & $\pm 4.12$ & $\pm 2.19$ & & $\pm 0.55$ & & $\pm 0.00~(\times 5)$\\ 
 \bottomrule 
\end{tabular}



\caption{\small Performance of stack-softplus regression with the depth set as $T\in\{1,2,3,5,10\}$, where stack-softplus regression with $T=1$ reduces to softplus regression.\vspace{1mm}} 
\centering 
\label{table:9} 
\begin{tabular}{c|ccccc} 
\toprule 
Dataset & softplus & stack-$\varsigma $ ($T$=2) & stack-$\varsigma $ ($T$=3) & stack-$\varsigma $ ($T$=5) & stack-$\varsigma $ ($T$=10) \\
 \midrule
ijcnn1 & $8.41 \pm 0.03$ & $6.73 \pm 0.13$ & $6.44 \pm 0.21$ & $6.43 \pm 0.15$ & $\mathbf{6.39} \pm 0.08$ \\ 
a9a & $15.02 \pm 0.06$ & $14.96 \pm 0.04$ & $\mathbf{14.93} \pm 0.06$ & $15.00 \pm 0.06$ & $14.97 \pm 0.08$ \\
 \bottomrule 
\end{tabular} 

\caption{\small Performance of SS-softplus regression with $K_{\max}=20$ and the depth set as $T\in\{1,2,3,5,10\}$, where SS-softplus regression with $T=1$ reduces to sum-softplus regression.\vspace{1mm}} 
\centering 
\label{table:10} 
\begin{tabular}{c|ccccc} 
\toprule 
 Dataset & sum-$\varsigma$ & SS-$\varsigma$ ($T$=2) & SS-$\varsigma$ ($T$=3) & SS-$\varsigma$ ($T$=5) & SS-$\varsigma$ ($T$=10) \\ 
 \midrule
ijcnn1 & $3.39 \pm 0.17$ & $2.32 \pm 0.18$ & $2.31 \pm 0.17$ & $2.24 \pm 0.12$ & $\mathbf{2.19} \pm 0.11$ \\ 
a9a & $\mathbf{14.88} \pm 0.05$ & $14.98 \pm 0.03$ & $15.07 \pm 0.20$ & $15.02 \pm 0.11$ & $15.09 \pm 0.06$ \\
 \bottomrule 
\end{tabular} 
 
\end{table}

\end{document}